%% file: Main_revision_without_track_changes.tex
\documentclass[journal]{IEEEtran}
\input{usage_package}
\usepackage{soul}

\definecolor{main}{HTML}{A9A9A9}    
\definecolor{sub}{HTML}{f8f8ff}     
\definecolor{maroon}{cmyk}{0,0.87,0.68,0.32}
\definecolor{MyTeal}{RGB}{0,80,80}
\definecolor{MyPurple}{RGB}{80,45,115}
\definecolor{MyOrange}{RGB}{160,75,0}
\definecolor{MyCrimson}{RGB}{170,35,55}
\newcommand*\colourcheck[1]{%
  \expandafter\newcommand\csname #1check\endcsname{\textcolor{#1}{\ding{52}}}%
}
\colourcheck{blue}
\colourcheck{green}
\colourcheck{red}
\definecolor{MyTeal}{RGB}{0,128,128}
\newtcolorbox{boxI}{
    colback = sub, 
    colframe = main, 
    boxrule = 0pt, 
    toprule = 4pt,
    left=2pt,
    right=2pt,
    after skip=2pt,
    before skip=2pt,
       bottom=0pt,
    top=0pt
}

\newtcolorbox{boxIRQ}{
    colback = sub, 
    colframe = mainrq, 
    boxrule = 0pt, 
    toprule = 4pt,
    left=2pt,
    right=2pt,
    after skip=4pt,
    before skip=3pt,
    bottom=0pt,
    top=0pt
}
\definecolor{mainrq}{HTML}{3939c6}

\newcolumntype{P}[1]{>{\centering\arraybackslash}p{#1}}

\SetKw{Continue}{continue}

\makeatletter
\xpatchcmd{\proof}{\topsep6\p@\@plus6\p@\relax}{}{}{}
\makeatother

\makeatletter
\newcommand{\paragraphalphastyle}{
    \renewcommand{\theparagraph}{\alph{paragraph}}      
    \renewcommand{\theparagraphdis}{\alph{paragraph})}  
}

\newcommand{\paragraphnumstyle}{
    \renewcommand{\theparagraph}{\arabic{paragraph}}      
    \renewcommand{\theparagraphdis}{\arabic{paragraph})}  
}
\makeatother

\newtheorem{theorem}{Theorem}
\newtheorem{assumption}{Assumption}
\newtheorem{lemma}{Lemma}
\newtheorem{definition}{Definition}
\newtheorem{remark}{Remark}

\newtheorem{constraint}{Constraint}
\def\BibTeX{{\rm B\kern-.05em{\sc i\kern-.025em b}\kern-.08em
    T\kern-.1667em\lower.7ex\hbox{E}\kern-.125emX}}
\begin{document}


\title{Carpe Diem: Critical Learning Period-Aware Contract-Based Incentives for Federated Learning}


\newcommand{\clps} {{\text{\fontfamily{qcr}\selectfont CLPs}}}
\newcommand{\clp}{{\text{\fontfamily{qcr}\selectfont CLP}}}
\newcommand{\nonclp}{{\text{\fontfamily{qcr}\selectfont non-CLP}}}
\newcommand{\nonclps} {{\text{\fontfamily{qcr}\selectfont non-CLPs}}}
\newcommand{\proposed}{{\text{\fontfamily{qcr}\selectfont R3T}}}

\author{Thanh Linh Nguyen,~\IEEEmembership{Graduate Student Member,~IEEE}, Dinh Thai Hoang,~\IEEEmembership{Senior Member,~IEEE}, Diep N. Nguyen,~\IEEEmembership{Senior Member,~IEEE}, and Quoc-Viet Pham,~\IEEEmembership{Senior Member,~IEEE} \\
\thanks{Thanh Linh Nguyen and Quoc-Viet Pham (Corresponding Author) are with the School of Computer Science and Statistics, Trinity College Dublin, The University of Dublin,
Dublin 2, D02PN40, Ireland (e-mail: \{tnguyen3, viet.pham\}@tcd.ie).
}
\thanks{Dinh~Thai~Hoang and Diep~N.~Nguyen are with the School of Electrical and Data Engineering, University of Technology Sydney, Sydney, NSW 2007, Australia (e-mail: \{hoang.dinh, diep.nguyen\}@uts.edu.au).} 
\thanks{Part of this article was presented at the 2025 IEEE 101st Vehicular Technology Conference \cite{11174880}.}
}


\maketitle
\begin{abstract}
Critical learning periods (\clps) in federated learning (FL) refer to early stages during which low-quality contributions (e.g., sparse training data availability, strategic non-participation) can permanently impair the performance of the global model owned by the model owner (i.e., a cloud server). However, existing incentive mechanisms typically assume temporal homogeneity, treating all training rounds as equally important, thereby failing to prioritize and attract high-quality contributions during \clps. This inefficiency is compounded by information asymmetry due to privacy regulations, where the cloud lacks knowledge of client training capabilities, leading to adverse selection and moral hazard. Thus, in this article, we propose a time-aware contract-theoretic incentive framework, named Right Reward Right Time (\text{\fontfamily{qcr}\selectfont R3T}), to encourage client involvement, especially during \clps, to maximize 
the utility of the cloud server in FL. We formulate a cloud utility function that captures the trade-off between the achieved model performance and rewards allocated for clients' contributions, explicitly accounting for client heterogeneity in time and system capabilities, effort, and joining time. Then, we devise a \clp-aware incentive mechanism deriving an optimal contract design that satisfies individual rationality, incentive compatibility, and budget feasibility constraints, motivating rational clients to participate early and contribute efforts. By providing the right reward at the right time, our approach can attract the highest-quality contributions during \clps. Simulation and proof-of-concept studies show that \text{\fontfamily{qcr}\selectfont R3T} mitigates information asymmetry, increases cloud utility, and yields superior economic efficiency compared to conventional incentive mechanisms. Our proof-of-concept results demonstrate up to a 47.6\% reduction in the total number of clients and up to a 300\% improvement in convergence speed while achieving competitive test accuracy. {The source code can be found at https://github.com/linhnt31/R3T}.
\end{abstract}

\begin{IEEEkeywords}
Critical learning periods, Federated learning, Incentive mechanism, Blockchain, Contract theory.
\end{IEEEkeywords}

\newcommand\mycommfont[1]{\footnotesize\ttfamily{#1}}
\SetCommentSty{mycommfont}

\SetKwInput{KwInput}{Input}                
\SetKwInput{KwOutput}{Output}              
\renewcommand{\algorithmicforall}{\textbf{for each}}

\input{Sections_revision_wt_trackchanges/1_Introduction}
\input{Sections_revision_wt_trackchanges/2_1_Related_work}
\input{Sections_revision_wt_trackchanges/2_System_Model}
\input{Sections_revision_wt_trackchanges/3_Incentive_mechanism_Design}

\input{Sections_revision_wt_trackchanges/4_Experiment}
\input{Sections_revision_wt_trackchanges/5_Conclusion}

\section*{Acknowledgment}
This research has been conducted with financial support of Taighde Éireann–Research Ireland under Grant number 18/CRT/6222; School of Computer Science and Statistics, Trinity College Dublin. The work of Quoc-Viet Pham is supported in part by Research Ireland under the European Innovation Council CHIST-ERA SHIELD project (Project No. 216449, Award No. 19226). For the purpose of Open Access, the author has applied a CC BY public copyright licence to any Author Accepted Manuscript version arising from this submission.
\bibliographystyle{IEEEtran}
\bibliography{Refs}

\begin{IEEEbiography}[{\includegraphics[width=1in,height=1.25in,clip,keepaspectratio]{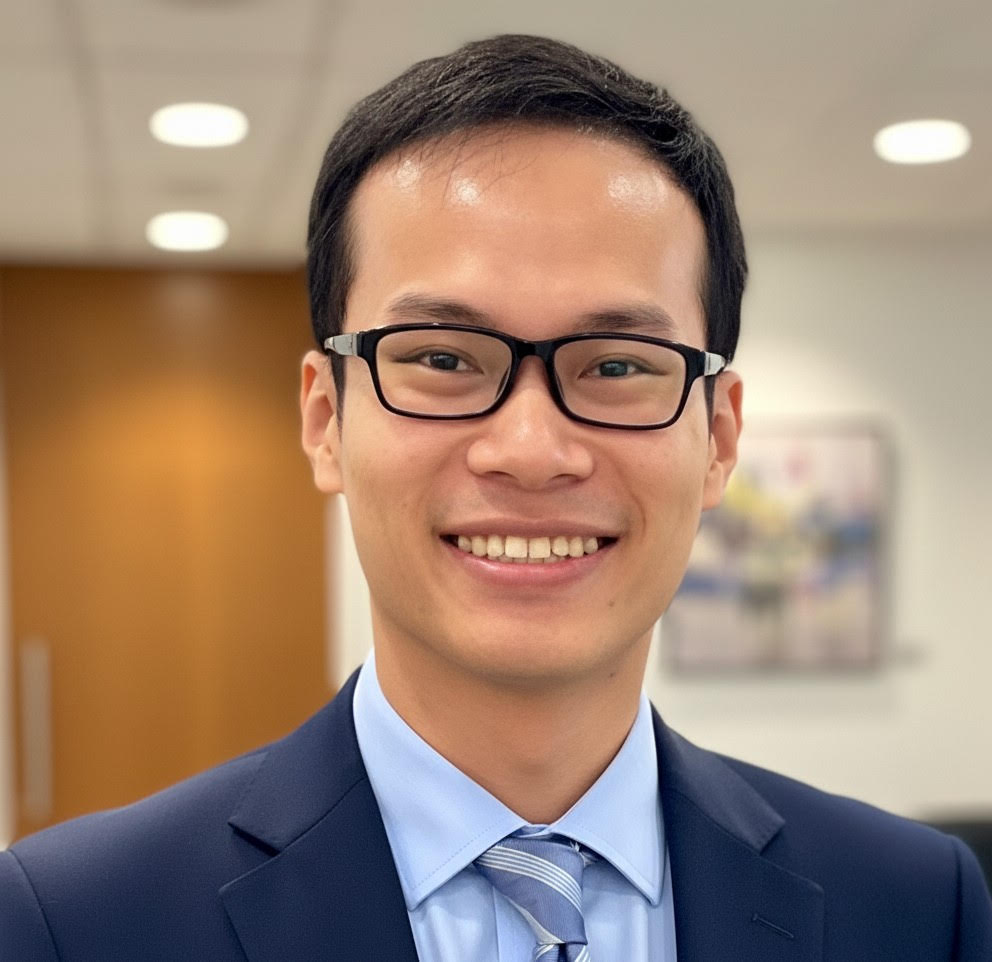}}]{Thanh Linh Nguyen} (Graduate Student Member, IEEE) is a Ph.D. candidate at the School of Computer Science and Statistics, Trinity College Dublin, Ireland. Nguyen received his engineering degree in telecommunications from the Posts and Telecommunications Institute of Technology, Hanoi, Vietnam. His current research interests include privacy-preserving data sharing, federated learning, incentivization, and network economics.
\end{IEEEbiography}
\begin{IEEEbiography}[{\includegraphics[width=1in,height=1.25in,clip,keepaspectratio]{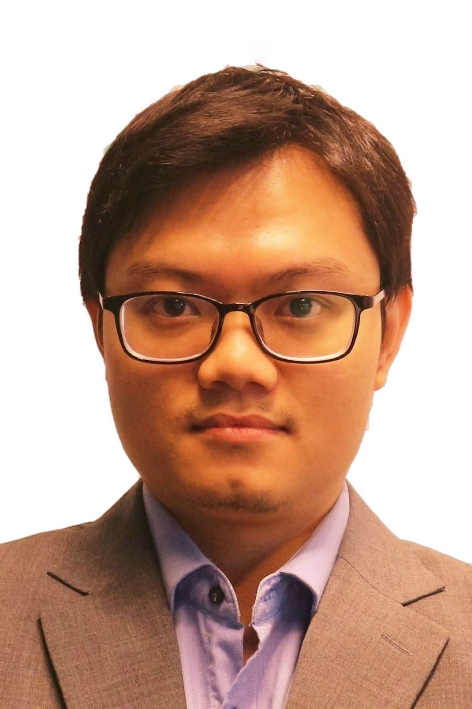}}]{Dinh Thai Hoang} (M’16, SM’22) is currently a faculty member at the School of Electrical and Data Engineering, University of Technology Sydney, Australia. He received his Ph.D. in Computer Science and Engineering from Nanyang Technological University, Singapore, in 2016. His research interests include emerging wireless communications and networking topics, especially machine learning applications in networking, edge computing, and cybersecurity. He has received several prestigious awards, including the Australian Research Council Discovery Early Career Researcher Award, IEEE TCSC Award for Excellence in Scalable Computing for Contributions on “Intelligent Mobile Edge Computing Systems” (Early Career Researcher), IEEE Asia-Pacific Board (APB) Outstanding Paper Award 2022, and IEEE Communications Society Best Survey Paper Award 2023. He is currently an Editor of IEEE TMC, IEEE TWC, IEEE TCOM, and IEEE TNSE.
\end{IEEEbiography}
\begin{IEEEbiography}[{\includegraphics[width=1in,height=1.25in,clip,keepaspectratio]{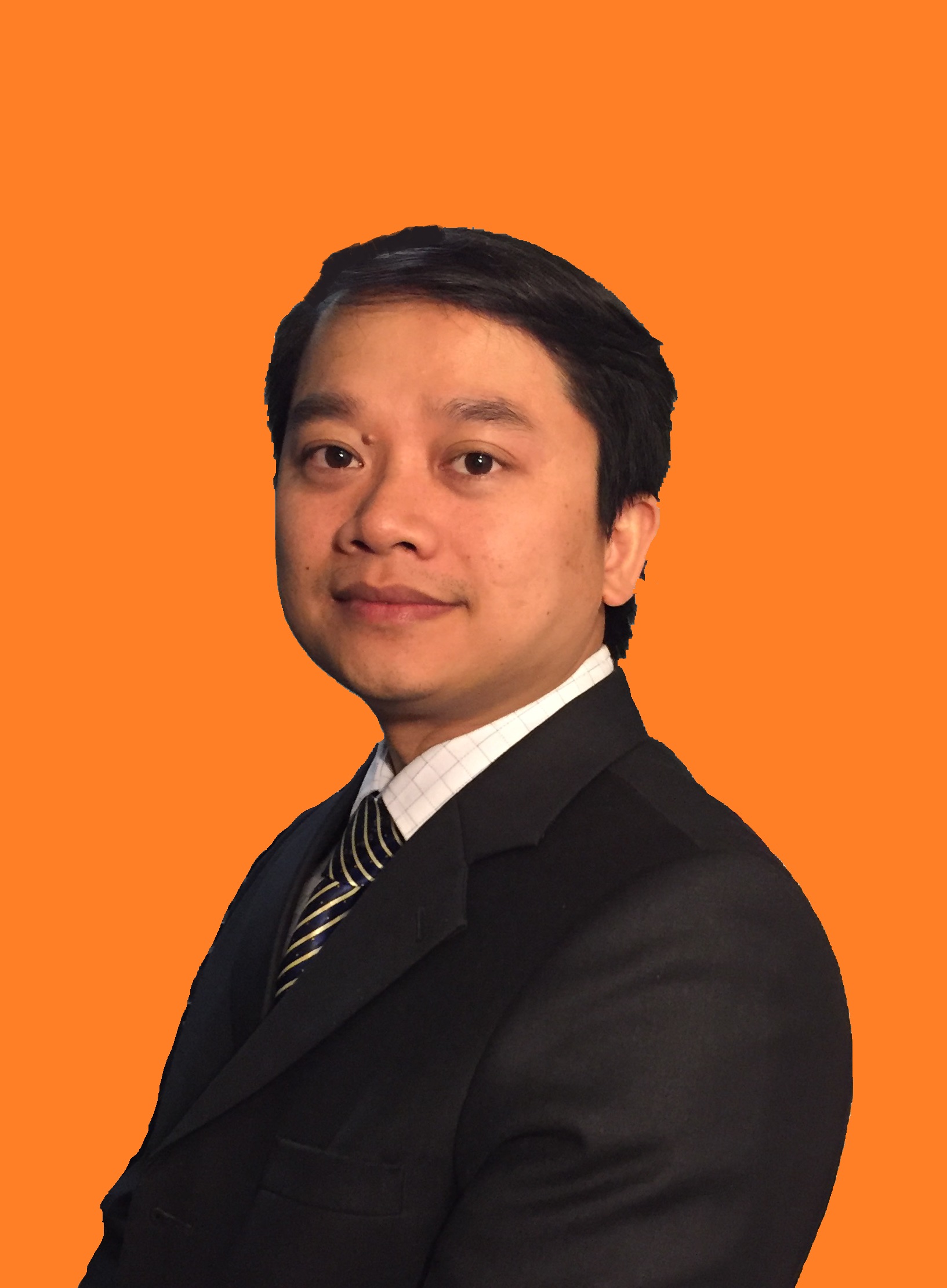}}]{Diep N. Nguyen} (Senior Member, IEEE) received the M.E. degree in electrical and computer engineering from the University of California at San Diego (UCSD), La Jolla, CA, USA, in 2008, and the Ph.D. degree in electrical and computer engineering from The University of Arizona (UA), Tucson, AZ, USA, in 2013. He is currently the Head of 5G/6G Wireless Communications and Networking Lab, Director of Agile Communications and Computing group, Faculty of Engineering and Information Technology, University of Technology Sydney (UTS), Sydney, Australia. Before joining UTS, he was a DECRA Research Fellow with Macquarie University, Macquarie Park, NSW, Australia, and a Member of the Technical Staff with Broadcom Corporation, CA, USA, and ARCON Corporation, Boston, MA, USA, and consulting the Federal Administration of Aviation on turning detection of UAVs and aircraft, and the U.S. Air Force Research Laboratory, USA, on anti-jamming. His research interests include computer networking, wireless communications, and machine learning application, with emphasis on systems’ performance and security/privacy. Dr. Nguyen received several awards from U.S. Congress,  the U.S. National Science Foundation, and the Australian Research Council. He has served on the Editorial Boards of the IEEE Transactions on Mobile Computing, IEEE Communications Surveys \& Tutorials (COMST), IEEE Open Journal of the Communications Society.
\end{IEEEbiography}
\begin{IEEEbiography}[{\includegraphics[width=1in,height=1.25in,clip,keepaspectratio]{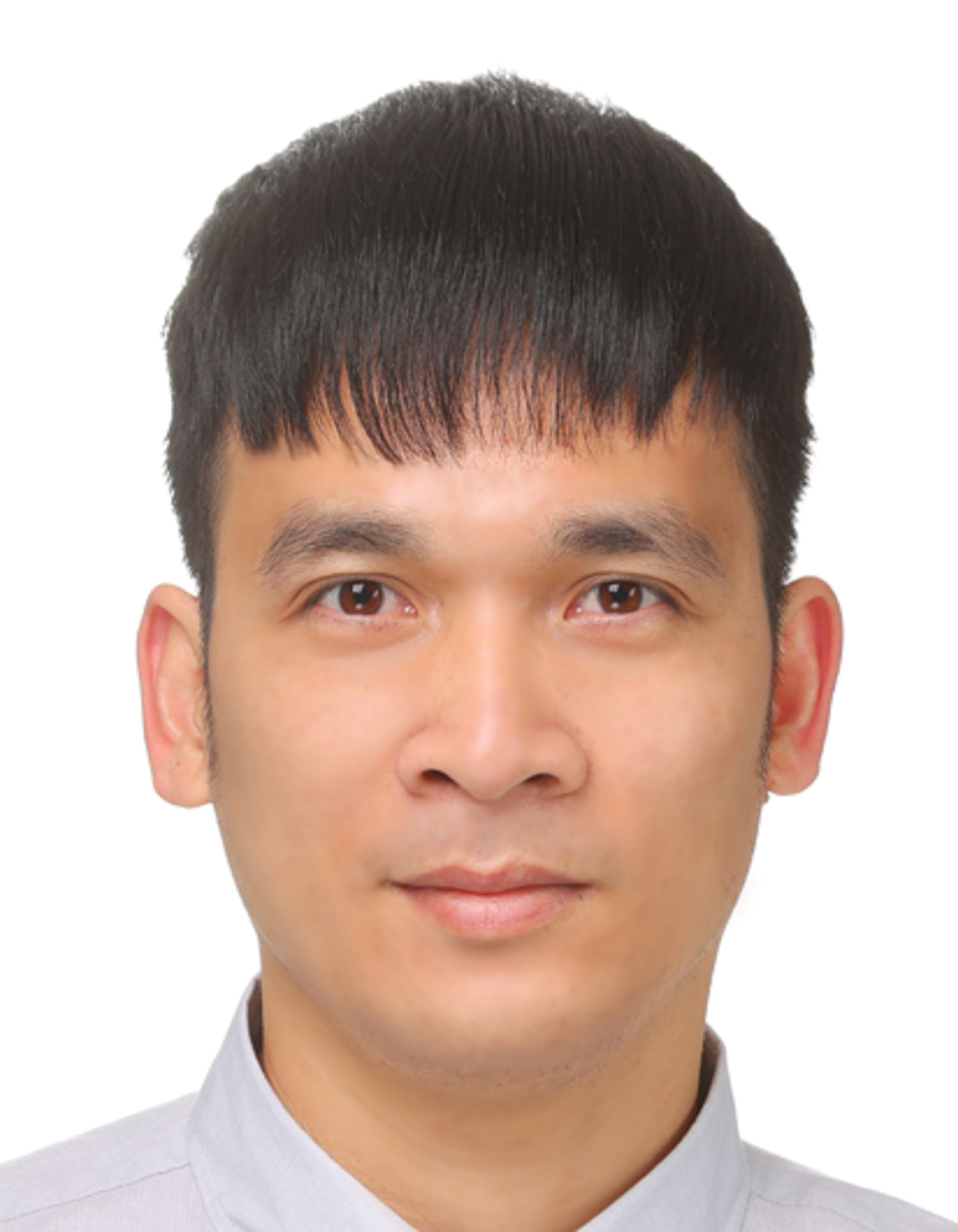}}]{Quoc-Viet Pham} (Senior Member, IEEE) is currently an Assistant Professor in Networks and Distributed Systems at the School of Computer Science and Statistics, Trinity College Dublin, Ireland. He earned his BSc and PhD degrees (with the Best PhD Dissertation Award) in Telecommunications from Hanoi University of Science and Technology and Inje University in 2013 and 2017, respectively.

He has special research interests in the areas of network AI, distributed machine learning, quantum AI, and data privacy in distributed systems. He was honoured with the IEEE ComSoc Best Young Researcher Award for EMEA 2023 and the Web-of-Science Highly Cited Researcher Award 2024 in recognition of his research excellence and broad influence. He currently serves as an Area Editor for IEEE Communications Surveys \& Tutorials, and as an (Associate) Editor for IEEE Communications Letters, IEEE Communications Standards Magazine, IEEE Transactions on Mobile Computing, and IEEE Transactions on Network Science and Engineering.
\end{IEEEbiography}

\end{document}

%% file: usage_package.tex
\usepackage{dsfont}
\usepackage{algpseudocode}
\usepackage{amsthm, xpatch}
\usepackage[thinc]{esdiff}
\usepackage{hyperref}
\usepackage{subcaption}
\usepackage{url}
\usepackage{cite}
\usepackage{xcolor}
\usepackage[nolist]{acronym}
\usepackage{graphicx, graphics} 
\usepackage{float}
\usepackage{threeparttable}
\usepackage{cancel}
\usepackage[official]{eurosym}
\usepackage{fancyhdr}
\usepackage{comment}
\usepackage{chemformula}
\usepackage{setspace}
\usepackage[T1]{fontenc} 
\usepackage{amsmath, mathtools, tcolorbox}
\usepackage{cleveref}
\usepackage[cmintegrals]{newtxmath}
\usepackage{bm} 
\usepackage{amsmath,amsfonts}
\usepackage{array}
\usepackage{textcomp}
\usepackage{stfloats}
\usepackage{url}
\usepackage{verbatim}
\def\BibTeX{{\rm B\kern-.05em{\sc i\kern-.025em b}\kern-.08em
    T\kern-.1667em\lower.7ex\hbox{E}\kern-.125emX}}
\usepackage{balance}
\usepackage{diagbox}
\usepackage[linesnumbered,ruled,vlined]{algorithm2e}
\usepackage{pdfrender}
\usepackage{xcolor}
\usepackage{glossaries}
\usepackage{makecell}
\usepackage{tabularray}
\usepackage[usestackEOL]{stackengine}
\usepackage{array, multirow, cellspace}
\usepackage{pifont}
\usepackage{lipsum, tabularx, ragged2e, multirow}
\usepackage{hyperref}
\usepackage[utf8]{inputenc}
\usepackage{enumitem}
\usepackage{wrapfig}
\usepackage{color, colortbl,hhline}
\usepackage{flushend}
\usepackage{balance}
\usepackage{wasysym}
\usepackage{booktabs}

\usepackage[switch, columnwise]{lineno}
\usepackage{blindtext}

\def\({\left(}
\def\){\right)}

\def\[{\left[}
\def\]{\right]}

\usepackage{glossaries}

\hypersetup{
     colorlinks = true,
     linkcolor = blue,
     anchorcolor = black,
     citecolor = blue,
     filecolor = blue,
     urlcolor = black
}

%% file: Sections_revision_wt_trackchanges/1_Introduction.tex
\section{Introduction}
\subsection{Background and Motivations}
\IEEEPARstart{F}ederated learning (FL) is a privacy-preserving distributed machine learning (ML) paradigm, where distributed clients (e.g., devices or organizations) collaboratively train a shared artificial intelligence model under the coordination of a central server without disclosing private data \cite{mcmahan2017communication, nguyen2025federated, nguyen2026exploiting}. In classical FL settings, the importance of all training rounds and the efforts of distributed clients across these rounds are weighted equally. However, this setting is questioned by a recent finding in FL known as critical learning periods (\text{\fontfamily{qcr}\selectfont CLPs}), which highlight the varying impacts of different training rounds and significantly affect FL learning and training efficiency \cite{yan2022seizing, huang2025pa3fed}. This finding was inspired by earlier research on \text{\fontfamily{qcr}\selectfont CLPs} in species' cognitive and learning functions (i.e., humans and animals) and deep learning networks within a centralized setting \cite{kandel2000principles, achille2018critical, jastrzebski2021catastrophic, golatkar2019time}. Particularly, \text{\fontfamily{qcr}\selectfont CLPs} are early learning phases (cf. Fig. \ref{fig:clp_illustration}(\subref{fig:clp_illustration1})) affecting long-term learning capabilities \cite{hensch2005critical}. Empirical studies show that the absence of proper learning during these periods can lead to irreversible deficits. For instance, barn owls exposed to misaligned auditory and visual cues during their critical development periods fail to properly localize spatial locations in adulthood \cite{knudsen1990sensitive}. This phenomenon was also observed in deep neural networks, in which any deficits and low effort (e.g., sparse data availability, late client participation) during this initial learning stage can cause lasting impairment on model performance, regardless of subsequent efforts (cf., Fig. \ref{fig:clp_illustration}(\subref{fig:clp_illustration2})) \cite{kleinman2023critical, yan2022seizing}. The reason for the phenomenon is still unclear, but possible explanations include the depth of the learning model, the structure of training data distribution, or defects in the implementation and training process \cite{achille2018critical, jastrzebski2021catastrophic, golatkar2019time}.

\begin{figure}[t!]
\centering
\begin{subfigure}[t]{0.24\textwidth}
    \centering
    \includegraphics[width=\linewidth]{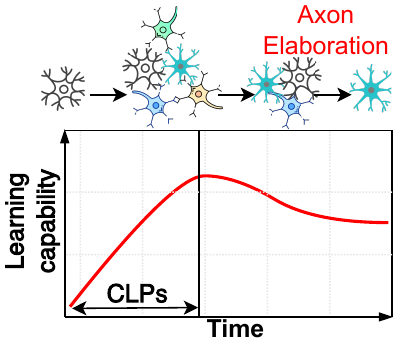}
    \caption{}
    \label{fig:clp_illustration1}
\end{subfigure}
\begin{subfigure}[t]{0.24\textwidth}
    \centering
    \includegraphics[width=\linewidth]{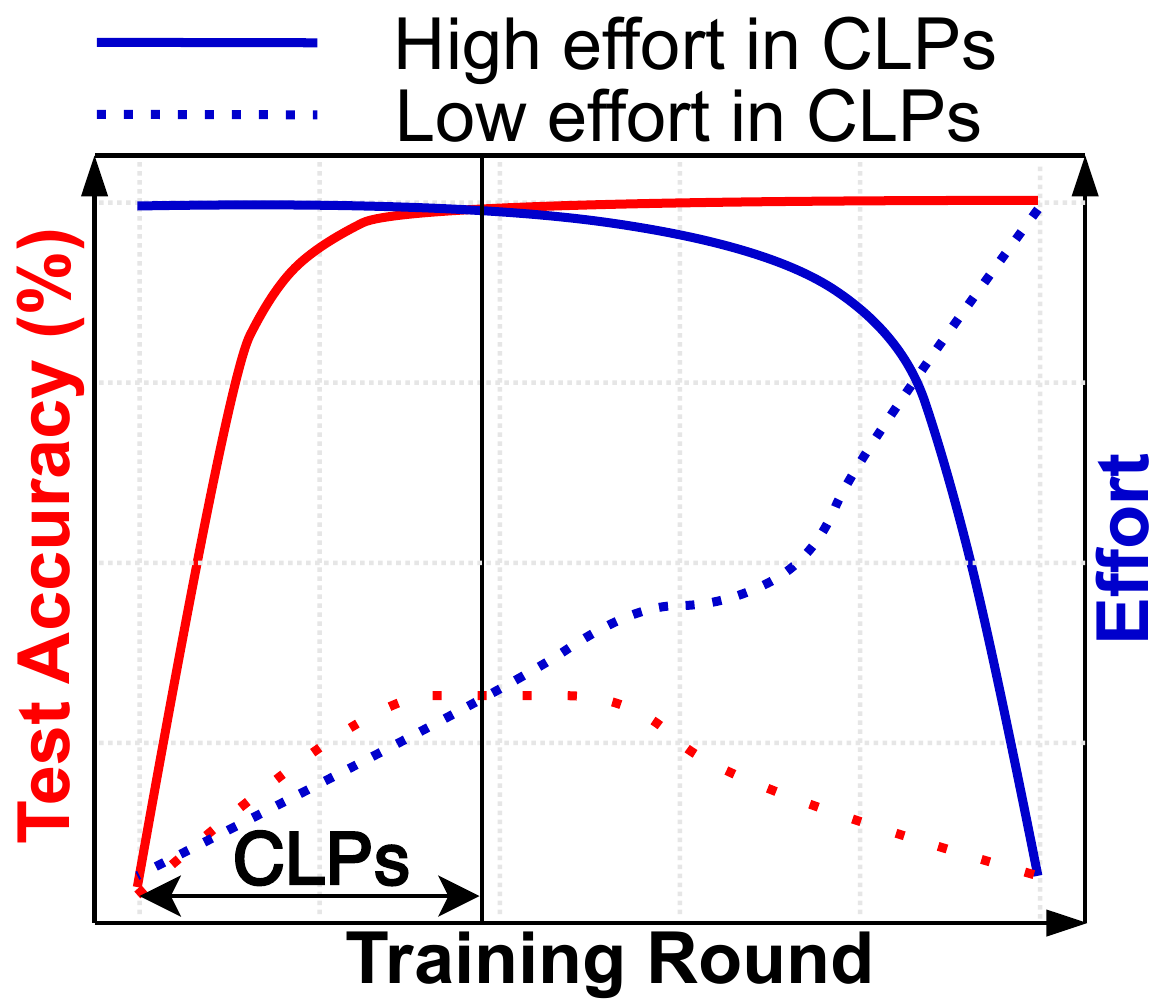}  
    \caption{}
    \label{fig:clp_illustration2}
\end{subfigure}
    \caption{Illustration of: a) \text{\fontfamily{qcr}\selectfont CLPs} in the biological field \cite{hensch2004critical}, and b) how early effort affects FL global model performance \cite{yan2022seizing}.}
\label{fig:clp_illustration}
\end{figure}

Despite the promising merits of enhancing FL efficiency (e.g., accuracy and convergence rate), \text{\fontfamily{qcr}\selectfont CLPs} remain under-explored, including a lack of transparent, fair, and timely incentives for participation and contribution elicitation, and information asymmetry among clients and the global model owner (e.g., a cloud server\footnote{The cloud server and the cloud are used interchangeably.}) during these periods. Existing {\fontfamily{qcr}\selectfont CLP}-considered works in FL, such as {\clp} detection, client selection, or  {\fontfamily{qcr}\selectfont CLP}-aware aggregation,  are based on a common assumption of voluntary engagement of clients in the training process (e.g., \cite{yan2022seizing, yan2023criticalfl, fedprime, huang2025pa3fed, 11060892}). {However, this is unrealistic in practice because clients, acting as \textit{rational}, \textit{strategic}, and \textit{benefit-driven} agents, participate only when they receive adequate compensation for their efforts (e.g., data, computation resources, time costs, and potential privacy leakage risks). Moreover, the learned global model is considered a non-excludable public good, meaning that clients benefit from the model regardless of their individual effort levels~\cite{10418498, anonymous2026incentives}. Consequently, this may incentivize free-riding and weaken truthful efforts, which can degrade global model quality. The resulting performance deterioration can, in turn, impair service quality and user experience for downstream consumers of the model (e.g., data and model marketplaces~\cite{9025575}), threatening the long-term sustainability of FL deployments and services}. Although several works have been proposed to bring incentivization by considering effort-reward approaches to fairly compensate FL clients' effort (e.g., \cite{anonymous2026incentives, murhekar2024you, kang2019incentive, 10097767, chen2024fdfl, 10323190}), \textit{they treat clients' efforts and their impacts on the global model performance on every training round equally} and \textit{ignore the temporal factor}, whereas efforts during \text{\fontfamily{qcr}\selectfont CLPs} have more weights and mainly contribute to the overall performance of the global model. By failing to prioritize {\fontfamily{qcr}\selectfont CLP}-driven incentive solutions, current methods inefficiently allocate cloud rewards (e.g., money), thereby reducing cloud utility—defined as the difference between gains from global model performance and compensation costs (cf.~Section~\ref{sec:system_model}).

Furthermore, the design of efficient incentives is hindered by information asymmetry between the FL clients and the cloud server. Specifically, due to privacy restrictions, the cloud lacks prior knowledge of the clients' private training capabilities, including the data, communication, and computation resources, efforts (e.g., data quantity and distribution), or available joining time. Therefore, assumptions that the cloud server knows precise information \cite{10539084, zhan2020learning, sarikaya2019motivating, murhekar2023incentives}, such as clients' computing resources or clients' data size, may not hold in practice. {\textit{A compounded challenge overlooked by prior works is that existing contract-theoretic incentive approaches, while effective for handling information asymmetry in FL incentives (e.g., \cite{sun2021pain}), fail to incorporate the temporal heterogeneity of {\fontfamily{qcr}\selectfont CLPs}, where early deficits cause irreversible model impairments regardless of later efforts \cite{yan2022seizing, huang2025pa3fed}}. This leads to suboptimal contract designs that treat all training rounds equally, inefficiently allocating incentives without prioritizing high-quality contributions\footnote{We refer to\textbf{ high-quality contributions} as high training efforts, such as large amounts of closer-distribution-to-the-learning-task data, together with substantial computational resources, available in {\fontfamily{qcr}\selectfont CLPs}.} during {\clps}, resulting in diluted contributions and permanent performance losses under asymmetry about clients' training capabilities.}

\textit{There is also uncertainty regarding whether clients will exert the promised effort post-contract} \cite{bolton2004contract}. {Even with contractually specified incentives, clients may engage in moral-hazard behaviors (e.g., low-effort training or strategic dropouts), and the cloud may be accused of delaying or reneging on payments, which can lead to disputes. Addressing these issues requires costly monitoring and verification and remains error-prone, increasing the difficulty of sustaining FL and implementing fair incentive mechanisms \cite{10097767}.} Motivated by the aforementioned research gaps and challenges, in this article, we aim to address the following research question (RQ).

\begin{boxIRQ}
\noindent \underline{\textbf{RQ}}: \textit{How can we design an incentive mechanism to attract rational and benefit-driven clients with the highest-quality contributions to improve FL learning and training efficiency, while simultaneously mitigating information asymmetry between clients and the cloud, given the crucial role of \text{\fontfamily{qcr}\selectfont CLPs} in determining the final FL model performance?}
\end{boxIRQ}

To answer the RQ, and inspired by the spirit of \textit{carpe diem} (i.e., seizing the right moment), we propose a \textit{{\fontfamily{qcr}\selectfont CLP}-aware} contract-theoretic incentive framework, named \underline{R}ight \underline{R}eward \underline{R}ight \underline{T}ime ({\fontfamily{qcr}\selectfont R3T}). Leveraging contract theory, {\fontfamily{qcr}\selectfont R3T} effectively mitigates the information asymmetry issues by designing a set of temporally coupled contract items offered by the cloud server~\cite{bolton2004contract}. The resulting optimal contract induces clients to self-reveal their training capabilities, join early, and exert efforts by assigning higher priority and appropriately larger rewards to early-stage participation under both complete and incomplete information cases, while satisfying individual rationality, incentive compatibility, and budget feasibility constraints. By delivering the right reward at the right time, {\fontfamily{qcr}\selectfont R3T} attracts the highest-quality clients to join early and contribute efforts during {\fontfamily{qcr}\selectfont CLPs}, where these efforts are most consequential for final model performance. {\fontfamily{qcr}\selectfont R3T} also integrates with blockchain smart contracts to decentralize incentive governance and reduce disputes caused by relying solely on the cloud \cite{wang2022infedge, yu2023ironforge, 3718082}.

\subsection{Contributions}
The main contributions of this paper are as follows:
 
\begin{itemize}   
\renewcommand{\labelitemi}{\textcolor{black}{$\blacktriangleright$}}

    \item \textit{Temporally integrated contract-theoretic incentive framework:} We develop a \textit{time-aware} contract-theoretic framework, named  {\fontfamily{qcr}\selectfont R3T}, which is the first to explicitly integrate {\fontfamily{qcr}\selectfont CLPs} into FL incentive design. {\fontfamily{qcr}\selectfont R3T} characterizes cloud utility as a function of client time and system capability, effort, joining time, and reward decisions. 
    
    \item  \textit{{\fontfamily{qcr}\selectfont CLP}-aware incentive mechanism design:} We develop a {\fontfamily{qcr}\selectfont CLP}-aware incentive mechanism that analytically derives an optimal contract to jointly optimize client participation strategies and cloud utility while adaptively attracting high-quality rational clients in the early training stages to enhance FL training and learning efficiency. The proposed incentive mechanism provides the right reward at the right time and handles issues of information asymmetries between clients and the cloud.

    \item \textit{Performance evaluation:} We conduct simulations to verify the effectiveness of our {\fontfamily{qcr}\selectfont R3T} in information symmetry and information asymmetry cases under individual rationality, incentive compatibility, and budget feasibility constraints. {\fontfamily{qcr}\selectfont R3T} demonstrates feasibility and superior efficiency compared with conventional incentive mechanism benchmarks, including linear pricing and conventional contract theory-based methods.

    \item \textit{Proof of concept:} We implement a proof of concept of {\fontfamily{qcr}\selectfont R3T} and demonstrate the feasibility and efficiency of the proposed mechanism and system design. Significantly, we show that {\fontfamily{qcr}\selectfont R3T} can boost convergence speed by up to 300\% while achieving competitive final FL model accuracies on standardized benchmark datasets, CIFAR-10 and Fashion-MNIST, compared with the state-of-the-art incentive mechanism benchmarks.
\end{itemize}

The remainder of this paper is organized as follows. Section~\ref{sec:relatedwork} reviews the related work, followed by the system model in Section \ref{sec:system_model}. Sections~\ref{sec:optimal_incentive_design} and \ref{sec:experiment} present the proposed game-theoretic {\fontfamily{qcr}\selectfont R3T} framework and its experimental results, respectively. Section \ref{sec:conclusion} concludes this paper.

%% file: Sections_revision_wt_trackchanges/2_1_Related_work.tex

\section{Related Work}
\label{sec:relatedwork}
This section reviews relevant background and advances in {\fontfamily{qcr}\selectfont CLPs} for artificial neural networks and incentive mechanisms for FL. The intersectional gap
between these domains motivates our research contribution. The comparison of prior research and this work is shown in Table~\ref{tab:comparison}.

\subsection{Critical Learning Periods in Centralized ML and FL} 

A significant amount of research has examined {\clps}, including their effects on training efficiency and learning generalization, their interactions with optimization hyperparameters, and methods to detect and exploit them in both centralized and distributed learning settings~\cite{achille2018critical, kleinman2024critical, fukase2025one, yan2022seizing, yan2023criticalfl, fedprime, yan2023defl, 11302878, 11275875, 10802950, ecolearnclps, 10737242, li2023revisiting}. 

In centralized deep learning, Achille et al. in \cite{achille2018critical} empirically linked {\clps }, especially in the memorization phase, to the downstream performance. Their findings demonstrated that, analogous to animals, sensory deficits in {\fontfamily{qcr}\selectfont CLPs} cause lasting impairments to the learning model network's information and connectivity, as quantified by an approximation of Fisher Information, and adversely affect learning outcomes, regardless of subsequent additional training. Complementing this empirical evidence, Kleinman et al. in \cite{kleinman2024critical} provided analytical insights indicating that the deep neural network model depth and the structure of the data distribution are fundamental sources of {\fontfamily{qcr}\selectfont CLPs} in deep networks, not solely explained by biological aging processes observed in animals. More recently, Fukase et al. in \cite{fukase2025one} studied a practical problem of identifying {\clps } and proposed a detection approach based on the cosine distance between the initial and final model weights.

\input{Tables/table_of_comparison}
Parallel findings have been reported in FL. Yan et al. in \cite{yan2022seizing, yan2023criticalfl} showed that FL also exhibits {\fontfamily{qcr}\selectfont CLPs} that disproportionately shape final model performance under statistical data heterogeneity and system heterogeneity. To characterize and detect these phases, they introduced {\clp}-tracking metrics such as the federated Fisher Information matrix and {\fontfamily{qcr}\selectfont CLP}-aware federated norm vector. Building on these foundations, subsequent studies \cite{fedprime, ecolearnclps, 11275875} leveraged {\clp} awareness for system-level optimization. For instance, \clp-guided client selection has been developed using local client utilities that capture factors such as energy consumption, estimated contribution effort to the global model, and a tailored {\clp} detection has been proposed to selectively upload updates when selected clients are inferred to operate within {\clps}~\cite{fedprime}. Furthermore, authors in \cite{10737242, 11302878} extended {\clp}-aware designs to asynchronous and mobile FL by dynamically scheduling updates and adjusting training hyperparameters to mitigate global model generalization degradation caused by test-data shifts and temporal class imbalances. Li et al. in \cite{li2023revisiting} improved the global model generalization and mitigated heterogeneity by studying the link of {\clps} and their proposed adaptive aggregation, which upweights more coherent clients during early training and stops the learned reweighting around the end of {\clps}.

\textit{Despite these advancements, the majority of {\fontfamily{qcr}\selectfont CLP}-aware FL studies implicitly assume voluntary participation and largely overlook costs and benefit-driven and strategic behavior of rational clients. In practical FL deployments, clients may withhold data, communication, and computational resources, or strategically degrade model update quality, if participation costs are not offset by adequate incentives. Consequently, without explicit {\fontfamily{qcr}\selectfont CLP}-targeted incentive mechanisms, the FL systems may fail to elicit high-quality contributions required during {\fontfamily{qcr}\selectfont CLPs}, which prior evidence has shown to result in long-term irrecoverable impairment on the global model performance, despite subsequent efforts.}

\subsection{Incentive Mechanisms for Federated Learning} 

Incentive mechanisms have been widely incorporated into FL to stimulate participation and sustain cooperation among rational and benefit-driven clients, thereby improving convergence efficiency and the resulting global model quality~\cite{10323190, anonymous2026incentives, 10739914, 10539084, 10264195, 9497718, han2024dynamic, sun2021pain, ding2023joint, 10557146, 10418498, wu2024incentive, 11363323, murhekar2023incentives, 103714961, 10586269, chen2024fdfl, yuan2023tradefl, le2025applications}. 

A dominant line of work adopts game-theoretic formulations, which provide a principled framework for modeling strategic interactions and utility-maximizing decisions of distributed clients and the cloud server in FL (e.g., \cite{10264195, 9497718, han2024dynamic, sun2021pain, ding2023joint, 10557146, 10418498}). Specifically, authors in \cite{ding2023joint, 103714961, 10323190, 10264195, 9497718, han2024dynamic, 10586269, 10739914} designed leader-follower-based incentive mechanisms to encourage contribution in one and multiple training rounds, where the cloud server first announces the total reward or price per contribution unit, and then clients determine effort strategies (e.g., data volume or computing power for training global model or cleaning noisy local data samples) to maximize individual utilities and system-wide social welfare \cite{murhekar2023incentives}. \textit{However, many of these formulations rely on simplified information assumptions and do not explicitly model server-client information asymmetries, which is exacerbated by data-privacy regulations that restrict the disclosure of clients’ local data characteristics (e.g., data volume, quality, and privacy costs)}. To handle incomplete information, contract-theoretic designs exploit the self-revelation principle~\cite{bolton2004contract}. For instance, Sun et al. in \cite{sun2021pain} constructed contract items based on privacy-leakage costs under both complete and incomplete information to compensate privacy loss while preserving learning performance, and Xu et al. in \cite{11363323} further improved contract-based schemes by differentiating client contributions in the case clients have the same contract type, thereby enhancing the fairness. Under the same information conditions, Huang et al. in \cite{chen2024fdfl, 10557146, anonymous2026incentives, 10539084, 10586269} proposed contribution-based incentive mechanisms, where the client's data size and/or data quality are incorporated, to derive optimal clients' contributions that minimize the cloud server's costs. Tang et al. in \cite{10418498} proposed a non-cooperative game-based incentive mechanism to model the interactions between clients (i.e., organizations), deriving optimal processing capacity of each client, to address free-rider problems of public goods in FL under constraints of individual rationality and budget balance. Wu et al. in \cite{wu2024incentive} targeted problems of reward timeliness and monetary reward infeasibility in FL incentivization by providing a proportion of local model updates as rewards proportional to the client's per-round contributions.

In order to improve transparency and reduce centralization-induced disputes, blockchain has also been integrated with FL to support auditable incentive computation and automated distribution~\cite{wang2023incentive, han2024dynamic, yuan2023tradefl, 10418498, 10539084, nguyen2021federated}. For instance, Yuan et al. in \cite{yuan2023tradefl} leveraged Ethereum smart contracts to autonomously execute an incentive mechanism that explicitly accounts for competitive effects and provides fair and non-repudiative compensation, to enhance social welfare. 

\textit{Nevertheless, existing incentive designs largely treat all training rounds as equally valuable and, consequently, do not yet address how to deliver the right reward at the right time. In particular, none of these works considers the presence of {\fontfamily{qcr}\selectfont CLPs}, which have been shown to shape final FL model performance; nor do they analyze how information asymmetries between clients and the cloud server interact with {\fontfamily{qcr}\selectfont CLPs} to influence incentive effectiveness and final learning outcomes.}

To the best of our knowledge, {\fontfamily{qcr}\selectfont R3T} is the first {\fontfamily{qcr}\selectfont CLP}-aware incentive mechanism that jointly accounts for client joining time and contribution effort, to recruit rational and benefit-driven clients with the highest-quality contributions during early learning stages, while mitigating information asymmetry between clients and the cloud server. {\fontfamily{qcr}\selectfont R3T} enables transparent and autonomous reward computation and distribution.

%% file: Tables/table_of_comparison.tex
\begin{table}
    \centering
    \small
    \caption{Comparison between prior research \& {\fontfamily{qcr}\selectfont R3T}. \\ $\CIRCLE$, $\LEFTcircle$, and $\Circle$ indicate that the topic is \textit{covered}, \textit{partially or indirectly covered}, and \textit{uncovered}, respectively.
}
    \label{tab:comparison}
    \begin{tabular}{|c|c|c|c|c|c|c|c|}
    \rowcolor{gray!20}
    \hline
        \textbf{Work} & \textbf{Venue} & \textbf{Year}  & \textbf{F1} & \textbf{F2} & \textbf{F3} & \textbf{F4} & \textbf{F5} \\\hline 
        \rowcolor{maroon!10}\hline
        \multicolumn{8}{|c|}{\textit{Critical Learning Periods in Centralized and FL}} \\\hline 

         \cite{achille2018critical} & ICLR & 2018 &$\Circle$  & $\Circle$  & $\CIRCLE$ & $\Circle$ & $\Circle$ \\\hline 
         
         \cite{yan2022seizing} & AAAI & 2022 &$\Circle$  & $\Circle$  & $\CIRCLE$ & $\Circle$ & $\Circle$ \\\hline 

         \cite{li2023revisiting} & ICML & 2023 & $\Circle$  & $\Circle$  & $\CIRCLE$ & $\Circle$ & $\Circle$ \\\hline 

         \cite{yan2023criticalfl} & SIGKDD &  2023 & $\Circle$  & $\Circle$  & $\CIRCLE$ & $\Circle$ & $\Circle$\\\hline 

         \cite{fedprime} & PKDD & 2024 & $\Circle$  & $\Circle$  & $\CIRCLE$ & $\Circle$ & $\Circle$ \\\hline 

         \cite{11302878} &  TMC & 2025 & $\Circle$  & $\Circle$  & $\CIRCLE$ & $\CIRCLE$ & $\Circle$ \\\hline 
         
         \cite{11275875}& TMC & 2025 & $\Circle$  & $\Circle$  & $\CIRCLE$ & $\Circle$ & $\Circle$ \\\hline 

         \cite{10737242}& TMC & 2025 & $\Circle$  & $\Circle$  & $\CIRCLE$ & $\Circle$ & $\Circle$ \\\hline          
        \rowcolor{maroon!10}\hline
        \multicolumn{8}{|c|}{\textit{Incentive Mechanisms for FL}} \\\hline

        \cite{murhekar2023incentives}  & NEURIPS & 2023 & $\CIRCLE$ &  $\Circle$ & $\Circle$ & $\Circle$ & $\Circle$ \\\hline 
        
        \cite{yuan2023tradefl} & ICDCS & 2023 & $\LEFTcircle$  & $\Circle$ & $\Circle$ & $\CIRCLE$& $\CIRCLE$\\\hline 

         \cite{10323190} & TNSE & 2024 & $\CIRCLE$ &  $\Circle$ & $\Circle$ & $\Circle$ & $\Circle$ \\\hline

        \cite{10539084} & TNSE & 2024 & $\CIRCLE$ & $\CIRCLE$ & $\Circle$ & $\Circle$ & $\LEFTcircle$ \\\hline 
        
        \cite{wu2024incentive} & ICLR & 2024 & $\CIRCLE$  & $\Circle$ & $\Circle$ & $\Circle$ & $\Circle$ \\\hline 
        
        \cite{10418498} &  TMC & 2024 & $\CIRCLE$ & $\Circle$   &  $\Circle$  &  $\CIRCLE$   & $\CIRCLE$ \\\hline 

         \cite{chen2024fdfl} & TIFS & 2024 & $\CIRCLE$ & $\Circle$ & $\Circle$ & $\LEFTcircle$ & $\Circle$ \\\hline 
         
        \cite{10739914} & TNSE & 2025 & $\CIRCLE$ &  $\Circle$ & $\Circle$ & $\Circle$ & $\Circle$ \\\hline 
        
        \cite{103714961} & WWW & 2025 & $\CIRCLE$ & $\Circle$ & $\Circle$  & $\LEFTcircle$ & $\Circle$  \\\hline 
        
        \cite{11363323} & TNSM & 2026 & $\CIRCLE$ & $\CIRCLE$ & $\Circle$  & $\Circle$ & $\Circle$ \\\hline 
        
         \cite{anonymous2026incentives} & ICLR & 2026 & $\CIRCLE$  &  $\CIRCLE$ & $\Circle$ & $\Circle$ & $\Circle$\\\hline 
         
         \textbf{{\fontfamily{qcr}\selectfont R3T}} & - & 2026 & $\CIRCLE$  & $\CIRCLE$  &  $\CIRCLE$ & $\CIRCLE$ & $\CIRCLE$ \\\hline 
    \end{tabular}
\begin{minipage}{0.95\linewidth} ~\\
\footnotesize	
     \textit{\textbf{F1}: Incentivization; \textbf{F2}: Information asymmetry; \textbf{F3}: {\clp } awareness; \textbf{F4}: Proof of concept; \textbf{F5}: decentralized incentive management.}
\end{minipage}
\end{table}

%% file: Sections_revision_wt_trackchanges/2_System_Model.tex
\section{System model }
\label{sec:system_model}
\subsection{System Overview}\label{subsec:sys_overview}
\begin{figure}[t!]
	\centering
	\includegraphics[width=\linewidth]{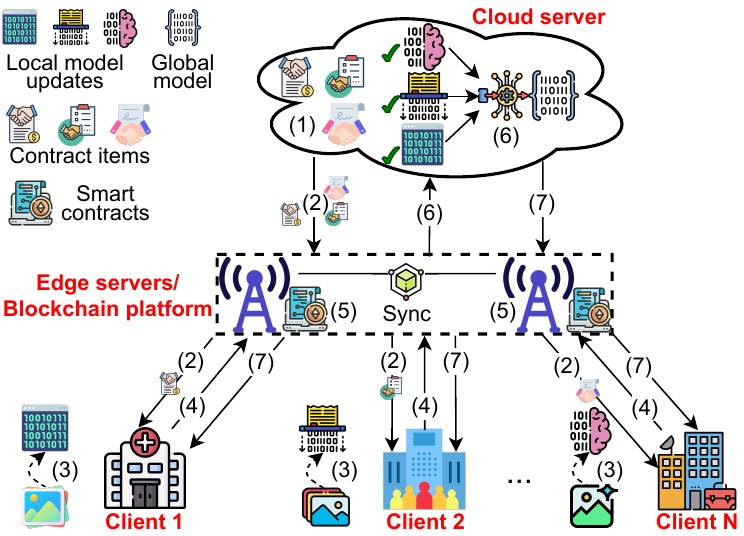}
	\caption{{\fontfamily{qcr}\selectfont R3T}'s one-round learning workflow: (1) prepare contracts, (2) sign contracts \& download global model, (3) local training, (4) upload updates, (5) forward \& validate updates, (6) receive \& aggregate model updates, and (7) settlement.}
    \label{fig:system_model}
\end{figure} 
As depicted in Fig.~\ref{fig:system_model}, the {\fontfamily{qcr}\selectfont R3T} framework consists of three main components: edge servers forming a blockchain platform, clients, and a cloud server. We denote a set of edge servers as $\mathcal{L}$, where each edge server $l \in \mathcal{L}$ connects with clients to exchange information. $\mathcal{L}$ forms, maintains, and operates a blockchain platform to record and validate information such as client identification (ID), local model updates, global models, joining time, and effort levels. Also, {\fontfamily{qcr}\selectfont R3T} integrates a distributed file storage system such as InterPlanetary File System (IPFS), to achieve privacy preservation, high availability, scalability, and low latency for data storage and retrieval \cite{jadav2022blockchain}. 

{{The FL process evolves over training rounds indexed by $t \in \mathcal{T}=\{1,\ldots,T\}$.}} The set of clients is denoted as $\mathcal{N}=\{1,...,n,...,N\}$, where $N$ represents the number of clients. Clients join and contribute their time, computing, and data resources to jointly train the global model published by the cloud server, in exchange for benefits such as monetary rewards. In each training round, a subset of clients $\mathcal{S}^{(t)}  \subseteq \mathcal{N}$ participates in improving the model performance (e.g., in the training round $t$, the cloud server randomly selects $|\mathcal{S}^{(t)}|$ clients out of $\mathcal{N}$). Besides, clients have the right to use or share their dataset $\mathcal{D}$, and the local loss function of a client is defined as  ${F}$.  In {\proposed}, the cloud initiates a set of global model parameters $\mathbf{w}$ and seeks high-quality clients to optimize the global model performance. It receives local model updates from clients, conducts model aggregation, and handles payments via a blockchain platform. 
{{For clarity, let $\mathbf{w}^{(t-1)}$ denote the global model parameter set available at
the beginning of round $t$. After local training
based on $\mathbf{w}^{(t-1)}$, each participating client $n$ produces an
updated local model, denoted by $\mathbf{w}_{n}^{(t)}$. The cloud server then aggregates the received local models to obtain the updated global model
$\mathbf{w}^{(t)}$ for the next round. Accordingly, $\mathbf{w}^{(0)}$ denotes
the initialized global model before the training process starts, and
$\mathbf{w}^{(T)}$ denotes the final global model after $T$ training rounds. Details of the {\proposed}'s learning workflow are discussed in the following section.}}

\subsection{Learning Workflow}\label{subsec:learning_workflow}
{\proposed} is designed to perform model training, storage, and reward calculation and allocation in a fair and auditable manner, leveraging blockchain and FL technologies. It integrates a {time-aware} contract theory-based mechanism to address right-time and right-reward allocation where {\clp} phenomenon exists. The {\proposed}'s \textbf{one-training-round} process involves the following steps. 
1) \textbf{Prepare contracts.} The cloud initiates {a set of global model parameters $\mathbf{w}^{(0)}$ at $t=0$}, designs a set of contract items, and distributes them to $N$ registered clients. 
2) \textbf{Sign contracts \& Download global model.} {From $t=1 \text{ to } t=T$}, each client signs a contract item to participate in the training process and downloads the global model {$\mathbf{w}^{(t-1)}$}. 
3) \textbf{Local training.} Clients train {$\mathbf{w}^{(t-1)}$} using their private data $\mathcal{D}$. 
4) \textbf{Upload updates.} Clients send their trained model updates { $\mathbf{w}_n^{(t)}$}, metadata (e.g., IDs, addresses, chosen contracts), and hashed trained model updates to their associated edge servers after training\footnote{We assume that clients are truthful in reporting their updates to the cloud. If required, the cloud can verify their updates using the Trusted Execution Environments \cite{zhang2020enabling} or reputation systems \cite{yu2024mathcal}. Besides, given the traceability property of the underlying blockchain \cite{yu2023ironforge} and incentive mechanism design (e.g., incentive compatibility) presented later in Section~\ref{sec:optimal_incentive_design}, clients are encouraged to truthfully report to benefit from joining and contributing.}. 
5) \textbf{Forward \& Validate updates.} Edges relay trained model updates to the cloud to minimize latency in the model aggregation process, and cross-validate hashed trained model updates and metadata to ensure data integrity and maintain consensus, storing records on the blockchain for transparency and dispute arbitration.
6) \textbf{Receive \& Aggregate model updates\footnote{To avoid the straggler effect, we employ a synchronous learning update scheme, which has provable convergence \cite{chen2019round, kairouz2021advances}.}.} The cloud aggregates them {to produce an updated global model $\mathbf{w}^{(t)}$} using an aggregation algorithm such as {FedAvg} \cite{mcmahan2017communication}. 
7) \textbf{Settlement.} The process iterates until the global model converges. Upon completion, smart contracts calculate and distribute rewards to clients to compensate for their contributions. Edge servers also receive rewards for validating transactions and storing data.

\subsection{Time Frame, Client Type, Contract, and Strategy}\label{subsec:contract}
\paragraphnumstyle
\setcounter{paragraph}{0}
\paragraph{Time Frame} The time frame is divided into $T$ training rounds\footnote{In this work, training rounds and time slots are used interchangeably.}, denoted as $\mathcal{T} = \{1,..., t,..., T \}$, in which \text{\fontfamily{qcr}\selectfont CLPs} $\subset \mathcal{T}$ are vital periods in determining the final model performance.

{
\paragraph{Client Type}\label{para:client_type} Clients are heterogeneous in their ability to participate effectively in the training process due to differences in factors such as participation availability and system-level training capability. To capture this heterogeneity in a tractable manner, we classify $N$ clients into $K$ ordered latent types. We denote the training capability of client $n$ of type $k$ by $\theta_{n,k}$, or simply by $\theta_k$\footnote{For simplicity, we interchangeably use $\theta_{n,k}$ and $\theta_k$ to denote the capability of a type-$k$ client $n$. Similarly, the contract item $\phi_{k}$ can be denoted as $\{{e_{n,k}, t_{n,k}, R_{n,k}}\}$ or $\{{e_{k}, t_{k}, R_{k}}\}$. The same notation applies to the model parameter set $\mathbf{w}$, which is introduced later.}. {The type parameter $\theta_k$ is interpreted as an \textit{effective training capability index} that summarizes a client's overall ability to provide timely participation and high-quality contributions (cf. $e_k$ in the next subsection) in the training process.} Under privacy constraints, the cloud does not directly observe the raw attributes underlying $\theta_k$ for each client. Instead, it only has access to coarse statistical information about the type distribution through market research or consented surveys~\cite{10.1145/3209582.3209588}.

}
\paragraph{Cloud's Contract Menu}\label{para:contract_menu} To motivate clients to join early, share resources, and conduct training, the cloud must offer transparent and fair contribution-reward bundles, compensating clients for their efforts. Recognizing the importance of \text{\fontfamily{qcr}\selectfont CLPs}, the cloud offers extra bounties to attract high-quality, resource-rich clients to participate actively. Consequently, each client's gain is influenced by their chosen time of participation in the training process. The cloud has a contract set $\boldsymbol{\phi} = \{\phi_k\}_{k \in \mathcal{K}}$, which contains $K$ contract items corresponding to $K$ client types per training round\footnote{{Before announcing the contract set, the cloud designs the menu based on its prior estimate of the client-type distribution introduced earlier. This estimate may, in principle, be refined over repeated training rounds using observable contract-fulfillment outcomes, including realized joining time, completion or dropout behavior, and blockchain-validated metadata. Moreover, under the incentive compatibility in Constraint~\ref{cons:ic}, clients self-select the contract item aligned with their effective types, which can further improve the cloud’s type-distribution estimate. In this work, however, we assume that the type-distribution estimate remains fixed over the $T$ training rounds~\cite{sun2021pain, zhang2015contract, kang2019incentive, 11082438}.}}. The contract item $\phi_k = \{e_k, t_k, R_k\}$ specifies the relationship between type-$k$ client's effort, joining time, and reward. {{
Here, $e_k$ denotes the effective local training contribution exerted by a type-$k$ client in one training round. It is modeled as a normalized scalar that captures how much useful local training contribution (e.g., data volume) the client contributes to the global model improvement. Since the cloud cannot directly inspect raw local data or exact private client-side resources, effort is evaluated through observable and verifiable protocol-level signals, including the client’s participating round, selected contract item, and validated metadata, such as the local amount of contributed data, cryptographic hashes of model artifacts, and round participation record. Hence, the server assesses effort through contract-fulfillment evidence rather than direct access to private data}}. While $t_k$ is the training time slot, which is either in \text{\fontfamily{qcr}\selectfont CLPs} or in \text{\fontfamily{qcr}\selectfont non-CLPs}, $R_k$ is the corresponding reward for type-$k$ client paid by the cloud for finishing the training round, and is given as follows:
\begin{align}
\label{eq:define_reward}
R_k = r_k + B_k, 
B_k = \theta_k h(t_k) e_k,
\end{align}
\noindent where $r_k$ is the basic salary and $B_k$ is the bonus, which has been shown as an effective payment structure to motivate clients' efforts toward collective goals (i.e., enhancing global model performance) \cite{joseph1998role}. Specifically, clients are assured to have a base compensation of salary $r$ upon joining the training process. Based on their individual types, the timing of their participation, and the level of effort they contribute, bonuses are awarded. These bonuses direct clients participating in {\fontfamily{qcr}\selectfont CLPs} to contribute more efforts during this period, as well as providing additional rewards for the highest-contributing clients \cite{churchill1993sales}. Besides, inspired by the demand-price relationships in \cite{SELCUK20191191}, we propose a function $h(t_k)$ which is a time-aware bonus unit function (i.e., bonus-time relationship function), where the crucial role of \text{\fontfamily{qcr}\selectfont CLPs} is known by clients, and is presented as follows:
\begin{align}
\label{eq:time_aware_bonus}
h(t_k) = \left\{ \begin{array}{ll}
         1 + \frac{\vartheta}{\ln(2t_k)} & \mbox{if ${t_k} \in \text{\fontfamily{qcr}\selectfont CLPs}$};\\
        1 & \mbox{if ${t_k} \in \text{\fontfamily{qcr}\selectfont non-CLPs}$},\end{array} \right. 
\end{align}        
\noindent where $\vartheta$ represents the bonus unit coefficient. \label{logarithmic_func} {{The logarithmic form is chosen to provide a tractable {\clp}-aware incentive schedule. It assigns the largest bonus to the earliest rounds, decreases monotonically with training time, and exhibits diminishing marginal decay, so that later {\clp} rounds still retain non-negligible incentive value rather than being discounted too aggressively. Compared with linear and exponential schedules, this form provides a balanced decay pattern while preserving the monotonic structure required by constraints defined later in Section~\ref{subsec:formulation}. In contrast, a learned schedule may require additional training and validation data, and may not guarantee monotonicity without explicit extra constraints.}} This function offers high-quality clients bonuses when they choose to join the model training process at an early stage, given the importance of {\clps}.


{
\begin{figure}[t!]
	\centering
	\includegraphics[width=0.97\linewidth]{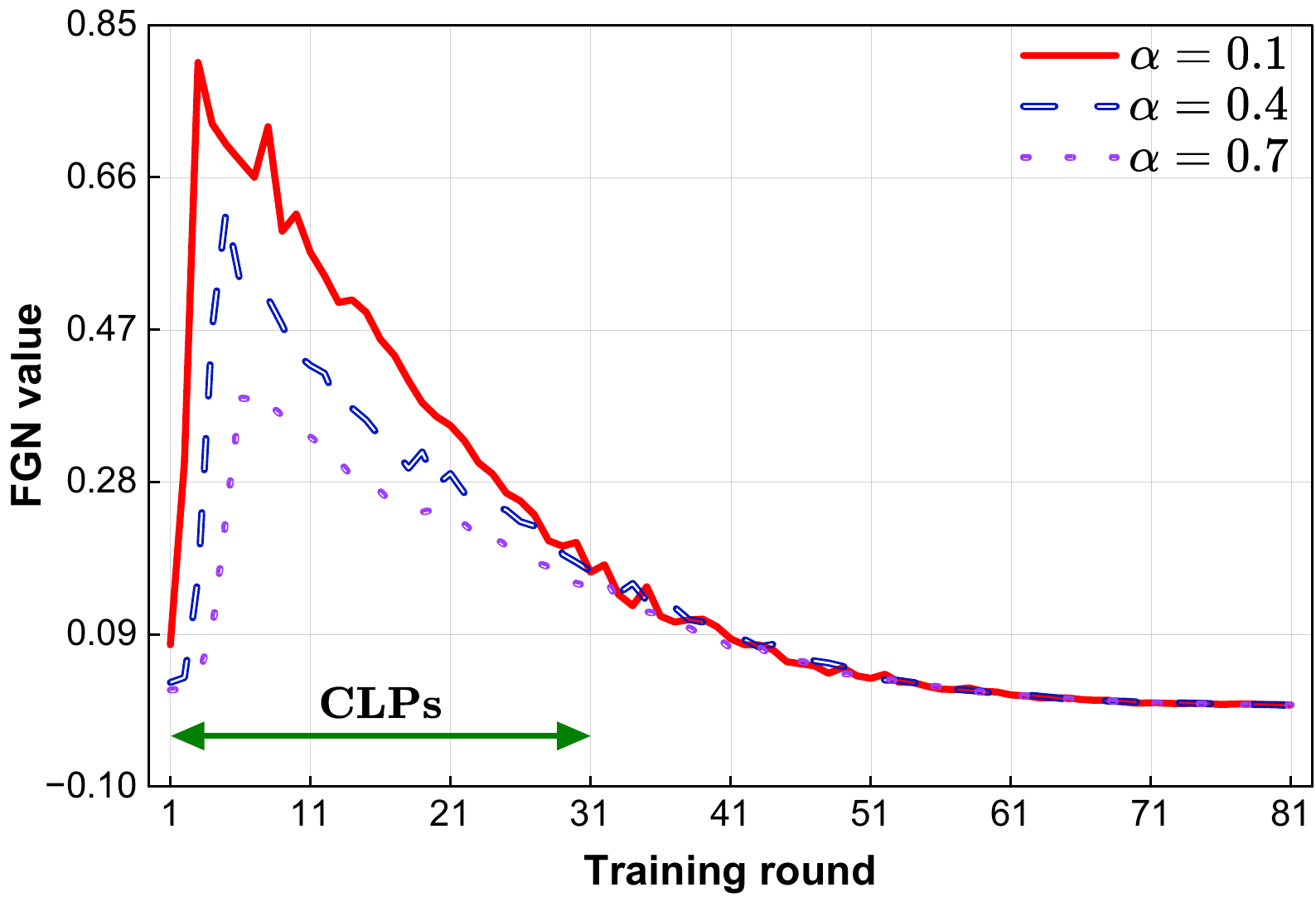}
	\caption{{Illustration of detecting {\clps} using FGN, where the double-arrow line indicates the identified {\clps}. The results are conducted on the CIFAR-10 dataset~\cite{krizhevsky2009learning}, which is non-IID partitioned using Dirichlet distributions $\textit{Dir}(0.1), \textit{Dir}(0.4), \textit{Dir}(0.7)$, respectively.}}
    \label{fig:clp_probing}
\end{figure} 

\paragraph{{\clp} Probing}{\label{para:clp_probing}} To determine whether training round $t$ belongs to the {\clps}, the cloud estimates the {\clp} status online using a federated gradient norm (FGN) proxy inspired by~\cite{yan2023criticalfl}. Specifically, for each participating client $n \in \mathcal{S}^{(t)}$, let
$\mathcal{B}_{n}^{(t)}=\{\xi_{n,1}^{(t)},\ldots,\xi_{n,B}^{(t)}\}$ denote a
mini-batch of size $B$ sampled at round $t$. The corresponding batch loss is
defined by  $\ell_{B,n}^{(t)}\!\left(\mathbf{w}^{(t-1)}\right)
=
\frac{1}{B}\sum_{b=1}^{B}
\ell\!\left(\mathbf{w}^{(t-1)};\xi_{n,b}^{(t)}\right),$
and its batch gradient is derived as $\mathbf{g}_{B,n}^{(t)}
=
\nabla \ell_{B,n}^{(t)}\!\left(\mathbf{w}^{(t-1)}\right)
=
\frac{1}{B}\sum_{b=1}^{B}
\nabla \ell\!\left(\mathbf{w}^{(t-1)};\xi_{n,b}^{(t)}\right).$
After one local gradient step with learning rate $\eta$, the batch-loss change
is approximated by the first-order Taylor expansion as:
\begin{align}
\Delta \ell_{B,n}^{(t)}
=
\ell_{B,n}^{(t)}\!\left(\mathbf{w}^{(t-1)}-\eta \mathbf{g}_{B,n}^{(t)}\right)
-
\ell_{B,n}^{(t)}\!\left(\mathbf{w}^{(t-1)}\right)
\approx
-\eta \left\|\mathbf{g}_{B,n}^{(t)}\right\|_2^2.
\end{align}
Accordingly, we define the round-$t$ FGN by:
\begin{align}
\label{eq:fgn}
\mathrm{FGN}(t)
=
-\eta\sum_{n \in \mathcal{S}^{(t)}} \omega_n^{(t)}
\left\|\mathbf{g}_{B,n}^{(t)}\right\|_2^2,
\quad
\omega_n^{(t)}
=
\frac{e_n^{(t)}}{\sum_{m \in \mathcal{S}^{(t)}} e_m^{(t)}},
\end{align}
where $\omega_n^{(t)}$ denotes the normalized contribution weight of client
$n$. Then, round $t$ is identified as belonging to the {\clps} if:
\begin{align}
\label{eq:fgn_condition}
\frac{\mathrm{FGN}(t)-\mathrm{FGN}(t-1)}{\mathrm{FGN}(t-1)} \geq \tau_{\clp}, 
\end{align}
where $\tau_{\clp}>0$ is a threshold controlling the sensitivity of
{\clp} detection, which is empirically determined \cite{yan2023criticalfl, huang2025pa3fed}. Empirically, the FGN trajectory typically increases during the early training stage and then gradually declines. Therefore, as depicted in Fig.~\ref{fig:clp_probing}, the initial phase provides an approximate estimate of the {\clps}, and the detected {\clp} status is subsequently used to evaluate the time-aware bonus unit function $h(t_k)$ in Eq.~\eqref{eq:time_aware_bonus}.}

\paragraph{Client's Strategy Space} Each client decides whether to participate in the training process (if \textit{yes}), how much effort they should exert, and which time slot to join. The choice of different time slots may result in varying bonuses. Particularly, participating during the \text{\fontfamily{qcr}\selectfont CLPs} yields higher payoffs for clients who contribute significant effort, as training the global model with huge effort during these periods is crucial for improving final model performance.

\subsection{Utilities}\label{subsec:uti}
According to \eqref{eq:define_reward}, $R_k$ is dependent on $e_k$, $t_k$, and $r_k$, thus each contract item is rewritten by $\phi_k = \{e_k, t_k, r_k\}$. {We define the utilities of each client, edge server, and cloud server per training round as follows.}

\paragraph{Client}\label{para:client_uti} The utility of type-$k$ client $n$ can be defined:
\begin{align}
\label{eq:client_utility}
 \mathcal{U}_n (e_k, t_k, r_k) &= R_k - \left(\frac{1}{2} \delta e^2_k + \beta r_k \right) \nonumber \\
& = {\theta}_k h(t_k) e_k - \frac{1}{2} \delta e^2_k - (\beta - 1)r_k,
\end{align}
where $\left(\frac{1}{2} \delta e_k^2 + \beta r_k \right)$ is the total contribution cost. This includes $\frac{1}{2} \delta e^2_k$, which is a quadratic cost function with respect to client's effort (e.g., sensing, collecting, and training of data or using cloud services for these tasks) according to \cite{bolton2004contract}, with $\delta > 0$ being a quadratic unit effort cost that varies on different learning tasks (e.g., generating data samples or training data samples). For example, training GPT-3 with 45 TB of compressed plaintext before filtering and 570 GB after filtering using 80 V100 GPUs costs more than \$2 million, or generating 100,000 data points using Amazon’s Mechanical Turk service will cost around \$70,000 \cite{brown2020language}. Additionally, $\beta r_k$ represents the time cost associated with joining time $t_k$, considered as a linear function of the reward, where $\beta > 1$ is the time cost coefficient \cite{bolton2004contract, wang2023connectivity}.

\paragraph{Edge Server} The edge server $l$’s utility is defined as $\mathcal{U}_l = {R_l} - v(f_l),$
where $R_l$ is the gain from validating data transactions, $v(f_l)$ is the validation cost function, and $f_l$ is the edge server’s computing resource (e.g., CPU-cycle frequency).

\paragraph{Cloud server} The cloud's utility is the difference between the gain (i.e., FL model performance) and costs (i.e., rewards for clients' contributions such as joining time and efforts in training the FL model, and for edge servers in validating data transactions) as follows:
\begin{align}
\label{eq:cloud_utility_1}
 \mathcal{U}_{c}(\boldsymbol{e}, \boldsymbol{t}, \boldsymbol{r}) = \lambda \sum_{k \in \mathcal{K}} g(h(t_k) e_k) - \sum_{k \in \mathcal{K}}{R_k} - \sum_{l \in \mathcal{L}}{R_l},
\end{align}
where $\lambda$ is a parameter adjusted by the cloud server to show its concerns for FL model performance (i.e., larger $\lambda$) or reward budget (i.e., smaller $\lambda$), $\boldsymbol{e} = \{e_{1}, \dots, e_{K}\}$, $\boldsymbol{t} = \{t_{1}, \dots, t_{K}\}$, and $\boldsymbol{r} = \{r_{1}, \dots, r_{K}\}$. {In each training round, multiple clients may select the same contract item $k$, and $R_k$ in \eqref{eq:cloud_utility_1} denotes the \emph{aggregate reward cost} incurred by the cloud for all participating clients choosing item $k$ in that round; this aggregate notation keeps \eqref{eq:cloud_utility_1} compact while remaining consistent with the budget feasibility constraint, which is defined in the next section.} Without loss of generality, $g(\cdot)$ is defined as a concave function \cite{murhekar2023incentives, kang2019incentive, zhou2019computation} with respect to the amount of effort and joining time. Specifically, with the same effort, each client can make different contributions and impacts on FL model performance at different training times.



\subsection{{\proposed} blockchain smart contracts}\label{subsec:blockchain}
{
\paragraph{Design-level threat model and blockchain scope}\label{para:threatmodel} {\proposed} employs blockchain as an accountability, integrity-verification, and reward-settlement layer for incentive management, rather than as a general-purpose defense against all FL security threats. We consider mutually distrustful but economically rational clients, edge servers, and the cloud. In this setting, clients may deny their selected contract items, joining rounds, or contributions, while the cloud may be accused of modifying contract items, denying valid submissions, delaying payment, or underpaying clients. Off-chain model artifacts may also be modified or become inconsistent with submitted records.

To address these incentive-related disputes, {\proposed} records contract selections, joining-round information, contribution metadata, off-chain storage references, cryptographic hashes of model artifacts, and reward-settlement events on-chain. These records provide tamper-evident evidence for contract fulfillment and dispute resolution. Active FL attacks \cite{yu2023ironforge}, such as model poisoning, Byzantine updates, and consensus-layer attacks, are outside the main scope of this work and can be addressed by complementary security mechanisms. Such defenses can be integrated with {\proposed} without changing the {\clp}-aware incentive design.}


\begin{figure}[t!]
	\centering
	\includegraphics[width=0.99\linewidth]{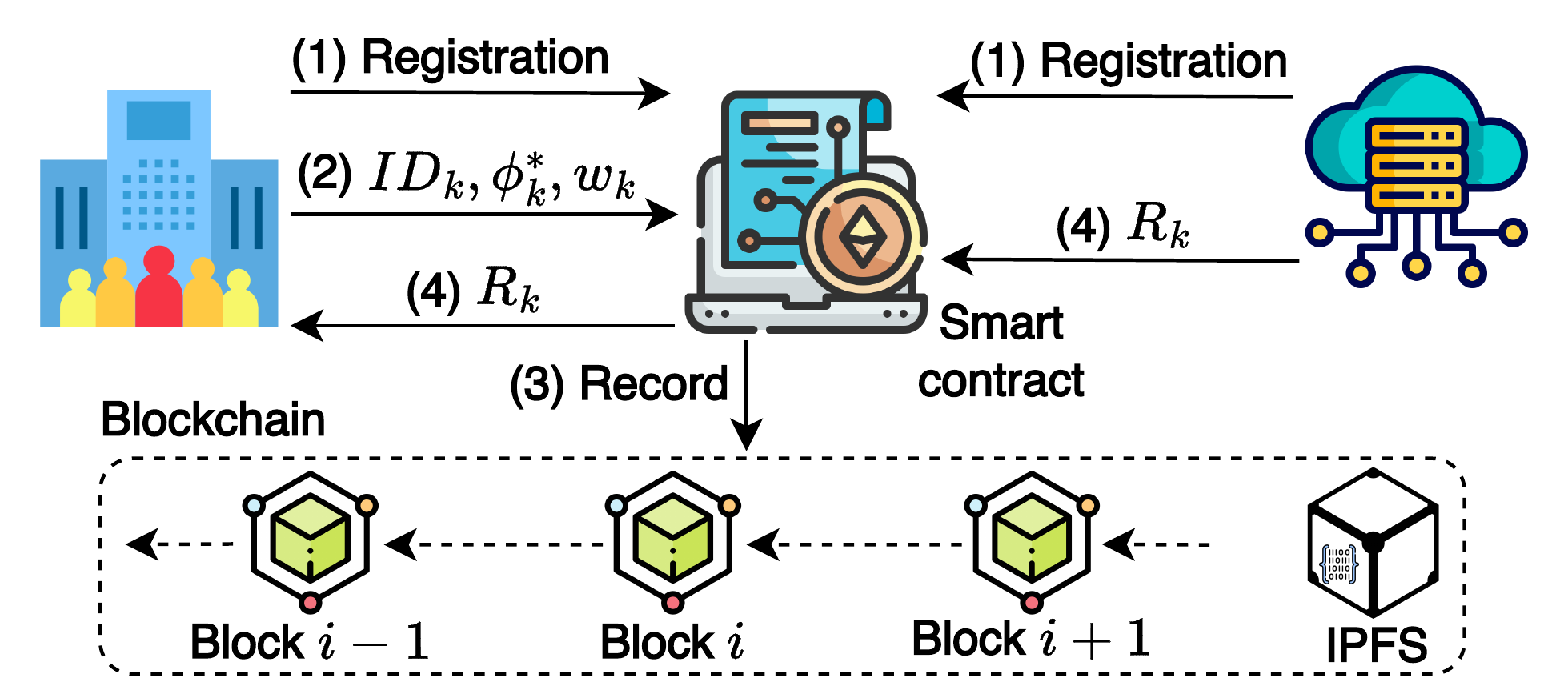}
	\caption{Procedure of {\fontfamily{qcr}\selectfont R3T}-based smart contracts.}
    \label{fig:smart_contract}
\end{figure}

\paragraph{{\proposed} smart contracts} Fig.~\ref{fig:smart_contract} illustrates the functionality of blockchain and smart contracts, including four main operations. 1) First, clients and the cloud initiate their participation. During this process, {\fontfamily{qcr}\selectfont R3T} records information, such as IDs and wallet addresses, for both clients and the cloud server to ensure that only registered participants can interact within the system. Additionally, the cloud publishes a contract set $\phi^*$, where $\phi^*$ is a set of optimal contract items determined by the cloud in Section~\ref{sec:optimal_incentive_design}, directly onto the blockchain for transparency and accessibility to clients. 2) Upon completing the training process, clients upload their strategy profiles or selected contract items $\phi_k^*$ with their IDs, the location of the model weights, and hashed model weights to the blockchain by invoking \textit{contributionSubmit()} function in the \textit{Contribution} smart contract. 3) To enhance scalability, minimize on-chain storage overhead, and provide robustness to single-point-of-failure attacks for the cloud server of {\fontfamily{qcr}\selectfont R3T}, \textit{dataStore()} in the \textit{Contribution} smart contract is employed. Global model weights are uploaded to private IPFS built by authorized edge servers while the resulting hash and corresponding training round number are recorded on-chain, ensuring data integrity and efficient and decentralized model management. 4) Finally, contract-based rewards $R_k$ (e.g., tokens or cryptocurrencies such as Ether in the Ethereum blockchain platform) are calculated and distributed through \textit{reward()} function in the \textit{Reward} smart contract.
 
Next, we analyze \text{\fontfamily{qcr}\selectfont R3T} under two information conditions over \text{\fontfamily{qcr}\selectfont CLPs} and \text{\fontfamily{qcr}\selectfont non-CLPs}. \textit{1) Complete Information Case:} Clients' types are public and known to the cloud and other clients. However, this case is impractical due to privacy leakage risks from exposing individual properties. \textit{2) Incomplete Information Case:} Clients' types are private and measured locally, but the cloud may statistically know the distribution of clients' types via market research or surveys \cite{10.1145/3209582.3209588}.

%% file: Sections_revision_wt_trackchanges/3_Incentive_mechanism_Design.tex
\section{{\proposed} - The Proposed Incentive Mechanism for critical learning periods }

\label{sec:optimal_incentive_design}
\subsection{Problem Formulation}\label{subsec:formulation}

A feasible contract set ensures clients, who choose their true preference types, receive equitable rewards that are commensurate with their costs and effort. Besides, in each training round, the cloud server can only afford a given budget $P$ for rewarding. To achieve this, it must satisfy the following individual rationality (IR), incentive compatibility (IC), and budget feasibility (BF) constraints:
\begin{constraint}[Individual Rationality]
\label{cons:ir}
Each type-$k$ client chooses the contract item if and only if its utility is non-negative,
\begin{align}
\label{eq:ir}
\mathcal{U}_n (e_k, t_k, r_k) \ge 0, \forall{k \in \mathcal{K}}, \forall{t_k \in \mathcal{T}}.
\end{align}
\end{constraint}

\begin{constraint}[Incentive Compatibility]
\label{cons:ic}
Each type-$k$ client maximizes its utility by choosing the contract item $\{e_k, t_k, r_k\}$, which is specifically designed for its type and is presented by
\begin{align} 
\label{eq:ic}
\mathcal{U}_n (e_k, t_k, r_k) \ge \mathcal{U}_n (e_{k'}, t_{k'}, r_{k'}), \forall{k, k' \in \mathcal{K}},  k \neq k'.
\end{align}
\end{constraint}

\begin{assumption}[Rationality]
\label{assump:rationality}
Each client is willing to join and contribute its effort to the training process, which guarantees IR and IC constraints and leads to the best utility for itself \cite{bolton2004contract}. For example, realizing the importance of {\fontfamily{qcr}\selectfont CLPs} with a correspondingly huge amount of rewards announced by the cloud server\footnote{That is considered as common knowledge \cite{bolton2004contract}.}, each client tends to join and allocate efforts in this period.
\end{assumption}

\begin{constraint}[Budget Feasibility]
\label{cons:bf}
The total rewards for all participating type-$k$ clients do not exceed cloud budget $P$, i.e., 
\begin{align} 
\label{eq:bf}
\sum\nolimits_{k=1}^{K} R_k \le P.
\end{align}
\end{constraint}

To maximize the utility, the cloud server offers the optimal contract set $\boldsymbol{\phi^*} = \{\boldsymbol{e^*}, \boldsymbol{t^*}, \boldsymbol{r^*}\}$ to clients. The optimal contract set is the solution to the following problem\footnote{\label{foot:blockchain_cost}The total blockchain mining rewards are constant because the computing resources required for this task are the same across all edge servers in each global training round \cite{wang2023incentive}.  Thus, this factor does not affect the optimization and analysis in this paper, so we omit it here.}:
\begin{subequations}
\begin{align}
\label{eq:optimal_contract_design_1}
\max_{(\boldsymbol{e}, \boldsymbol{t}, \boldsymbol{r})} \quad & \mathcal{U}_{c}{(\boldsymbol{e}, \boldsymbol{t}, \boldsymbol{r})}, \\
\textrm{s.t.} \quad & \eqref{eq:ir}, \eqref{eq:ic}, \eqref{eq:bf}, \forall t_k \in \mathcal{T}, \forall k \in \mathcal{K} \label{opt:ic}.
\end{align}
\end{subequations}

Given the cloud server's contract, we model the game among clients as follows:
\begin{itemize}[leftmargin=*]
\item \textit{Players:} $|\mathcal{S}^{(t)}|$ clients are in the set $\mathcal{N}$ per training round $t$.
\item \textit{Strategies:} each client $n \in \mathcal{S}^{(t)}$ decides which contract item $\phi_k = \{e_k, t_k, r_k \} \in \phi$ to choose based on its capability for model training task published by the cloud server.
\item \textit{Objectives:} each type-$k$ client $n$ aims to maximize its utility $\mathcal{U}_{n}(e_k, t_k, r_k)$ expressed in \eqref{eq:client_utility}.
\end{itemize}

Once $t_k$ and $r_k$ are given, we determine that the utility of a client in \eqref{eq:client_utility} is a strictly concave function with respect to its effort $e_k$. Hence, the optimal choice of effort $e_k$ for a type-$k$ client can be obtained by setting the first-order derivative of the client's utility function with respect to $e_k$ to zero. Specifically, 
\begin{align}
\label{eq:optimal_effort_2}
e^*_k = \frac{\theta_k h(t_k)}{\delta}.
\end{align}
From \eqref{eq:time_aware_bonus} and \eqref{eq:optimal_effort_2}, we can see that the optimal effort of each type-$k$ client is greater than zero and independent of $r_k$, but is increasing with its time and system capability $\theta_k$ and is decreasing with joining time $t_k$. In other words, if a client is willing to join early and has strong system capabilities, it has a higher chance of receiving more bonuses and rewards.

Substituting 
$e^*_k$ into \eqref{eq:define_reward}, \eqref{eq:client_utility}, \eqref{eq:cloud_utility_1}, the utilities of a type-$k$ client and the cloud server are rewritten respectively as
\begin{align}
R_k = r_k + \frac{\theta_k^2 h^2(t_k)}{\delta}, \label{eq:defined_reward_2} \end{align}
\begin{align}
\mathcal{U}_n (t_k, r_k) = \frac{{\theta}_k^2 h^2(t_k)}{2\delta} -  (\beta - 1)r_k, \label{eq:substituted_client_utility}
\end{align}
\begin{align}
\mathcal{U}_{c}(\boldsymbol{t}, \boldsymbol{r}) =  \sum_{k=1}^{K}{\left[\lambda g\left(\frac{\theta_k h^2(t_k)}{\delta}\right) - \left({r_k + \frac{\theta_k^2 h^2(t_k)}{\delta}}\right) \right]}. \label{eq:cloud_utility_2}
\end{align}
The optimal contract design can be rewritten by
\begin{subequations}
\begin{align}
\label{eq:optimal_contract_design_2}
\max_{(\boldsymbol{t}, \boldsymbol{r})} \quad & \mathcal{U}_{c}{(\boldsymbol{t}, \boldsymbol{r})}, \\
\textrm{s.t.} \quad &  \eqref{eq:ir}, \eqref{eq:ic}, \eqref{eq:bf}, { \forall t_k \in \mathcal{T}}, { \forall k \in \mathcal{K}}.
\end{align}
\end{subequations}%
\subsection{{\proposed} Incentivization under Complete Information Case}\label{subsec:complete_infor}
We study scenarios where the cloud server knows clients' time and system capabilities precisely before any interaction. This serves as an ideal benchmark, but it may be infeasible in real-world applications due to privacy regulations.

With precise knowledge of each client's type (i.e., client's available time and system capabilities), the cloud server can tailor personalized contracts to clients. Each type-$k$ client receives a corresponding contract item $\phi_k = \{e_k, t_k, r_k\}$, ensuring that IC and BF constraints in \eqref{eq:ic} and \eqref{eq:bf} are satisfied. The cloud server only considers the IR constraints in \eqref{eq:ir} for feasible contract design. The optimal contract design under the complete information case can be formulated as
\begin{subequations}
\begin{align}
\label{eq:optimal_contract_complete_info_1}
\max_{(\boldsymbol{t}, \boldsymbol{r})} \quad & \sum_{k=1}^{K}{\left[\lambda g\left(\frac{\theta_k h^2(t_k)}{\delta}\right) - {r_k - \frac{\theta_k^2 h^2(t_k)}{\delta}} \right]}, \\
\noindent \textrm{s.t.} \quad & \frac{{\theta}_k^2 h^2(t_k)}{2\delta} -  (\beta - 1)r_k \ge 0,  \\
& \sum_{k=1}^{K} \left[r_k + \frac{\theta_k^2 h^2(t_k)}{\delta}\right] \le P, { \forall t_k \in \mathcal{T}}, { \forall k \in \mathcal{K}} \label{subeq:bf}.
\end{align}
\end{subequations}

\begin{theorem}
\label{theorem_1}
All optimal contract items satisfy the condition $\frac{{\theta}_k^2 h^2(t_k)}{2\delta} -  (\beta - 1)r_k = 0,  \forall{k \in \mathcal{K}}$ under the complete information case, that is the client's utility is zero.
\end{theorem}

\begin{proof}
  A rational cloud server aims to exploit contributions (i.e., effort and early-available time) from clients, leading to zero utility for clients. To be more specific, the cloud server always chooses a larger effort $e_k$ (i.e., a larger $h(t_k)$ where $e_k = \frac{\theta_k h(t_k)}{\delta}$ shown in \eqref{eq:optimal_effort_2}), and a larger salary $r_k$. By increasing these variables, the cloud incentivizes clients to join promptly and invest time and effort in model training. In other words, the cloud exhausts its budget to quickly achieve model convergence as long as the IR constraints are not violated, specifically, $\frac{{\theta}_k^2 h^2(t_k)}{2\delta} -  (\beta - 1)r_k = 0$.
\end{proof}


It is important to note that a larger $h(t_k)$ means a smaller $t_k$ that is at the early training stage. The optimal contract design in \eqref{eq:optimal_contract_complete_info_1} is rewritten by:
\begin{subequations}
\begin{align}
\label{eq:optimal_contract_complete_info_2}
\max_{(\boldsymbol{t}, \boldsymbol{r})} \quad & \sum_{k=1}^{K}{\left[\lambda g\left(\frac{\theta_k h^2(t_k)}{\delta}\right) - r_k - \frac{\theta_k^2 h^2(t_k)}{\delta}\right]}, \\
\noindent \textrm{s.t.} \quad & \frac{{\theta}_k^2 h^2(t_k)}{2\delta} -  (\beta - 1)r_k = 0,  \label{eq:cic_reduced_ir}\\
& \sum_{k=1}^{K} \left[r_k + \frac{{\theta}_k^2 h^2(t_k)}{2\delta}\right] \le P, { \forall t_k \in \mathcal{T}}, { \forall k \in \mathcal{K}}.
\end{align}
\end{subequations}
In \eqref{eq:cic_reduced_ir}, by replacing $r_k = \frac{{\theta}_k^2 h^2(t_k)}{2\delta (\beta - 1)}$, we have
\begin{subequations}
\label{eq:optimal_contract_complete_info_3_main}
\begin{align}
\label{eq:optimal_contract_complete_info_3}
    \max_{\boldsymbol{t}} \quad & \sum_{k=1}^{K}{\left[\lambda g\left(\frac{\theta_k h^2(t_k)}{\delta}\right) - \frac{{\theta}_k^2 h^2(t_k)}{2\delta(\beta - 1)} - \frac{\theta_k^2 h^2(t_k)}{\delta}\right]}, \\
\noindent \textrm{s.t.} \quad  & \sum_{k=1}^{K} \left[\frac{{\theta}_k^2 h^2(t_k)}{2 \delta(\beta - 1)} + \frac{\theta_k^2 h^2(t_k)}{\delta}\right] \le P, { \forall t_k \in \mathcal{T}}, { \forall k \in \mathcal{K}}.
\end{align}
\end{subequations}
\noindent Therefore, by utilizing exhaustive search, we can calculate optimal values of $h(t_k)$ and $t_k$, then $r_k$, $R_k$, and $e_k$. \label{para:complexity1}{{Note that the exhaustive search is performed only over the discrete joining-time vector $\boldsymbol{t}$, where each $t_k \in \{1,\ldots,T\}$. Once a candidate $\boldsymbol{t}$ is fixed, the corresponding $h(t_k)$ is obtained from Eq.~\eqref{eq:time_aware_bonus}, and then $e_k$, $r_k$, and $R_k$ are computed in closed form using Eqs.~\eqref{eq:optimal_effort_2}-\eqref{eq:defined_reward_2}. Therefore, the search is not over the full $(t_k,r_k,e_k)$ space, but only over the $K$-dimensional discrete time-assignment space. The raw worst-case search space contains $T^K$ candidates, and each candidate can be evaluated in $\mathcal{O}(K)$ time under the complete-information case. Hence, the worst-case complexity is $\mathcal{O}(KT^K)$.}}

\subsection{{\proposed} Incentivization under Incomplete Information Case} \label{subsec:iic}
We now introduce the contract game in an incomplete information setting. Clients choose contract items (i.e., effort, joining time, and corresponding reward) offered by the cloud. While the cloud aims to optimize its utility (e.g., global model performance), clients seek to maximize their incentives. 





\begin{lemma}[Monotonicity in {\fontfamily{qcr}\selectfont CLPs}]
\label{lem:monotonicity}
The higher-type client should receive a higher reward $R$ when they join the training process early and contribute more effort. Otherwise, all clients would opt for higher-type contract items but join later and contribute less effort. Thus, under the IC constraints, if $\theta_{k} > \theta_{k'}, \forall k,k' \in \mathcal{K}$, then the reward (including salary and bonus), effort, and joining time satisfy the following inequalities:
\begin{align} 
\label{eq:monotonicity_1}
  R_k > R_{k'}, 
 r_k \ge r_{k'},
 B_k > B_{k'},
 e_k > e_{k'},
 t_k \leq t_{k'}.
\end{align}
\end{lemma}

\begin{proof}
According to the IC constraints in \eqref{eq:ic}, we have 
\begin{align}
& \frac{{\theta^2_k} h^2(t_k)}{2\delta} -  (\beta - 1)r_k \ge \frac{{\theta^2_k} h^2(t_{k'})}{2\delta}-  (\beta - 1)r_{k'}, \label{eq:IC1} \\
& \frac{{\theta^2_{k'}} h^2(t_{k'})}{2\delta} -  (\beta - 1)r_{k'} \ge \frac{{\theta^2_{k'}} h^2(t_k)}{2\delta} -  (\beta - 1)r_{k}, \label{eq:IC2}
\end{align}
 with $k, k' \in \mathcal{K}, k \neq k'$. Summing up \eqref{eq:IC1} and \eqref{eq:IC2}, we obtain
\begin{align}
\label{eq:proof_monotonicity_2}
\frac{h^2(t_k)}{2\delta}({\theta^2_k} - {\theta^2}_{k'}) \ge \frac{h^2(t_{k'})}{2\delta}({\theta^2_k} - {\theta^2}_{k'}).
\end{align}

\noindent As $\theta_k > \theta_{k'} > 0$, we have $({\theta_k}^2 - {\theta}_{k'}^2) > 0$. Dividing both sides by $({\theta_k}^2 - {\theta}_{k'}^2)$, we have $h^2(t_k) \geq h^2(t_{k'})$ and as $h(t_k) \ge 1, \forall{k \in \mathcal{K}}$, we conclude $h(t_k) \geq h(t_{k'})$. In addition, because $h(t_k)$ is a strictly decreasing function of $t_k$, we can obtain $t_{k} \leq t_{k'}$. Also, we prove if $h(t_k) \geq h(t_{k'})$, then $r_k \geq r_{k'}$ under the IC constraints. Based on \eqref{eq:IC2}, we have 
 \begin{align} 
\label{eq:proof_monotonicity_3}
& (\beta - 1)r_{k} -  (\beta - 1)r_{k'} \ge \frac{{\theta^2_{k'}} h^2(t_k)}{2\delta} -  \frac{{\theta^2_{k'}} h^2(t_{k'})}{2\delta}, \nonumber \\
& (\beta - 1)(r_{k} - r_{k'}) \ge \frac{{\theta}_{k'}^2}{2\delta} [h^2(t_{k}) - h^2(t_{k'})] \geq 0.
\end{align}
\noindent Since $\beta > 1$ and $h(t_k) \geq h(t_{k'})$, thus $r_k \geq r_{k'}$. Next, with the conditions of $\theta_{k} > \theta_{k'}, h(t_{k}) \geq h(t_{k'})$, and \eqref{eq:optimal_effort_2}, we have $e_{k} > e_{k'}$. As a result, we conclude $B_k > B_{k'}$, thereby $R_k > R_{k'}$ based on \eqref{eq:define_reward}. This completes the proof.
\end{proof}
\begin{lemma}[Client utility condition]
\label{lem:client_utility}
For any feasible contract item $(e_k, t_k, r_k)$, the type-$k$ client's utility satisfies
\begin{align}
 \mathcal{U}_n(e_1, t_1, r_1) < ... < \mathcal{U}_n(e_k, t_k, r_k) &< ... \notag \\ & < \mathcal{U}_n(e_K, t_K, r_K).
\end{align}
\end{lemma} \vspace{-2.5mm}
\begin{proof}
According to Lemma~\ref{lem:monotonicity}, high-type clients who receive more rewards must contribute greater efforts and join as early as possible (i.e., $ R_k > R_{k'}, e_k > e_{k'} \text{ and } t_k \leq t_{k'}$ are imposed together). Upon the conditions that ${\theta}_k > {\theta}_{k'}$ and \eqref{eq:optimal_effort_2}, we have
\begin{align}
 & \mathcal{U}_n(e_k, t_k, r_k) = \frac{{\theta^2_k} h^2(t_k)}{2\delta} -  (\beta - 1)r_k  \ge \frac{{\theta^2_k} h^2(t_{k'})}{2\delta}-  (\beta - 1)r_{k'}  \nonumber \\ 
& > \frac{{\theta^2_{k'}} h^2(t_{k'})}{2\delta} - (\beta - 1)r_{k'} = \mathcal{U}_n(e_{k'}, t_{k'}, r_{k'}) .
\end{align}
This completes the proof.
\end{proof}

\begin{lemma}[Simplify the IR constraints]
\label{lem:reduce_ir_constraints}
In the optimal contract, given that IC constraints in \eqref{eq:ic} are satisfied, the IR constraint for type-$1$ client is binding, i.e., $\frac{{\theta^2_{1}} h^2(t_{1})}{2\delta} -  (\beta - 1)r_1 = 0$.
\end{lemma}
\begin{proof}
According to Lemma~\ref{lem:client_utility}, we have 
\begin{align}
\label{eq:satisfy_ir_constraints}
\frac{{\theta^2_k} h^2(t_k)}{2\delta} -  (\beta - 1)r_k  & \ge \frac{{\theta^2_{1}} h^2(t_{1})}{2\delta}  -  (\beta - 1)r_{1} \ge 0.
\end{align}
\noindent As such, if the IR constraint of type-$1$ client is binding, that of all other types holds. This completes the proof.
\end{proof}
The IR constraints in \eqref{eq:ir} can be replaced by Lemma \ref{lem:reduce_ir_constraints}, which indicates that this type-$1$ client gains or loses nothing in participating in training, while other clients' utilities are higher than that of the binding one.
Next, we refer to \cite{zhang2015contract} to introduce the following definitions for simplifying IC constraints:
\begin{definition}[Conditions for IC constraints]
1) {Downward incentive constraints (DICs)} are IC constraints between client types $k$ and ${k'}$, \text{where } $1 \le {k'} \le k - 1$, 2) {local downward incentive constraints (LDIC)} is an IC constraint between client types $k$ and ${k'}, \text{where } {k'} = k - 1$,
3) {upward incentive constraints (UICs)} are IC constraints between client types $k$ and ${k'}, \text{where } {k}+1 \le k' \le {K}$, and    
4) {local upward incentive constraints (LUIC)} is an IC constraint between client types $k$ and ${k'}, \text{where }  k'  = {k} + 1$.
\end{definition}
 

\begin{lemma}[Simplify the IC constraints]
\label{def:further_reduce_ic_constraints}
In the optimal contract, we can replace the IC constraints in \eqref{eq:ic} by
\begin{align}
\label{eq:reduced_ldic}
   \frac{{\theta^2_k} h^2(t_k)}{2\delta} -  (\beta - 1)r_{k} = \frac{{\theta^2_k} h^2(t_{k-1})}{2\delta} -  (\beta - 1)r_{k-1}.
\end{align}
\end{lemma}
\begin{proof} 
We consider three client types where ${\theta}_{k-1} < {\theta}_{k} < {\theta}_{k+1}, \forall k \in \mathcal{K}$. Under IC constraints, we have two LDICs inequalities as follows
\begin{align} 
& \frac{{\theta^2_{k+1}} h^2(t_{k+1})}{2\delta} -  (\beta - 1)r_{k+1} \ge \frac{{\theta^2_{k+1}} h^2(t_k)}{2\delta} -  (\beta - 1)r_k,  \label{eq:ldic_proof1} \\
& \frac{{\theta^2_k} h^2(t_k)}{2\delta} -  (\beta - 1)r_k \ge \frac{{\theta^2_k} h^2(t_{k-1})}{2\delta} -  (\beta - 1)r_{k-1}. \label{eq:ldic_proof2} 
\end{align}
By virtue of monotonicity in Lemma~\ref{lem:monotonicity}, the condition of $\theta$, and \eqref{eq:ldic_proof2}, we have
\begin{align}
 \frac{{\theta}_{k+1}^2}{2\delta}\left[h^2(t_{k}) - h^2(t_{k-1})\right] & \ge  \frac{{\theta}_{k}^2}{2\delta} \left[h^2(t_{k}) - h^2(t_{k-1})\right] \notag\\ & \ge (\beta - 1)(r_{k} - r_{k-1}).
\end{align}
\noindent As a result, we convert \eqref{eq:ldic_proof1} as
\begin{align}
  \frac{{\theta^2_{k+1}} h^2(t_{k+1})}{2\delta} -  (\beta - 1)r_{k+1} & \ge \frac{{\theta^2_{k+1}} h^2(t_k)}{2\delta} -  (\beta - 1)r_k \nonumber \\ & \ge \frac{{\theta}_{k+1}^2 h^2(t_{k-1})}{2\delta} -  (\beta - 1)r_{k-1}.
\end{align}
Therefore, we conclude that LDIC has transitivity, i.e., if there is an LDIC between type-$(k-1)$ and type-$k$ clients, then we can further extend incentive constraints from type-$(k-1)$ client to type-$1$ client, that is, all DICs hold, i.e., 
\begin{align}
  \frac{{\theta^2_{k+1}} h^2(t_{k+1})}{2\delta} -  (\beta - 1)r_{k+1} & \ge \frac{{\theta}_{k+1}^2 h^2(t_{k})}{2\delta} -  (\beta - 1)r_{k}\nonumber \\
  & \ge \frac{{\theta}_{k+1}^2 h^2(t_{k-1})}{2\delta} -  (\beta - 1)r_{k-1} \nonumber \\  
  & \ge ... \nonumber \\ 
  & \ge \frac{{\theta}_{k+1}^2 h^2(t_{1})}{2\delta} -  (\beta - 1)r_{1}.
\end{align}
Thus, we have proved that with the LDIC, all DICs hold. Similarly, with the LUIC, all UICs are satisfied. We can write a generalization of DICs and UICs, respectively, as follows
\begin{align}
  \frac{{\theta^2_k} h^2(t_k)}{2\delta} -  (\beta - 1)r_{k}  & \ge \frac{{\theta}_{k}^2 h^2(t_{k{'}})}{2\delta}  -  (\beta - 1)r_{k{'}}, \notag \\ & 1 \le k{'} < k \le K, \notag \\
    \frac{{\theta^2_k} h^2(t_k)}{2\delta} -  (\beta - 1)r_{k}  & \ge \frac{{\theta}_{k}^2 h^2(t_{k{'}})}{2\delta}  -  (\beta - 1)r_{k{'}}, \notag \\ &1 \le k < k{'} \le K.
\end{align}
Therefore, the IC constraints are reduced to 
\begin{align}
  \frac{{\theta^2_k} h^2(t_k)}{2\delta} -  (\beta - 1)r_{k}  & \ge \frac{{\theta}_{k}^2 h^2(t_{k-1})}{2\delta}  -  (\beta - 1)r_{k-1},
\end{align}
which replace IC constraints in \eqref{eq:ic}. This completes the proof.
\end{proof}
\input{Algorithms/r3t_wt_trackchanges}


\begin{theorem}
\label{theorem_2}
For any type-$k$ client, the contract salary satisfies
$$r_k = \sum_{i=1}^{k} \frac{\theta^2_{i}[h^2(t_{i}) - h^2(t_{i-1})]}{2\delta(\beta -1)}.$$
\end{theorem}
\begin{proof} 
\label{proof_theorem_1}
Based on Lemma~\ref{lem:reduce_ir_constraints}, the salary of type-$1$ client can be calculated as follows
\begin{align}
\label{solve_ir}
    r_1 = \frac{\theta^2_1 h^2(t_1)}{2\delta(\beta -1)}.
\end{align}
Then, according to Lemma~\ref{def:further_reduce_ic_constraints} and \eqref{solve_ir}, we can further derive 
\begin{align}
    r_2 = r_1 + \frac{\theta^2_2 [h^2(t_2)-h^2(t_1)]}{2\delta(\beta -1)}.
\end{align} 
By iterating these steps, we can obtain
\begin{align}
\label{eq:optimal_r}
    r_k &= r_{k-1} + \frac{\theta^2_{k} [h^2(t_{k})-h^2(t_{k-1})]}{2\delta(\beta -1)} \nonumber \\
    &= \frac{\theta^2_{k} [h^2(t_{k})-h^2(t_{k-1})]}{2\delta(\beta -1)} + ... + \frac{\theta^2_{1}h^2(t_{1})}{2\delta(\beta -1)} \nonumber \\
    &= \sum_{i=1}^{k} \frac{\theta^2_{i}[h^2(t_{i}) - h^2(t_{i-1})]}{2\delta(\beta -1)}.
\end{align}
\label{explain_t0} At $i=1$, we have $h(t_0) = 0$ because $t_0 = 0$ {denotes the pre-training initialization phase}. This completes the proof.
\end{proof}
By using the simplified constraints above and substituting $e_k$ and $r_k$ with $h(t_k)$, the final form of the optimization problem is as follows:
%
\begin{subequations}
\label{eq:optimal_contract_design_4_main}
\begin{align}
\label{eq:optimal_contract_design_4}
\max_{\boldsymbol{t}} \quad & \sum_{k=1}^{K}\Big[\lambda g(\frac{\theta_k h^2(t_k)}{\delta}) - \sum_{i=1}^{k}\frac{\theta^2_{i}[h^2(t_{i}) - h^2(t_{i-1})]}{2\delta(\beta -1)} \nonumber \\
& {\quad \quad  - \text{ } \frac{\theta^2_k h^2(t_k)}{\delta}\Big]}, \\
\textrm{s.t.} \quad & \sum_{k=1}^{K} \left[\sum_{i=1}^{k}\frac{\theta^2_{i}[h^2(t_{i}) - h^2(t_{i-1})]}{2\delta(\beta -1)}+ \frac{\theta_k^2 h^2(t_k)}{\delta}\right] \le P, \\
& \eqref{eq:monotonicity_1}, { \forall t_k \in \mathcal{T}}, { \forall k \in \mathcal{K}}.
\end{align}
\end{subequations}

\label{para:complexity2}{{Using an exhaustive search algorithm, we determine the optimal joining-time vector $\boldsymbol{t}$, from which $h(t_k)$, $r_k$, $R_k$, and $e_k$ are then obtained analytically from Eq.~\eqref{eq:time_aware_bonus}, Theorem~\ref{theorem_2}, Eq.~\eqref{eq:define_reward}, and Eq.~\eqref{eq:optimal_effort_2}, respectively. Thus, the exhaustive search is conducted only over the discrete time-assignment space. In the raw worst case, this search space contains $T^K$ candidate vectors. For each candidate, the salary computation and feasibility verification must be performed jointly across contract items, since the recursive expression of $r_k$ depends on lower-type items and the IC, monotonicity, and budget-feasibility conditions couple multiple contract items. Under a direct implementation, this yields a per-candidate cost of $\mathcal{O}(K^2)$ and hence an overall worst-case complexity of $\mathcal{O}(K^2T^K)$. We further note that a moderate value of $K$ is usually sufficient to capture the heterogeneity in client training capabilities while preserving computational tractability. In addition, an excessively large $K$ is often undesirable because it increases optimization complexity, aggravates reward-budget pressure on the cloud under the IR, IC, and BF constraints, and may create risks of private information leakage at an unnecessarily fine granularity, which will be discussed more in Section~\ref{subsec:feasibility}.}}

Based on the analyses and proofs above, we now analyze monotonicity conditions in the {\fontfamily{qcr}\selectfont non-CLPs}. 
\begin{lemma}[Monotonicity in {\fontfamily{qcr}\selectfont non-CLPs}]
\label{lem:monotonicity-nonclp}
In the {\fontfamily{qcr}\selectfont non-CLPs}, {\proposed} still ensures incentive fairness by enabling a higher-type client to receive a higher reward. Thus, to satisfy the IC constraints in \eqref{eq:ic}, if $\theta_{k} > \theta_{k'}, \forall k,k' \in \mathcal{K}$, then the reward, salary, and effort must adhere to the following inequalities:
\begin{align} 
\label{eq:monotonicity_nonclp}
R_k > R_{k'},
 r_k = r_{k'} ,
  e_{k} > e_{k'}.
\end{align}
\end{lemma}

\begin{proof}
    According to \eqref{eq:define_reward}, $h(t_k) = 1$ and the impact of efforts in \text{\fontfamily{qcr}\selectfont non-CLPs} is not of importance as {\fontfamily{qcr}\selectfont CLPs}, thereby $ h(t_{k'}) = h(t_k) = 1, \forall{t_k, t_{k'} \notin \text{\fontfamily{qcr}\selectfont CLPs}}$ . This leads to $e_{k} > e_{k'}$ based on \eqref{eq:optimal_effort_2}. Moreover, via \eqref{eq:optimal_r} and $ h(t_{k'}) = h(t_k)$, we can conclude that $r_k = r_{k'} = r_1 = \frac{\theta^2_1 h^2(t_1)}{2\delta(\beta - 1)}$. Next, with the proved conditions (i.e., $h(t_k) = h(t_{k'}) \text{ and } r_k = r_{k'}$), and the expression \eqref{eq:define_reward}, we have $R_k > R_{k'}$. This completes the proof.
\end{proof}

\begin{remark}
\label{remark:nonclp}
    In \text{\fontfamily{qcr}\selectfont non-CLPs}, the impact of the client's contribution is less significant compared to \text{\fontfamily{qcr}\selectfont CLPs}. As Lemma~\ref{lem:monotonicity-nonclp} outlines, although all clients receive a fixed salary regardless of their types and do not receive weighted bonuses (i.e., $h(t_k)=1$) for specific joining times, {\proposed} ensures the fairness of their total rewards relative to their corresponding efforts.
\end{remark}

\begin{remark}
\label{remark:nonclp2}
    {{\proposed} mitigates information asymmetry between clients and the cloud via a self-revelation mechanism. Specifically, by offering a menu of contract items, {\proposed} incentivizes each client to select the contract item aligned with its effective type, while the IC constraint guarantees that such truthful type-matched selection is utility-maximizing. {Accordingly, client-type knowledge in {\proposed} is interpreted as a privacy-preserving estimate, rather than exact per-client prior knowledge. Furthermore, {\proposed} ensures fair reward allocation in both \text{\fontfamily{qcr}\selectfont CLPs} and \text{\fontfamily{qcr}\selectfont non-CLPs}, as established in Lemma~\ref{lem:monotonicity} and Lemma~\ref{lem:monotonicity-nonclp}.}} 
\end{remark}


The game and training processes of {\proposed} over $T$ training rounds are depicted in Algorithm~\ref{alg:r3t}. We have a tuple of $(e_k^*, t_k^*, r_k^*)$ being the optimal strategy of the type-$k$ client (see line~\ref{alg:line:optimalvalue}). {\clps} can be identified by using the $\mathrm{FGN}$ proxy (see line \ref{alg:identified_clps})~\cite{yan2023criticalfl}. {Algorithm~\ref{alg:r3t} presents the full {\proposed} policy, where client participation can be temporally incentivized and redistributed toward {\clps}, since insufficient high-quality contributions during these periods can cause irreversible impairment to the final model performance.} Specifically, {\proposed} increases the selected-client size during {\clps} to attract more higher-type clients through the proposed contract set (see line~\ref{alg:line:clp}), while decreasing the selected-client size during {\nonclps} to reduce communication overhead and budget consumption (see line~\ref{alg:line:nonclp}). Importantly,\textit{ this policy does not increase the total number of selected clients over the entire training process}, since the expanded participation during {\clps} is offset by reduced participation during {\nonclps}. {In the proof-of-concept experiments in Section~\ref{subsec:poc}, we further disable the adaptive client-count policy and keep $|\mathcal{S}^{(t)}|$ fixed for all $t$, so as to isolate the effect of the proposed {\clp}-aware incentive mechanism.}

%% file: Algorithms/r3t_wt_trackchanges.tex
\setlength{\textfloatsep}{0pt}
\begin{algorithm}[t!]
\footnotesize	
\caption{Game and Training Process of  {\fontfamily{qcr}\selectfont R3T}}
\label{alg:r3t}
  \KwInput{A set of clients $\mathcal{N}$, number of training rounds $T$, learning rate $\eta$, local loss function $F$, initial global parameter {$\mathbf{w}^{(0)}$}, bonus unit coefficient $\vartheta$.}
  \KwOutput{Final global model parameter {$\mathbf{w}^{(T)}$.}}
  \KwData{Training, validation, and test set $\mathcal{D}$.}
    
\For{$t$ = 1, 2,..., T}  {
    $\textcolor{gray}{\rhd{\textit{ Cloud performs:}}}$


    Obtain an optimal contract set $\boldsymbol{\phi^*}$ via \eqref{eq:optimal_contract_complete_info_3_main} or \eqref{eq:optimal_contract_design_4_main}. \label{alg:line:optimalvalue}
    
    Send the set $\boldsymbol{\phi^*} = \{\boldsymbol{e^*}, \boldsymbol{t^*}, \boldsymbol{r^*}\} = \{\phi_k^*\}_{k \in \mathcal{K}} $ to $\mathcal{N}$. 
    
    Select a client subset $\mathcal{S}^{(t)}$ from $\mathcal{N}$.
    
    $\textcolor{gray}{\rhd{\textit{ Clients perform:}}}$
    
    \For{each client $n \in \mathcal{S}^{(t)}$ in parallel} {
    Sign contract items $\{\phi_k^*\}_{k \in \mathcal{K}}$.
    
        \If{$t^*_k  = t$} { 
        Download  {$\mathbf{w}^{(t-1)}$}.
        
        Perform local training with $e^*_k$ using $\mathcal{D}_k$.
    
    
        ${\mathbf{w}^{(t)}_{k}} \gets {\mathbf{w}^{(t-1)}_{k}} - \eta \nabla {F}_k ({\mathbf{w}^{(t-1)}_{k}}, \mathcal{D}_k)$
        
        } 
    }
    $\textcolor{gray}{\rhd{\textit{ Blockchain performs:}}}$
    
    Cross-validate and record {{on-chain}} \{$ID_k$, $hash({\mathbf{w}^{(t)}_{k})}$, $\phi_k^*$\}.
    
    $\textcolor{gray}{\rhd{\textit{ Cloud performs: with FGN proxy in Eqs.~\eqref{eq:fgn}--\eqref{eq:fgn_condition}}}}$ 
    
    \If{$t \in {\clps}$}{ \label{alg:identified_clps}
    
        $|\mathcal{S}^{(t+1)}| \leftarrow \text{min}\{2|\mathcal{S}^{(t)}|, N\}$ \label{alg:line:clp}

        $h(t^*_k) = h(t) = 1 + \frac{\vartheta}{\ln(2t)}$
    }
    \Else{
        $|\mathcal{S}^{(t+1)}| \gets \text{max} \{\frac{1}{2}|\mathcal{S}^{(t)}|, \frac{1}{2} N \}$ \label{alg:line:nonclp}
        
    	$h(t^*_k) = 1$
    }
        
    
    ${\mathbf{w}^{(t)}} \gets \sum_{n \in \mathcal{S}^{(t)}}{\frac{e^*_{n,k}}{\sum_{j \in \mathcal{S}^{(t)}} e^*_{j, k}}} {\mathbf{w}^{(t)}_{n,k}}$
    

$\textcolor{gray}{\rhd{\textit{ Blockchain performs:}}}$

Record \{$h(t^*_k), {\mathbf{w}^{(t)}}$\}.
}
$\textcolor{gray}{\rhd{\textit{ Blockchain performs:}}}$

Settle the accumulated rewards to type-$k$ clients according to Eq.~\eqref{eq:define_reward}.

\textbf{Return} {$\mathbf{w}^{(T)}$}.
\end{algorithm}

%% file: Sections_revision_wt_trackchanges/4_Experiment.tex
\section{Experimental Results}
\label{sec:experiment}



\begin{figure*}[ht!]
\centering
\begin{subfigure}[t]{.32\textwidth}
    \centering
    \includegraphics[width=\linewidth]{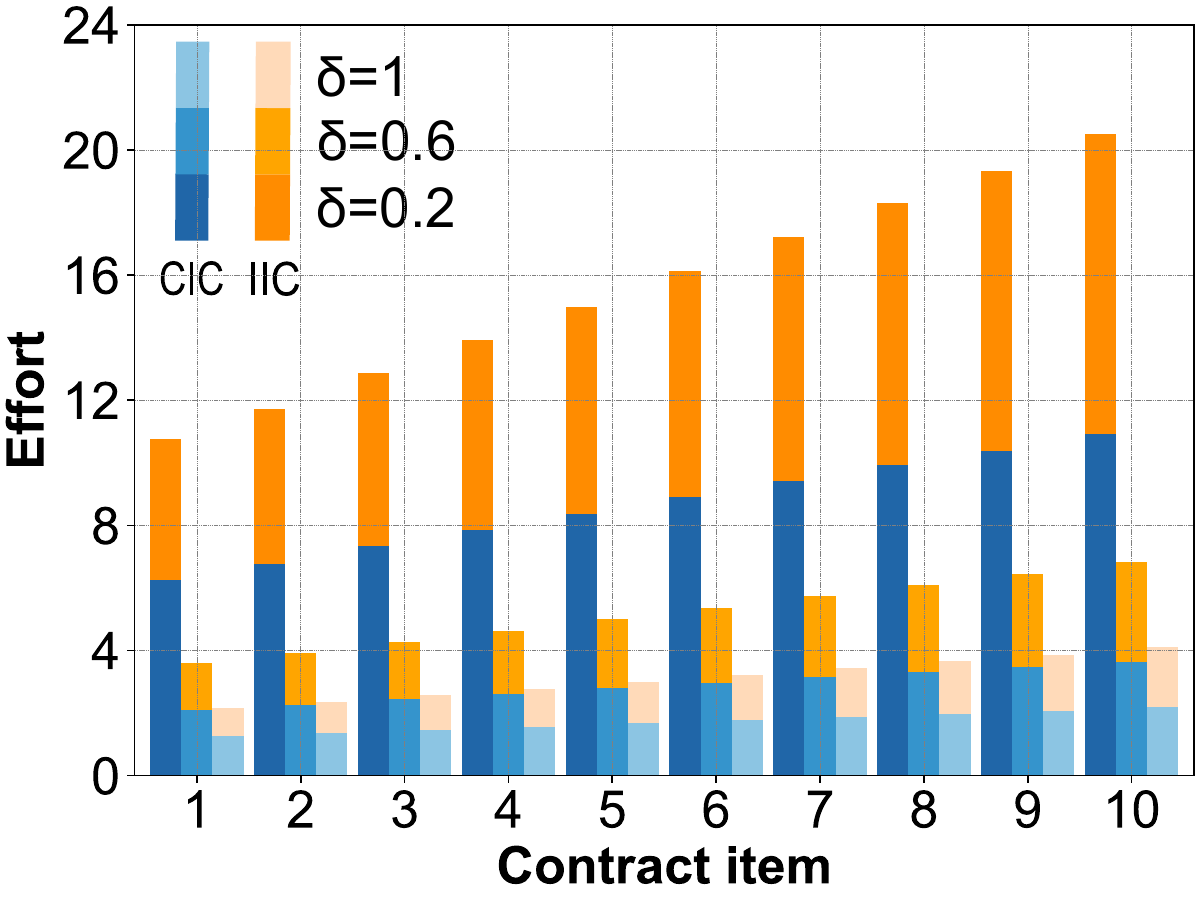}  
    \caption{Client's effort.}
    \label{cf:delta_effort}
\end{subfigure} \hfill
\begin{subfigure}[t]{.32\textwidth}
    \centering
    \includegraphics[width=\linewidth]{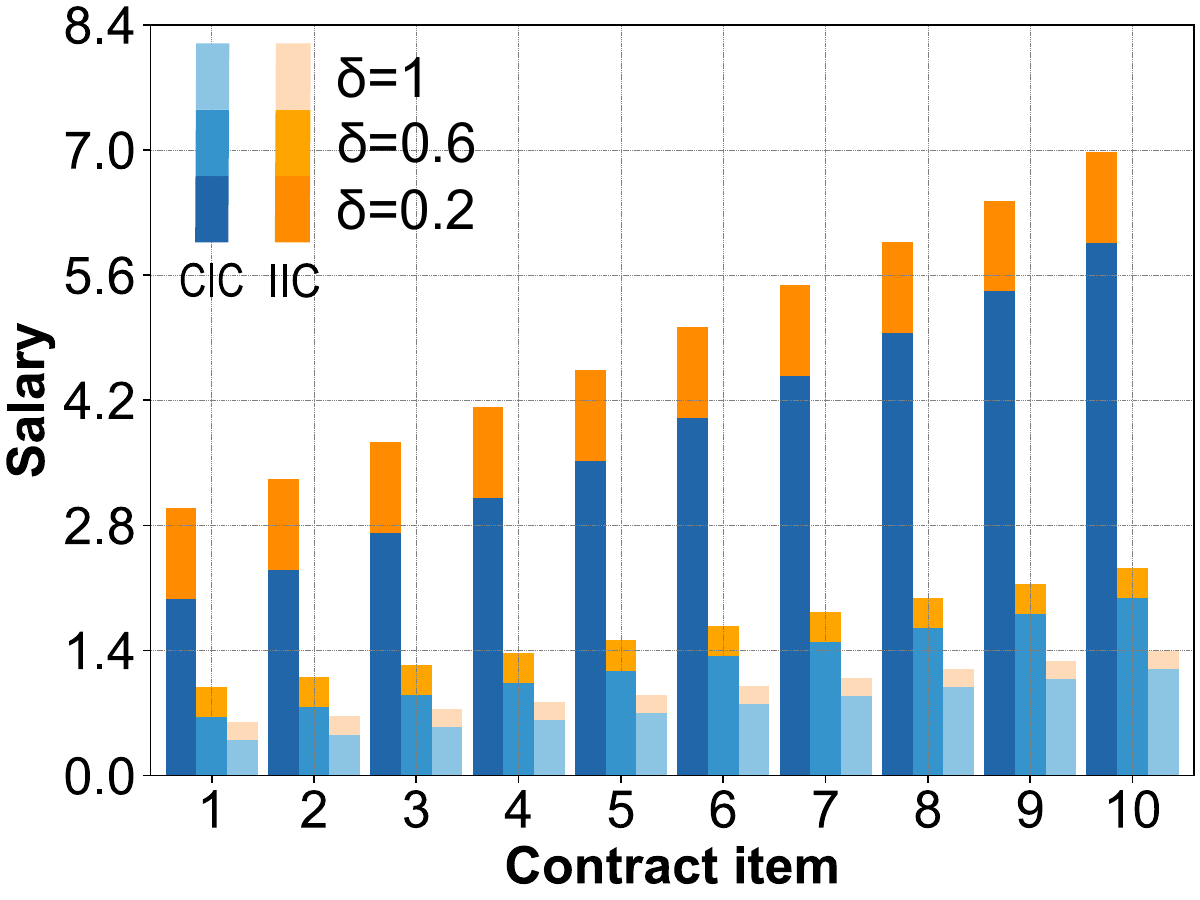}  
    \caption{Client's fixed salary.}
    \label{cf:delta_fixed_salary}
\end{subfigure} \hfill
\begin{subfigure}[t]{.32\textwidth}
    \centering
    \includegraphics[width=\linewidth]{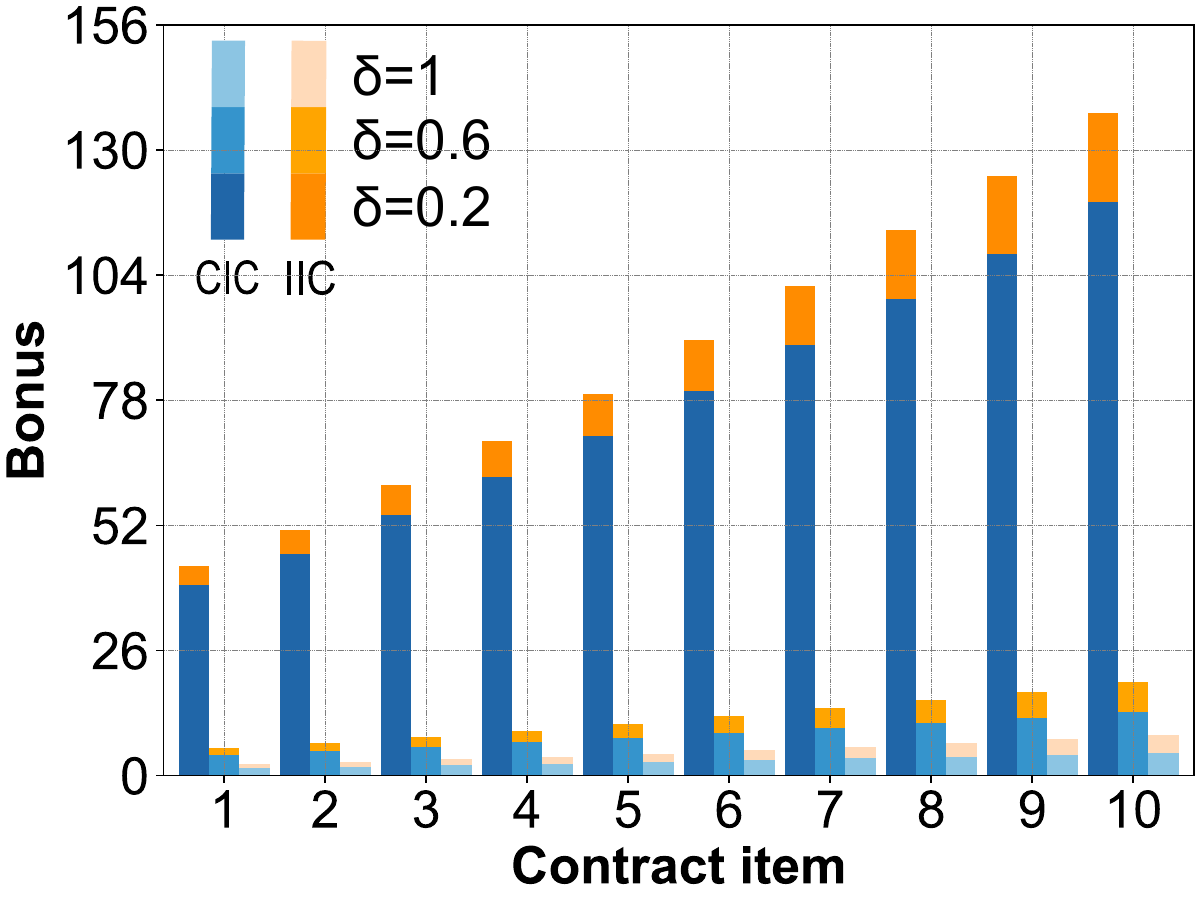}  
    \caption{Client's bonus.}
    \label{cf:delta_client_bonus}
\end{subfigure} \hfill
\begin{subfigure}[t]{.32\textwidth}
    \centering
    \includegraphics[width=\linewidth]{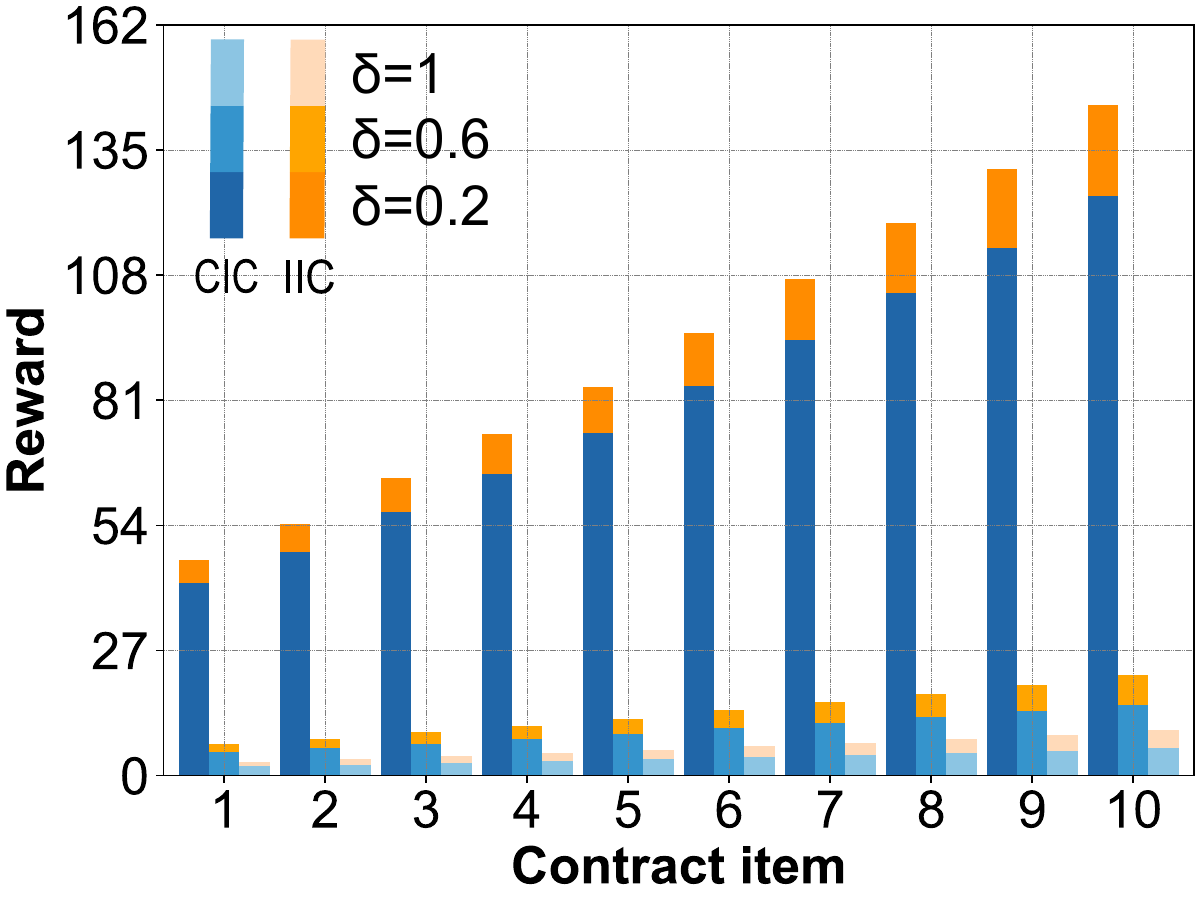}  
    \caption{Client's reward.}
    \label{cf:delta_client_reward}
\end{subfigure} \hfill
\begin{subfigure}[t]{.32\textwidth}
    \centering
    \includegraphics[width=\linewidth]{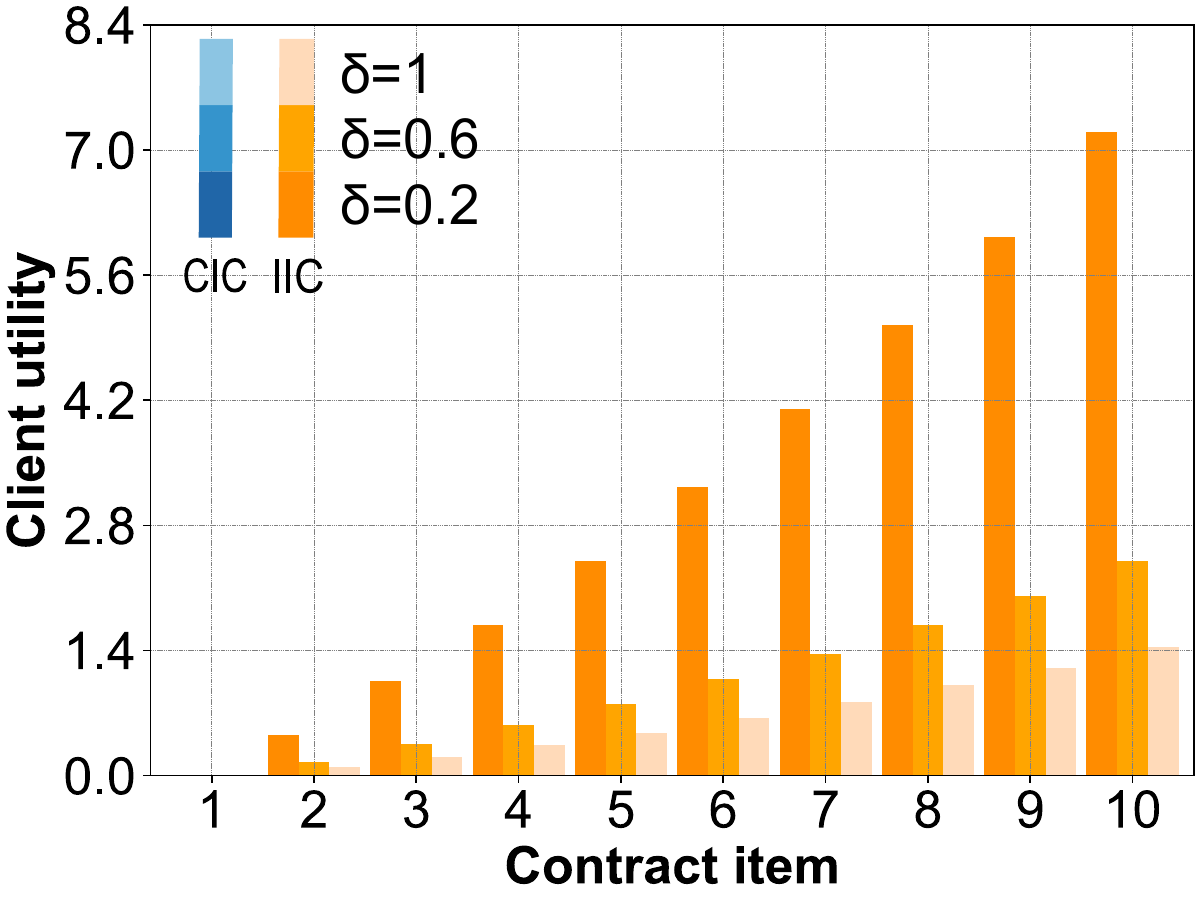}  
    \caption{Client utility.}
    \label{cf:delta_client_utility}
\end{subfigure} \hfill
\begin{subfigure}[t]{.32\textwidth}
    \centering
    \includegraphics[width=\linewidth]{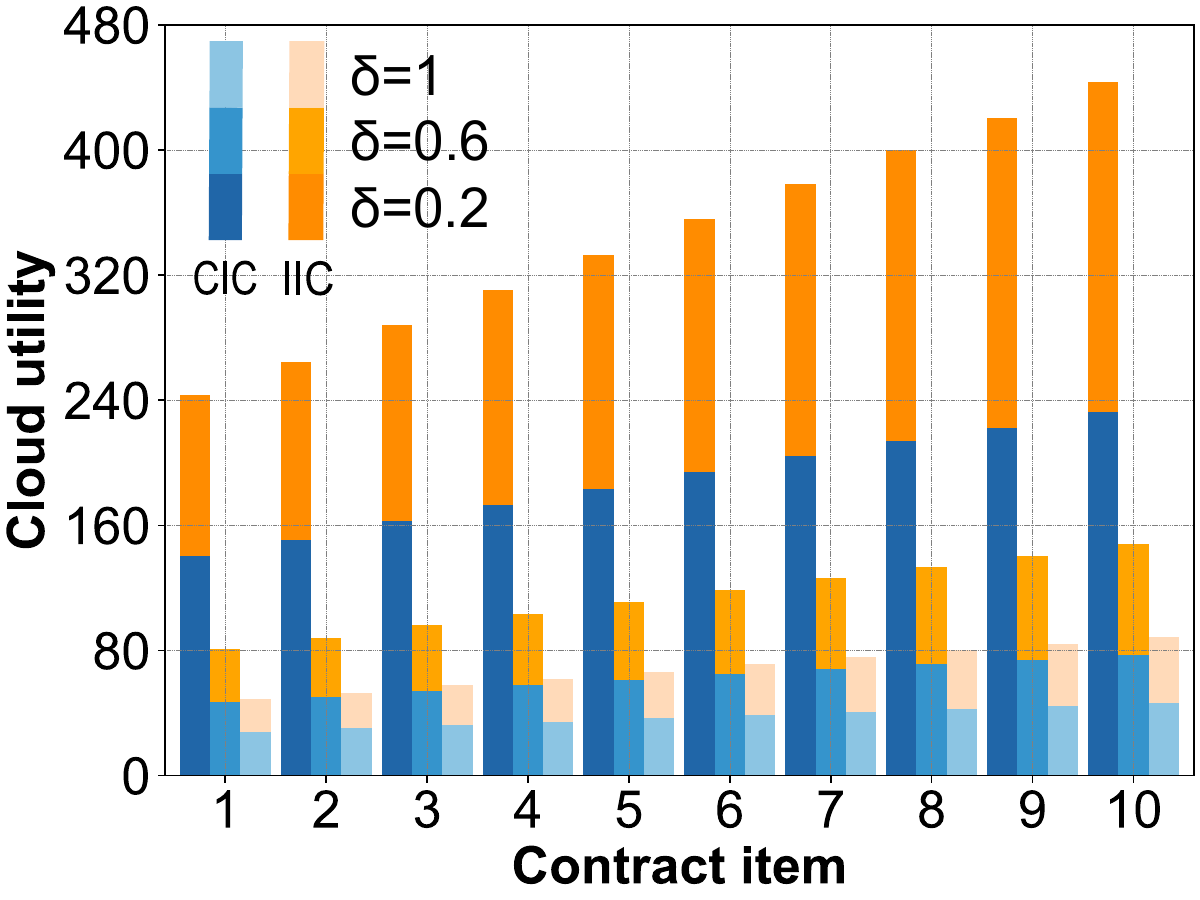}  
    \caption{Cloud utility.}
    \label{cf:delta_cloud_utility}
\end{subfigure} \hfill
\caption{The impact of unit effort cost $\delta$ on the client's effort, fixed salary, bonus, total reward, client utility, and cloud utility.}
\label{fig:impact_unit_effort_cost}
\end{figure*}
In this section, we conduct extensive experiments to analyze the feasibility and efficiency of {\proposed}. Then, we evaluate {\proposed}'s proof of concept on prominent benchmark datasets.

\subsection{Experiment Setup}\label{subsec:exp_setup}
\paragraphalphastyle
\setcounter{paragraph}{0}

\paragraph{Simulation parameter settings}\label{para:simulation_settings} Without loss of generality, our experimental settings refer to the setups in \cite{zhan2020learning, zhao2021incentive}. {We instantiate the performance-gain function as $g(\cdot)=h(t)e$. This affine function is a weakly concave special case of the general concave $g(\cdot)$ used in the system model. We also note that the main contract-theoretic results derived in the preceding sections do not rely on the strict concavity of $g(\cdot)$}. { To validate both contract feasibility and contract efficiency, we consider 25 training rounds (i.e., $T_{\text{sim}} = 25$), with the duration of {\clps} set to 10 rounds. Unless otherwise stated, key parameter settings are summarized in Table~\ref{tab_main:parameters}.}

\paragraph{Proof of concept} We develop a system to implement a proof of concept of {\proposed}, leveraging smart contracts on an Ethereum blockchain using Ganache \cite{ganache}. The Flower framework \cite{beutel2020flower} is integrated with the blockchain to facilitate federated learning experiments for model training and evaluation.  The system is deployed on an Ubuntu 22.04 system equipped with an Intel Core i7-1365U CPU@1.80 GHz and 32GB of memory. To enable interaction between clients, the cloud server, and the smart contracts, we utilize the Web3 API \cite{web3api} and FastAPI frameworks \cite{fastapi}.  Also, we develop two smart contracts for federated training and storage, and reward calculation using the Solidity programming language. 


\input{Tables/Parameter_settings}

{\paragraph{Model and Datasets}\label{para:model_data} We employ a convolutional neural network comprising three convolutional blocks and one fully connected classification head.  Each convolutional block consists of a Conv2D layer {{with He uniform initialization and ReLU activation}}, followed by batch normalization and, where applicable, max pooling for spatial dimension reduction. {{Dropout regularization is applied after the first block with a rate of 0.3, after the third block with a rate of 0.2, and after the dense layer with a rate of 0.1}}. The classification head consists of a 256-unit dense layer with ReLU activation and a 10-class softmax output layer. The model is optimized using Adam with a learning rate of {{$\eta = 3 \times 10^{-4}$}}, sparse categorical cross-entropy loss, and a local batch size of $B = 32$. Each client performs 2 local training epochs per communication round, over a total of $T_{\text{poc}} =81$ global rounds in proof-of-concept experiments.

For model training and evaluation, we use CIFAR-10 \cite{krizhevsky2009learning} and Fashion-MNIST (FMNIST)~\cite{xiao2017fashion} datasets under non-independent and non-identically distributed (non-IID) settings. To simulate statistical heterogeneity, both the number of samples and the label distribution across clients are unbalanced by sampling $p\sim\textit{Dir}{(\alpha)}$, where $\alpha \in \{0.1, 0.4, 0.7\}$ controls the degree of heterogeneity. Smaller values of $\alpha$ correspond to greater data skewness. {{Furthermore, without loss of generality, the number of data samples assigned to each client is sorted in ascending order with respect to client index, such that clients with higher IDs possess larger and more diverse local datasets.}}}

\paragraph{Benchmarks}\label{para:benchmark} We consider the following benchmarks in {\fontfamily{qcr}\selectfont CLPs} and {\fontfamily{qcr}\selectfont non-CLPs} for comparison with our proposed methods, which are labeled as {\proposed} in Complete Information Case ({\proposed}-CIC) and {\proposed} in Incomplete Information Case ({\proposed}-IIC).

\begin{itemize}
    \item \textbf{Contract theory Without considering Time under CIC (CTWT-CIC)}. Time property is not considered in the existing work (e.g., \cite{zhang2015contract,sun2021pain,zhan2020learning, wang2022infedge}), that is, $t$ is excluded from the contract item. All training time slots are treated equally.
    \item \textbf{Contract theory Without considering Time under IIC (CTWT-IIC)}. Similar to CTWT-CIC, this method does not consider temporal importance. In addition, the cloud server only knows the distribution of clients’ types.
    \item \textbf{Linear Pricing}. Under IIC, the cloud server determines a unit price $C$ for the client's effort $e$ \cite{bolton2004contract}.
    \item \textbf{Conventional Federated Learning ($\mathrm{CFL}$)}. FL methods integrating conventional incentive mechanisms (e.g., \cite{wang2022infedge, sun2021pain, ding2023joint}), which do not consider {\fontfamily{qcr}\selectfont CLPs}, are used to compare with {\proposed} proof of concept.
\end{itemize}

\begin{figure}[t!]
\centering
\begin{subfigure}[t]{.375\textwidth}
    \centering
    \includegraphics[width=\linewidth]{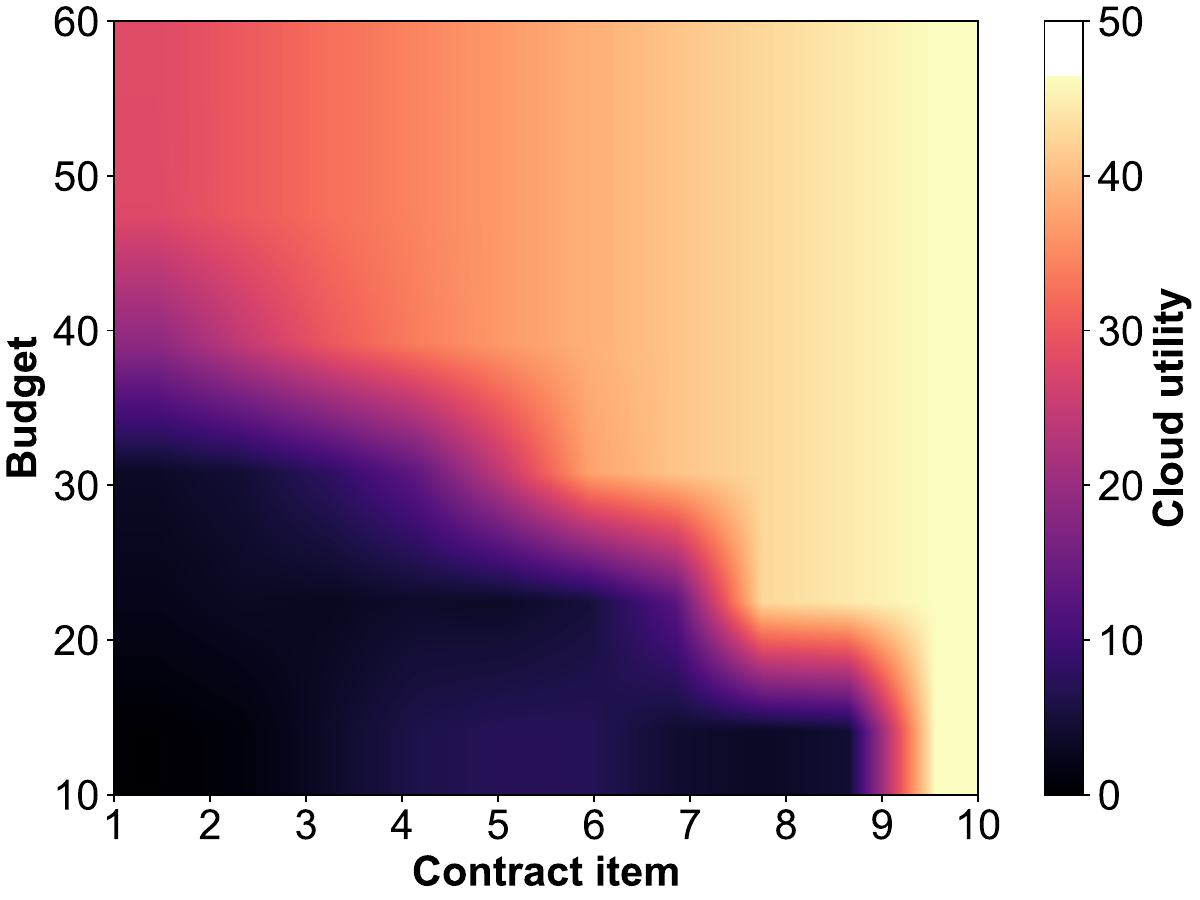}  
    \caption{Under CIC.}
    \label{cf:in_clps}
\end{subfigure}
\begin{subfigure}[t]{.375\textwidth}
    \centering
    \includegraphics[width=\linewidth]{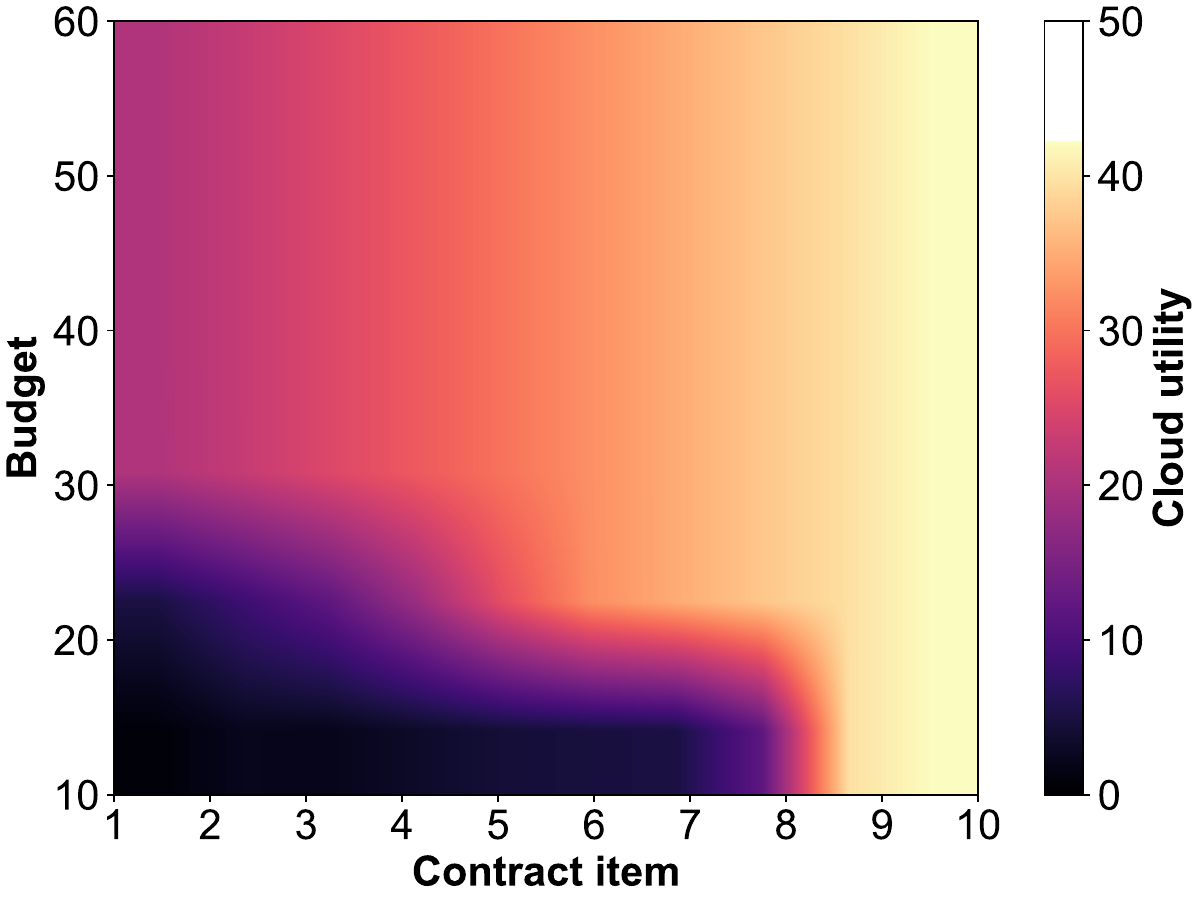}  
    \caption{Under IIC.}
    \label{cf:in_non_clps}
\end{subfigure}
\caption{The impact of cloud budget $P$ on cloud utility per contract item and total cloud utility.}
\label{fig:impact_of_budget}
\end{figure}

\subsection{Contract Feasibility}\label{subsec:feasibility}
\paragraphalphastyle
\setcounter{paragraph}{0}

We first examine the impacts of varying unit effort cost $\delta$ and cloud budget $P$ on {\proposed} over $N = 10$ clients in {\proposed} during the initial training round of {\fontfamily{qcr}\selectfont CLPs}. Then, we evaluate cloud utility, effort, reward, and client utility against contract items (i.e., client types) in {\proposed} during the initial training round of {\fontfamily{qcr}\selectfont CLPs} and {\fontfamily{qcr}\selectfont non-CLPs} starting at $t=1$ and $t=11$, respectively, where we set the cloud budget $P = 60$ per round and $\delta = 1$.

\paragraph{Impacts of different unit effort costs} In Figs.~ \ref{fig:impact_unit_effort_cost}(\subref{cf:delta_effort})-(\subref{cf:delta_cloud_utility}), under CIC and IIC in {\fontfamily{qcr}\selectfont CLPs}\footnote{We focus on one global training round in \text{{\fontfamily{qcr}\selectfont CLPs}} here, but analogous trends observed in {\fontfamily{qcr}\selectfont non-CLPs}.}, the client's fixed salary, bonus, effort, total reward, utility, and cloud utility have shown decreasing trends per contract item along with the increase of unit effort cost $\delta$, where we set $\delta = \{0.2, 0.6, 1\}$. In other words, clients are not willing to join and contribute efforts when $\delta$ becomes large. For example, given the tenth contract item under the IIC cases, the values of the client's effort units are around \{9.60, 3.20, 1.92\} corresponding to $\delta = \{0.2, 0.6, 1\}$. Besides, under CIC, where the cloud server possesses complete client information, the cloud can fine-tune and optimize its contract items to incentivize clients to join early and contribute more effort. This results in higher cloud utility while offering clients higher fixed salaries and bonuses, compared to IIC (e.g., 232.28 vs. 211.13, 5.96 vs. 1.01, and 119.21 vs. 18.47 given the tenth contract item with $\delta = 0.2$, respectively). In addition, due to being aware of client information, the cloud server exploits the client's efforts as much as possible while guaranteeing IR constraints (cf. Theorem~\ref{theorem_1}) in CIC, leading to zero client utility across all contract items (cf., Fig.~\ref{fig:impact_unit_effort_cost}(\subref{cf:delta_client_utility})).
\begin{figure*}[ht!]
\centering
\begin{subfigure}[t]{.32\textwidth}
    \centering
    \includegraphics[width=\linewidth]{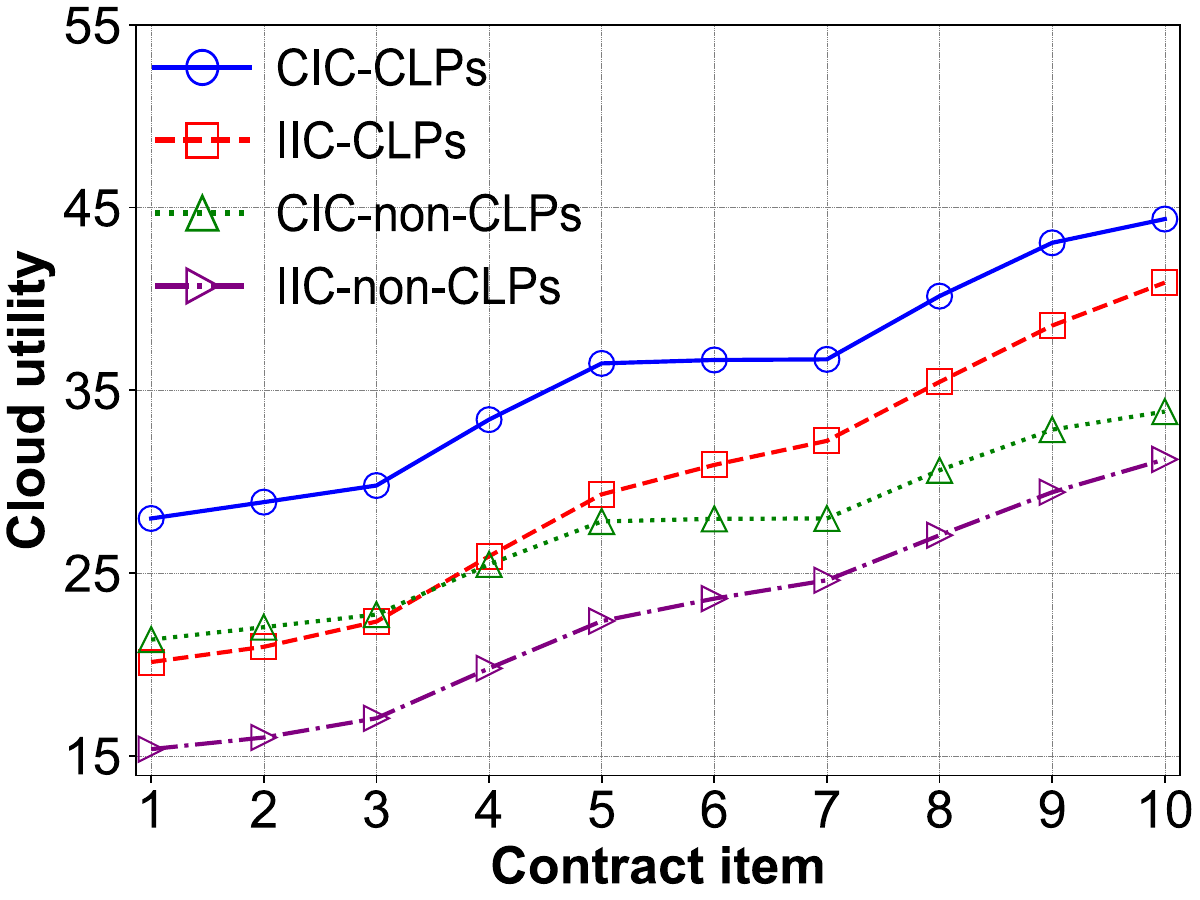}  
    \caption{Cloud utility.}
    \label{cf:cloud_utility}
\end{subfigure}\;
\begin{subfigure}[t]{.32\textwidth}
    \centering
    \includegraphics[width=\linewidth]{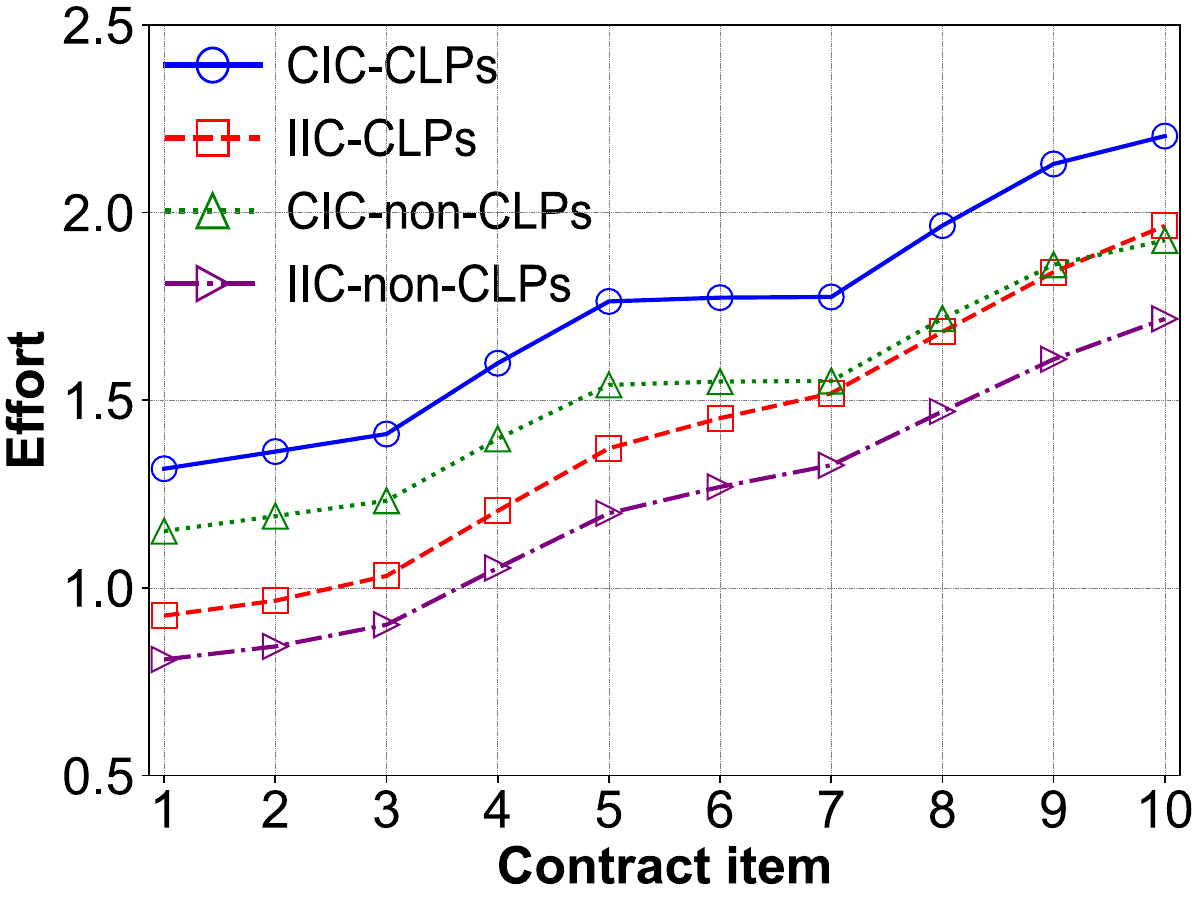}  
    \caption{Effort.}
    \label{cf:effort}
\end{subfigure}\;
\begin{subfigure}[t]{.32\textwidth}
    \centering
    \includegraphics[width=\linewidth]{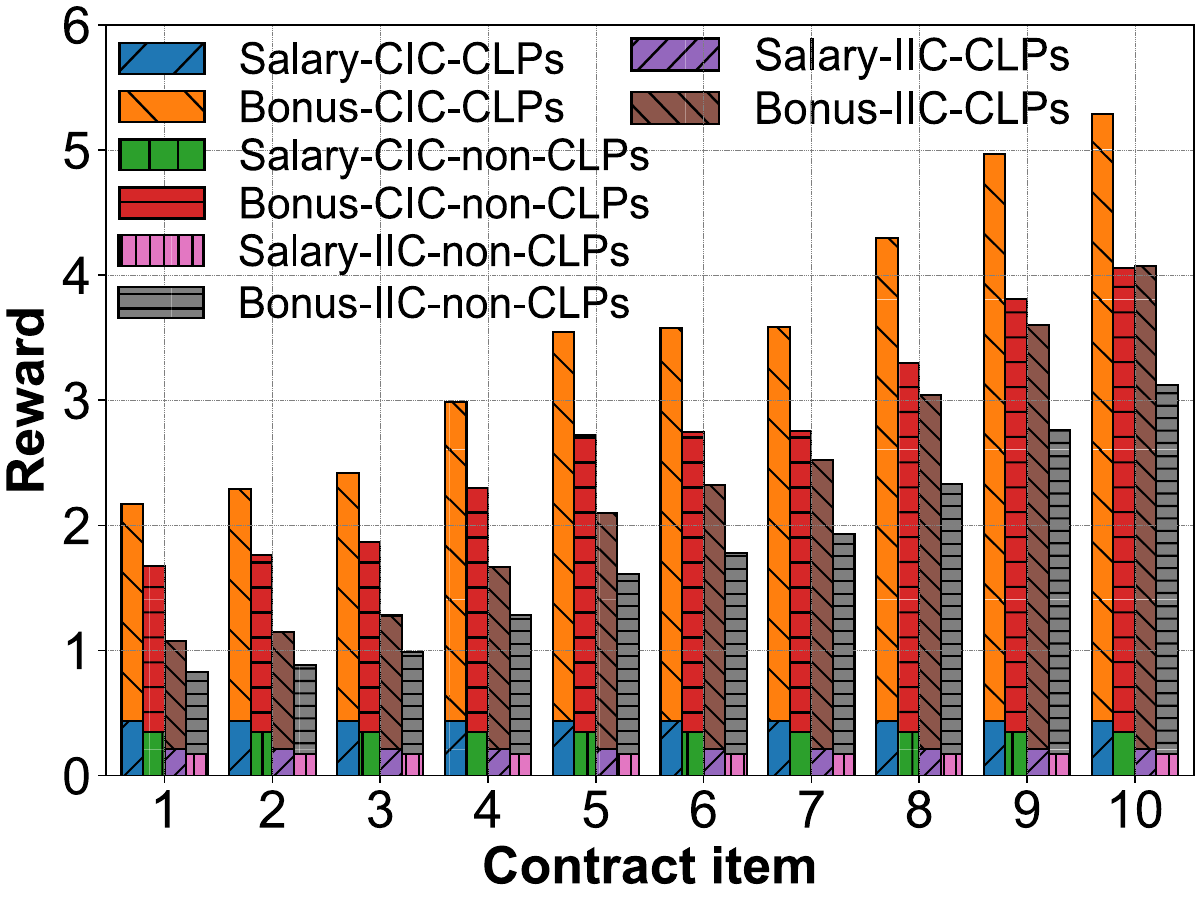}  
    \caption{Reward.}
    \label{cf:reward}
\end{subfigure}
\caption{Performance evaluation of {\proposed} per training round.}
\label{fig:contract_feasibility}
\end{figure*}

\paragraph{Impacts of different cloud budgets} Figs. ~\ref{fig:impact_of_budget}(\subref{cf:in_clps})-(\subref{cf:in_non_clps}) show the effect of varying cloud budgets $P$ on the cloud utility contributed by type-$k$ clients (i.e., $k$-$\text{th}$ contract items) in a global training round under both CIC and IIC  in \text{{\fontfamily{qcr}\selectfont CLPs}}. As $P$ increases, more clients can be incentivized to participate in and allocate efforts, resulting in greater total cloud utility over all joined clients. For instance, at $P = 60$, the total cloud utility of all type-$k$ clients reaches approximately 375 in CIC and 312 in IIC, compared to 266 in CIC and 311 in IIC at $P = 30$, respectively. Especially at low budgets (i.e., $P = \{10, 20, 30\}$), the total cloud utility under IIC exceeds that under CIC. This is supported by Assumption~\ref{assump:rationality} and the fact that the cloud server is aware of clients' types in CIC and only the distribution of clients' types in IIC, thereby resulting in different optimal contract designs. Consequently, in CIC, the cloud server may adopt a greedy allocation strategy to exhaust efforts from higher-type clients (cf. Theorem \ref{theorem_1}), but due to a limited budget, it leads to fewer participating clients and thus lower total cloud utility compared with that of IIC. Conversely, when the budget is huge enough (i.e., $P \geq 40$), the total cloud utility in CIC is far larger than that in IIC (i.e., 375 versus 311 at $P = 60$).

\paragraph{Performance evaluation} We consider cloud utility, client's effort, and client's reward according to contract items in a training round in {\fontfamily{qcr}\selectfont CLPs} and {\fontfamily{qcr}\selectfont non-CLPs} under conditions of CIC and IIC as shown in Figs.~\ref{fig:contract_feasibility}(\subref{cf:cloud_utility})-(\subref{cf:reward}). In Fig.~\ref{fig:contract_feasibility}(\subref{cf:cloud_utility}), the cloud server achieves the maximum utilities in both
{\fontfamily{qcr}\selectfont CLPs} and {\fontfamily{qcr}\selectfont non-CLPs} in CIC compared to IIC. This outcome occurs because the cloud server has complete knowledge of the client's capabilities, enabling it to fully leverage clients' effort and time without breaching IR constraints, thereby maximizing utility as per Theorem~\ref{theorem_1}. Conversely, under IIC,  the cloud server does not obtain complete information but only the distribution of the clients' types; therefore the cloud server's utility remains constrained by the utility attainable under CIC.
Figs. \ref{fig:contract_feasibility}(\subref{cf:effort})-(\subref{cf:reward}) contrast efforts and rewards (comprising salaries and bonuses) associated with different contract items. These results validate the monotonicity conditions, which are detailed in Lemma~\ref{lem:monotonicity} and Lemma~\ref{lem:monotonicity-nonclp}, and Remark~\ref{remark:nonclp}. As can be seen from the figures,  higher-type clients contribute more effort and, consequently, receive higher rewards.

 \begin{figure}[t!]
    \centering
    \includegraphics[width=.40\textwidth]{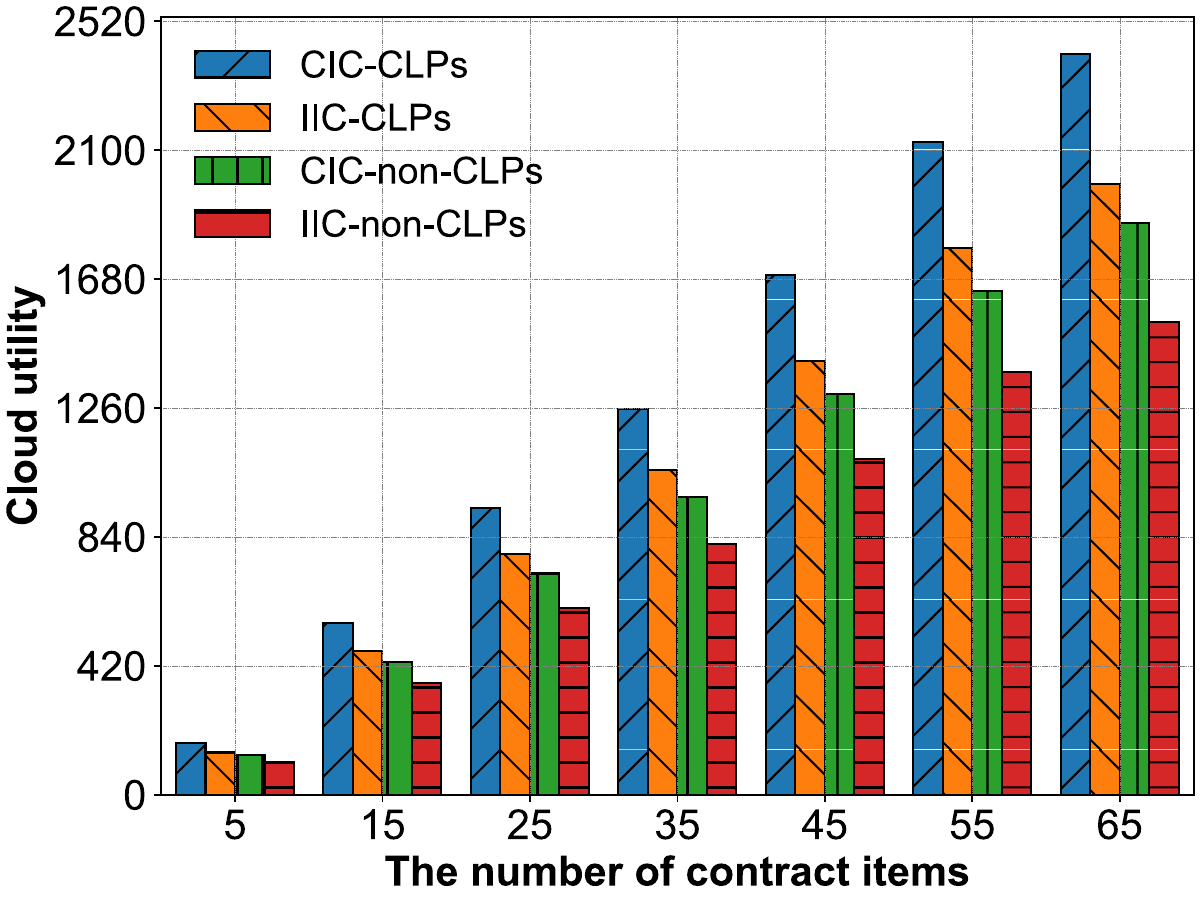}  
    \caption{Cloud utility with a different number of contract items.}
    \label{fig:contract_scalability}
\end{figure}

{\paragraph{Scalability and Privacy}\label{para:scalability} In Fig.~\ref{fig:contract_scalability}, we examine the effect of the number of contract items, denoted by $K$, on the cloud utility. As $K$ increases, the cloud server partitions clients into finer training capability types, which improves the precision of reward allocation and enables more accurate differentiation among heterogeneous clients. This allows the cloud to better motivate higher-type clients to participate earlier and exert greater effort. Accordingly, {\proposed} remains applicable to large and heterogeneous environments, such as the Internet-of-Things (IoT). {However, increasing $K$ also introduces important trade-offs. First, finer type partitioning increases the granularity of capability inference under the self-revelation mechanism. Upon contract item selection, the cloud may precisely identify each client's private training capability, raising privacy concerns in highly regulated applications such as healthcare. Second, a larger $K$ increases the reward budget required to satisfy the IC, IR, and BF constraints across all client types, thereby intensifying the budget pressure on the cloud. Finally, a larger $K$ also enlarges the optimization space and the number of feasibility checks, increasing the computational burden of contract design.}}

\paragraph{Contract feasibility} In Fig. \ref{fig:client_utility}, the utilities of four client types are compared when they select the same contract item. Their utilities follow the inequality $U_1 < U_4 < U_7 < U_{10}$ and each client can achieve maximum utility if and only if they select the contract item designed for their types, corroborating the results presented in Lemma~\ref{lem:client_utility} and Lemma~\ref{lem:monotonicity}, respectively, which explain the IC constraint. This contract design ensures that clients' types are automatically disclosed to the cloud server upon contract selection, effectively addressing information asymmetry between clients and the cloud server. Furthermore, clients receive non-negative utilities when they choose the contract items corresponding to their types, which validates the IR constraint.

\subsection{Contract Efficiency}\label{subsec:efficiency}

Fig.~\ref{fig:contract_efficiency1} and Fig.~\ref{fig:contract_efficiency2} examine the efficiency of {\proposed} with CTWT and linear pricing with the same cloud budget $P=60$, where the number of clients is $N = 15$ and unit effort cost $\delta = 1$. In this comparison, we assume that {\fontfamily{qcr}\selectfont CLPs} span from $t=1$ to $t=10$, and {\fontfamily{qcr}\selectfont non-CLPs} begin from $t=11$ to the end.


\paragraph{Cloud utility} Fig. \ref{fig:contract_efficiency1}(\subref{ce:cloud_utility}) shows cloud utilities in the {\fontfamily{qcr}\selectfont CLPs} and {\fontfamily{qcr}\selectfont non-CLPs}. {\proposed} demonstrably outperforms other methods by effectively attracting the highest-quality contributions in the {\fontfamily{qcr}\selectfont CLPs}. Consequently, the cloud server maintains a significantly higher utility level compared with other approaches during the {\fontfamily{qcr}\selectfont CLPs}, which determine the final model performance. Although the cloud utility of {\proposed} shows a gradual decrease, which becomes more pronounced once the learning and training process transitions to the {\fontfamily{qcr}\selectfont non-CLPs} (i.e., $t \ge 11$), {\proposed}'s total cloud utility still achieves higher results in comparison with conventional incentive benchmarks, as shown in Fig. \ref{fig:contract_efficiency1}(\subref{ce:total_cloud_utility}). For example, {\proposed} achieves cloud utility gains of up to approximately 7.2\% and 23.1\% over CTWT-IIC and linear pricing methods, respectively, in the IIC. 

Moreover, {\proposed} demonstrates higher economic efficiency compared to conventional benchmarks, as shown in Fig. \ref{fig:contract_efficiency1}(\subref{ce:economic_efficienct}).  \textit{Economic efficiency reflects the effectiveness of the cloud server in calculating and distributing \textbf{the total  reward to clients} based on their efforts and joining times}. For instance, {\proposed} allocates approximately  20\% and 96\% less total reward compared with CTWT and linear pricing, respectively, while still reaching higher final cloud utility in the IIC setting. \textit{This highlights the importance of providing the right reward at the right time}. 

\begin{figure}[t!]
    \centering
    \includegraphics[width=.40\textwidth]{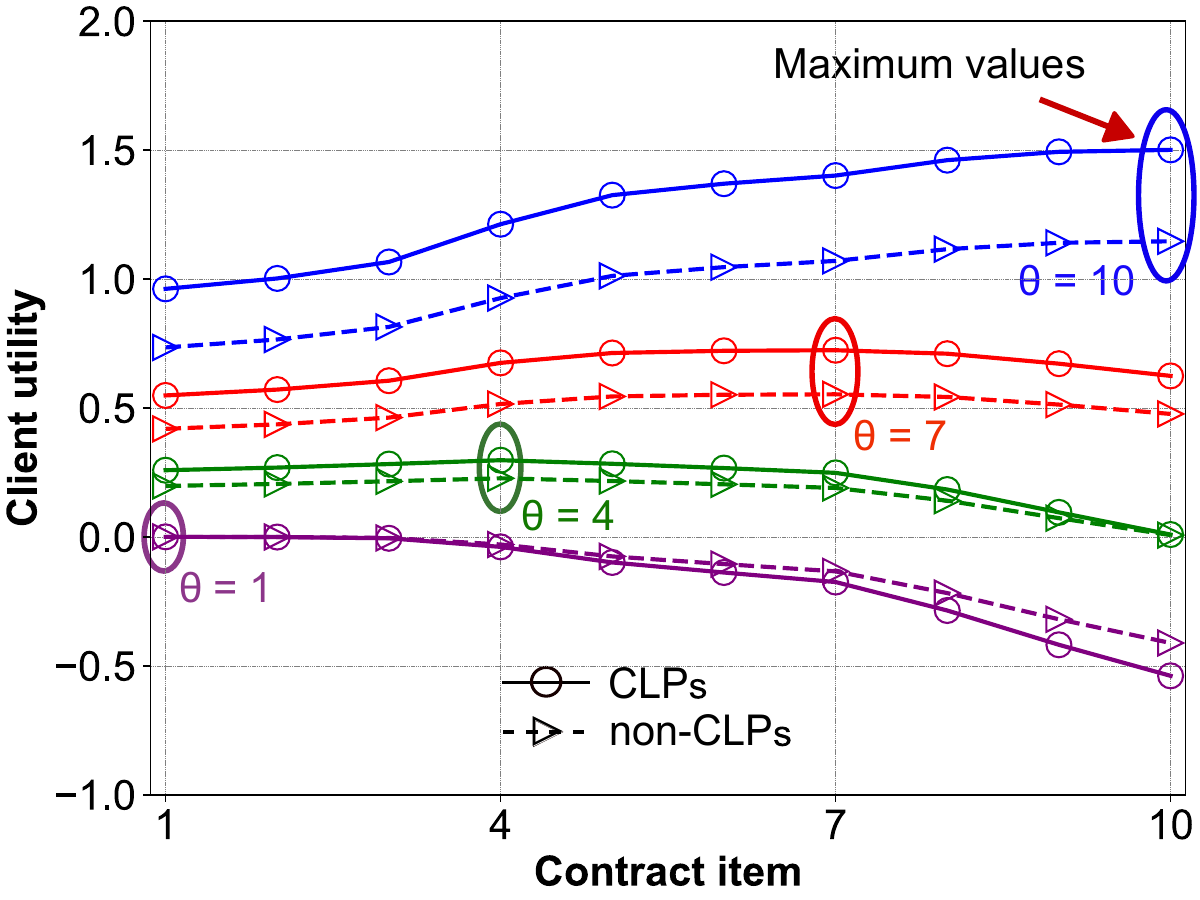}  
    \caption{Client utility under different contract items.}
    \label{fig:client_utility}
\end{figure}

\begin{figure*}[t!]
\centering
\begin{subfigure}[t]{.32\textwidth}
    \centering
    \includegraphics[width=\linewidth]{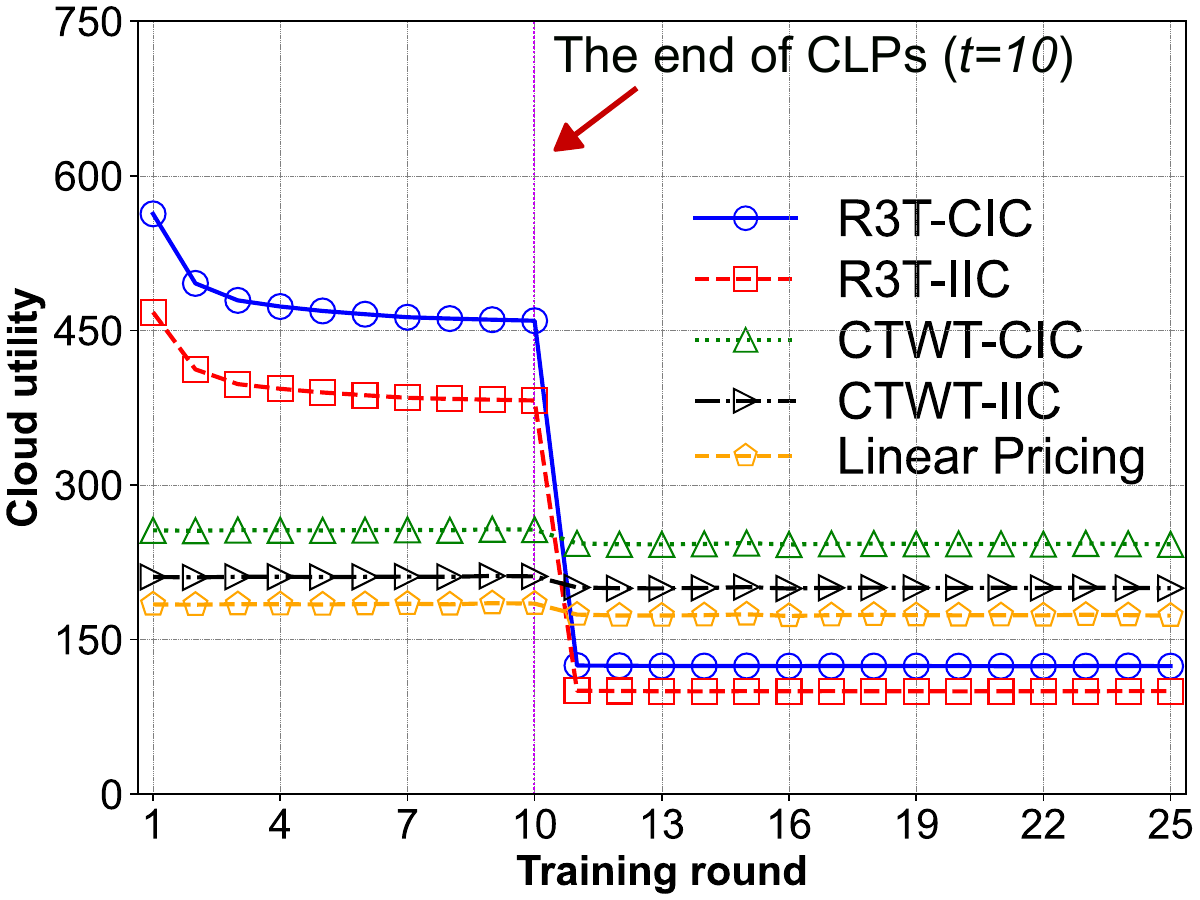}  
    \caption{Cloud utility.}
    \label{ce:cloud_utility}
\end{subfigure}\;
\begin{subfigure}[t]{.32\textwidth}
    \centering
    \includegraphics[width=\linewidth]{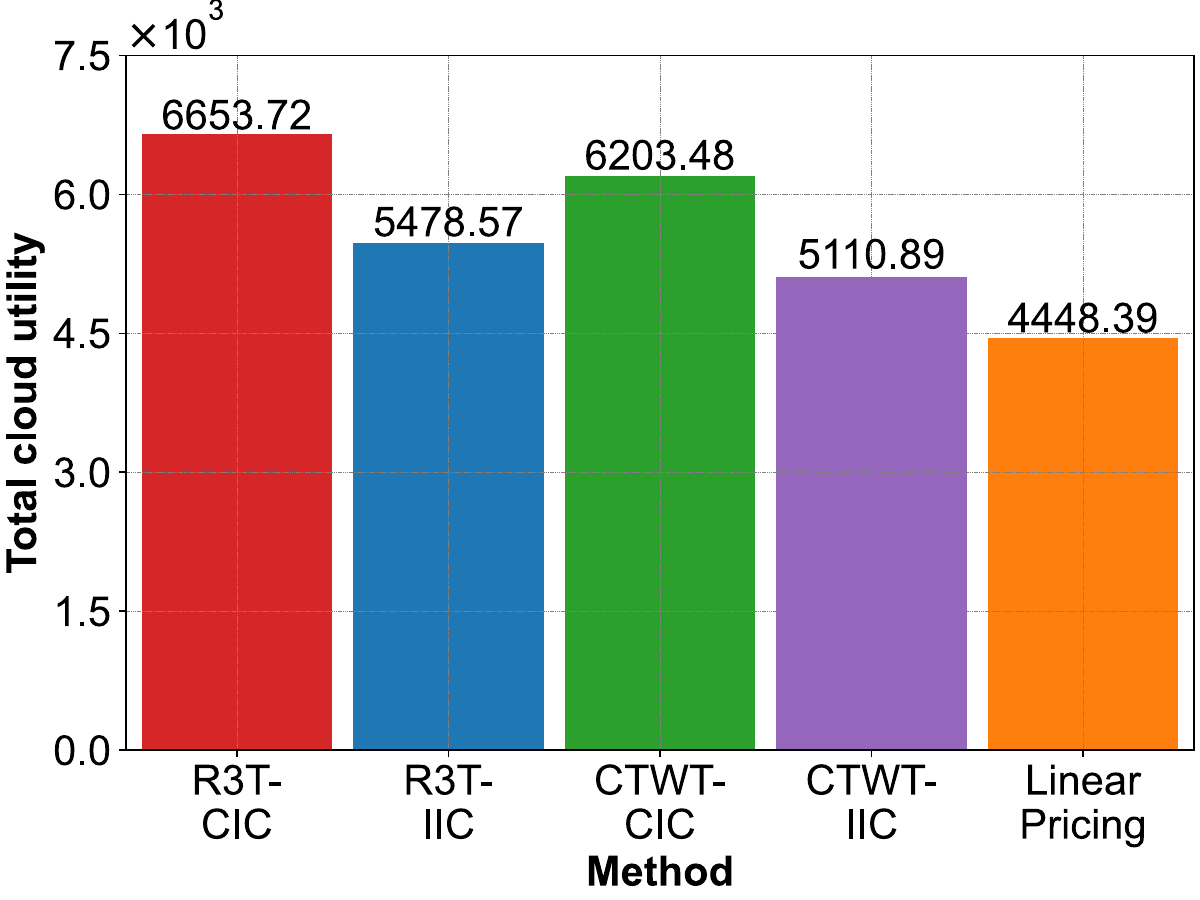}  
    \caption{Total cloud utility.}
    \label{ce:total_cloud_utility}
\end{subfigure}\;
\begin{subfigure}[t]{.32\textwidth}
    \centering
    \includegraphics[width=\linewidth]{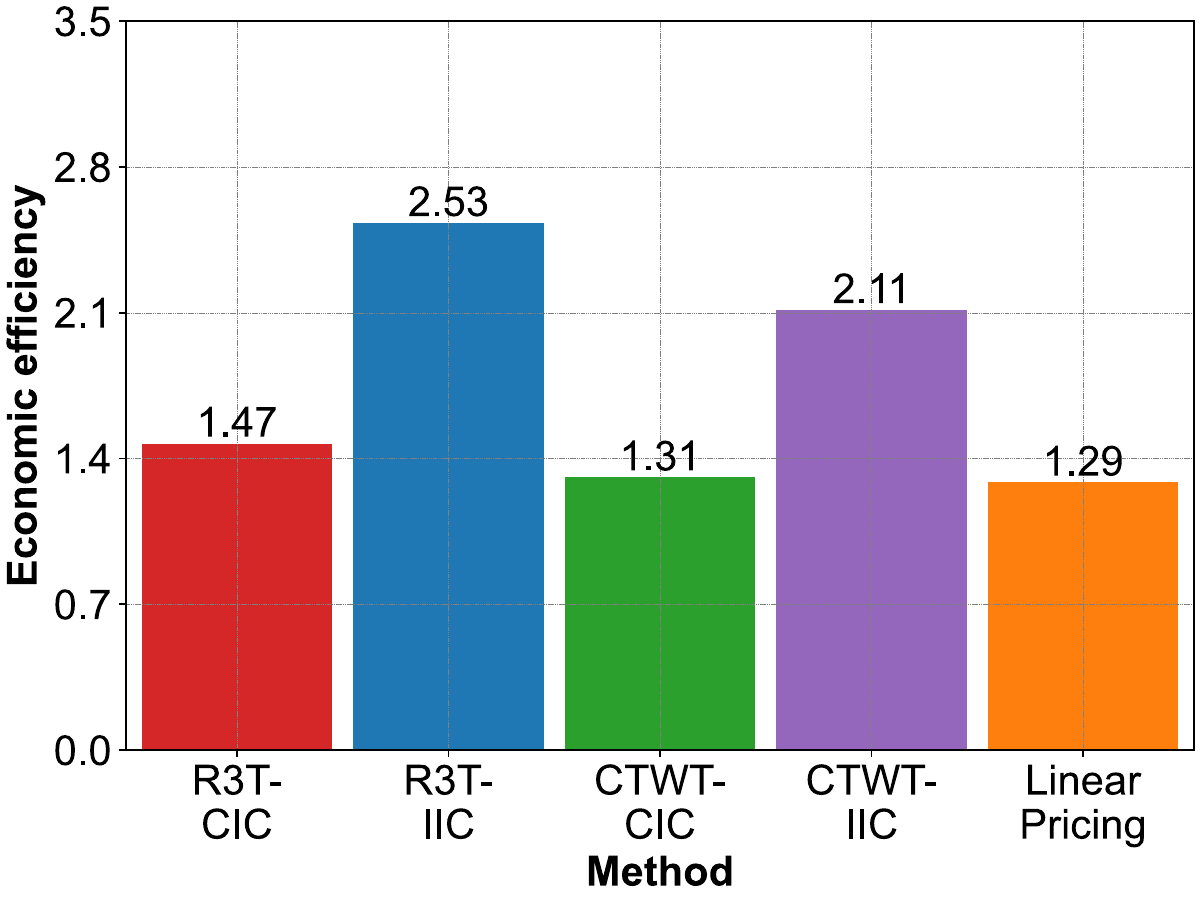}  
    \caption{Economic efficiency.}
    \label{ce:economic_efficienct}
\end{subfigure}
\caption{Cloud gain over $T$ training rounds.}
\label{fig:contract_efficiency1}
\end{figure*}

\begin{figure}[ht!]
\centering
\begin{subfigure}[t]{.4\textwidth}
    \centering
    \includegraphics[width=\linewidth]{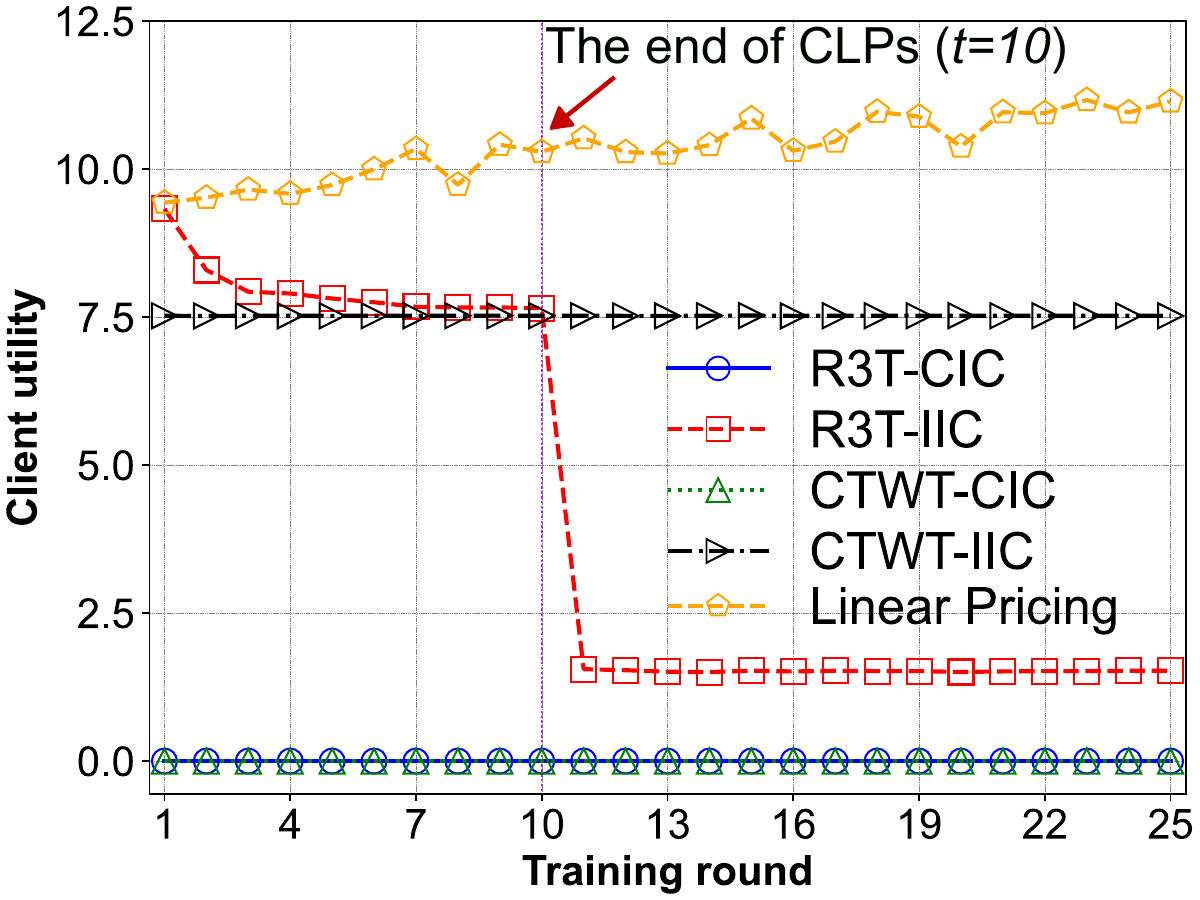}  
    \caption{Client utility.}
    \label{ce:client_utility}
\end{subfigure}
\begin{subfigure}[t]{.4\textwidth}
    \centering
    \includegraphics[width=\linewidth]{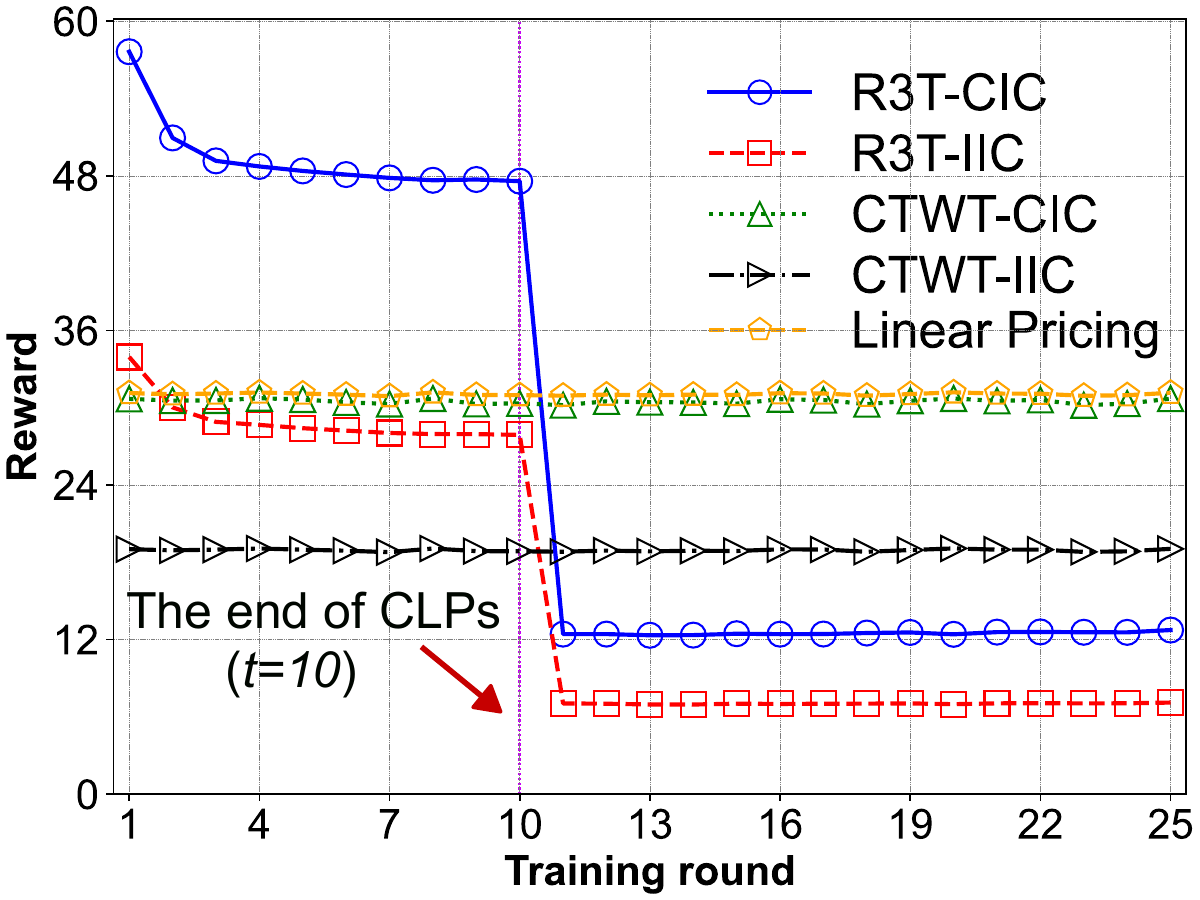}  
    \caption{Client's reward.}
    \label{ce:client_reward}
\end{subfigure}
\caption{Client gain over $T$ training rounds.}
\label{fig:contract_efficiency2}
\end{figure}

\paragraph{Client gain} As the importance of training rounds diminishes in {\fontfamily{qcr}\selectfont non-CLPs}, client utilities trend downward in {\proposed}, as shown in Fig. \ref{fig:contract_efficiency2}(\subref{ce:client_utility}). Meanwhile, the CTWT methods, which treat training rounds equally and randomly select a number of clients per round without early participation incentives, maintain stable client utilities - roughly 7.5 in CTWT-IIC methods and 0 in CTWT-CIC methods as indicated in Theorem~\ref{theorem_1} throughout the $T$ training rounds. The linear pricing method exhibits the lowest cloud utility but the highest client utility due to its fixed price per unit of client effort. This fixed pricing approach limits information gathering for the cloud and prevents optimal price and reward adjustments to match varying client capabilities, efforts, and costs.

Fig. \ref{fig:contract_efficiency2}(\subref{ce:client_reward}) shows the total client reward over $T$ training rounds. While {\proposed} allocates a significant amount of rewards to attract high-quality clients during the {\fontfamily{qcr}\selectfont CLPs}, the amount of rewards of other benchmarks remains generally stable in every training round. This again stems from the strategic allocation of rewards, prioritizing efforts made earlier in the training process in {\proposed}, and further explaining the economic efficiency of the cloud server.

\subsection{Proof of Concept of {\proposed}}{\label{subsec:poc}}
In Fig.~\ref{fig:testbed}, we introduce a proof of concept of {\proposed}. Clients with higher IDs are assigned higher training capabilities, characterized by larger data sizes and increased data diversity. We assume all 10 clients have identical system capabilities and act rationally and strategically {over 81 training rounds (i.e., $T_{\text{poc}} = 81$)} under Assumption~\ref{assump:rationality}. {{\proposed} and $\mathrm{CFL}$ select the same number of clients in each training round (i.e., 
$|\mathcal{S}_{\mathrm{\proposed}}^{(t)}|=|\mathcal{S}_{\mathrm{CFL}}^{(t)}|$, $\forall t\in\mathcal{T}$)}. Besides, {{\fontfamily{qcr}\selectfont CLPs} are detected as ending after training round 30 using the $\mathrm{FGN}$ proxy defined in Section~\ref{subsec:contract}(\ref{para:clp_probing})}. 
Clients register on the blockchain, and their information (e.g., IDs, addresses, efforts, and joined training rounds) is recorded throughout the training. Therefore, non-repudiation and transparency are guaranteed thanks to blockchain properties.

\begin{figure}[t]
	\centering
	\includegraphics[width=0.48\textwidth]{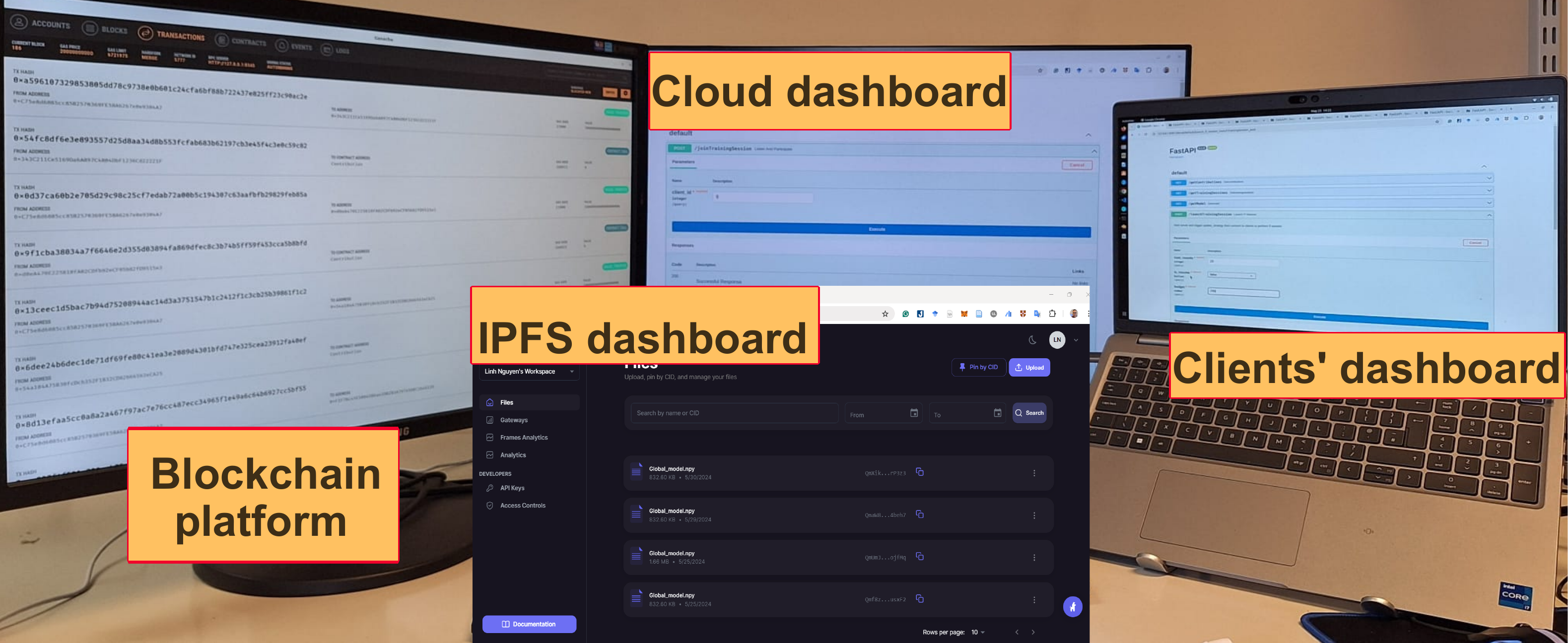}
	\caption{{{\proposed}'s proof of concept. The system includes the Ethereum blockchain, online IPFS storage to address scalability issues, a cloud dashboard, and clients' dashboards.}}
    \label{fig:testbed}
\end{figure}



\begin{figure}[t]
\centering
\begin{subfigure}[t]{.48\textwidth}
    \centering
    \includegraphics[width=\linewidth]
    {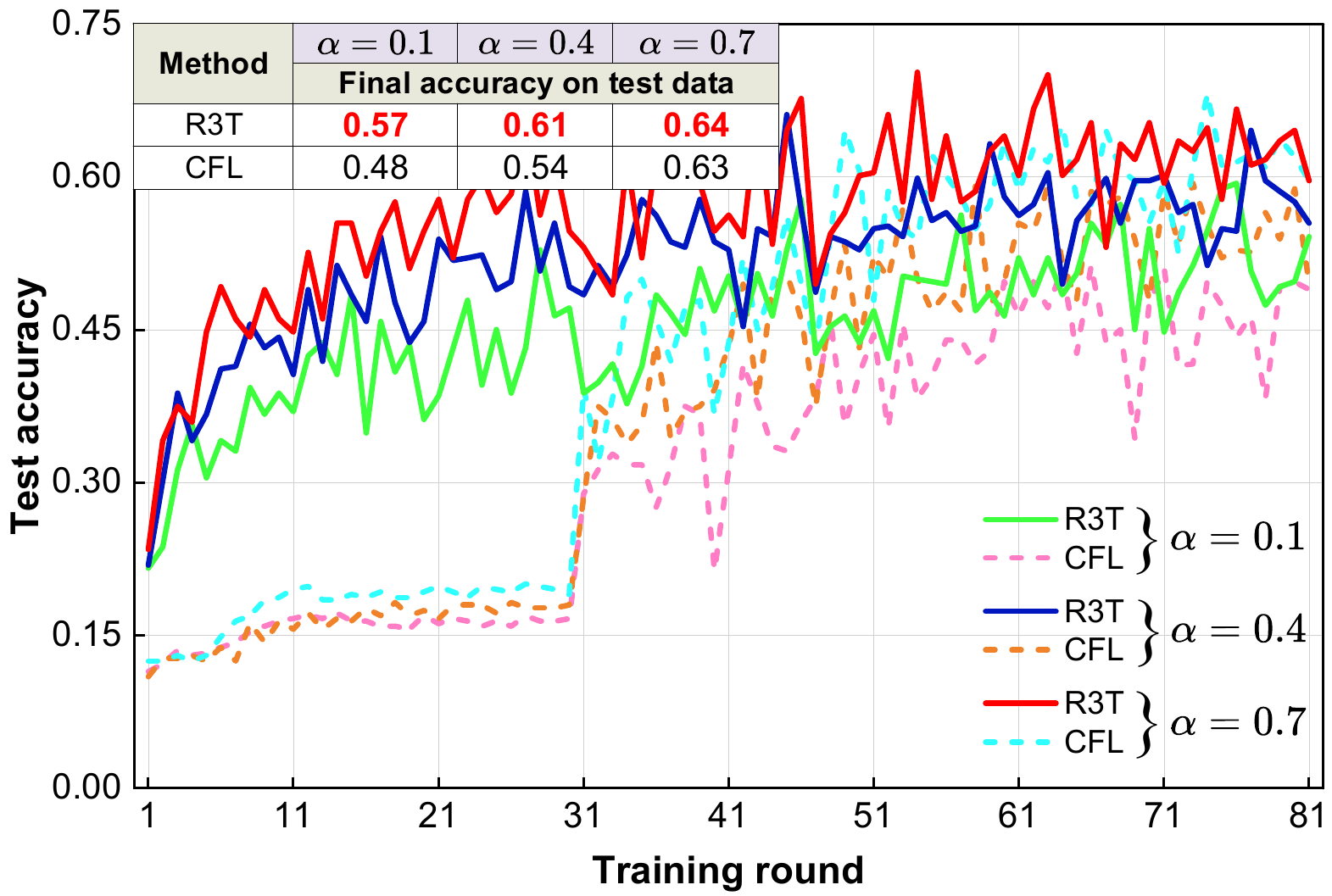}
    \label{fig:cifar10}
    \vspace{-1\baselineskip}
    \caption{CIFAR-10.}
\end{subfigure}
\begin{subfigure}[t]{.48\textwidth}
    \centering
    \includegraphics[width=\linewidth]
     {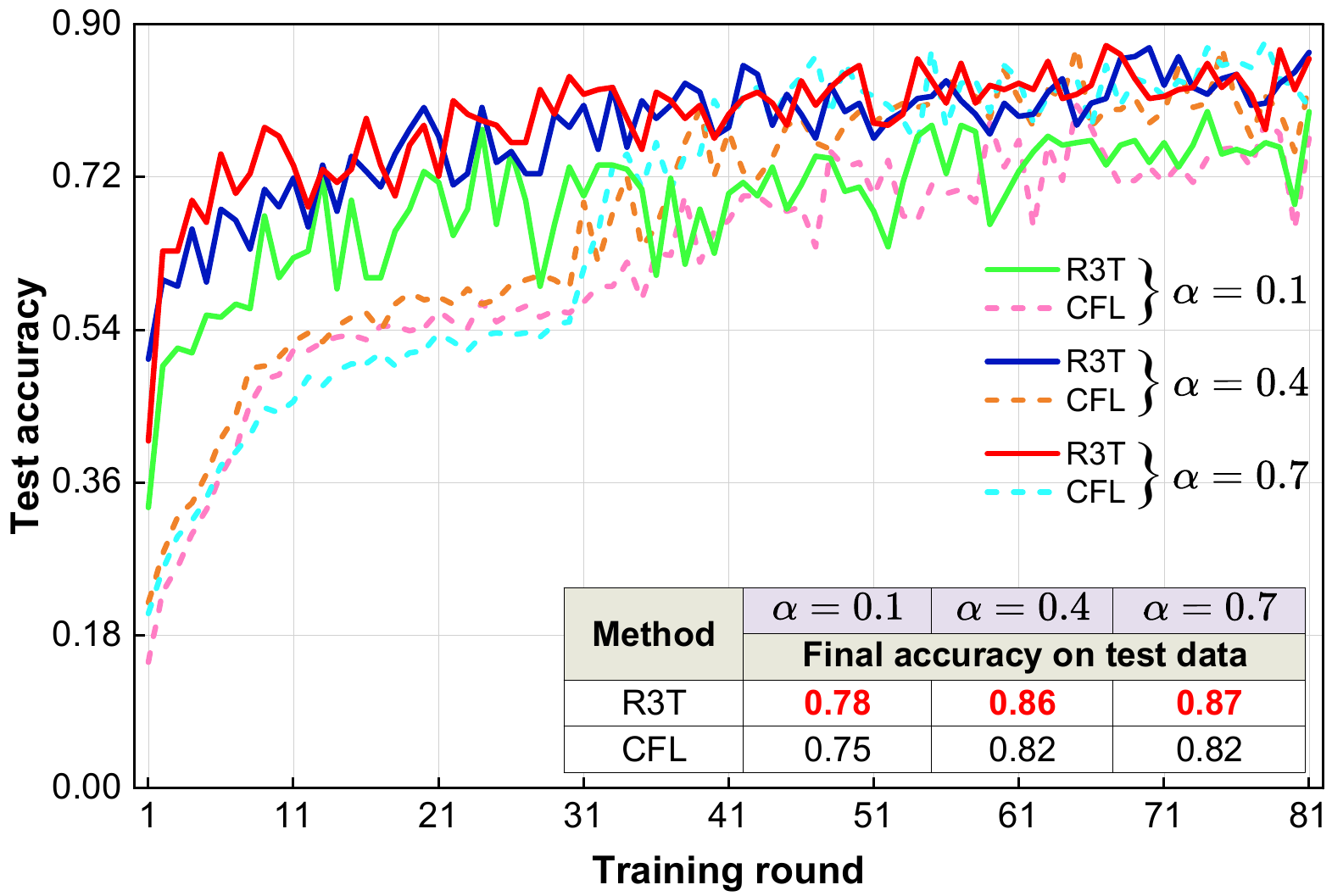}
    \caption{FMNIST.}
    \label{fig:fmnist}
\end{subfigure}
\caption{Test accuracy of {\proposed} and CFL over $T$ rounds.}
\label{fig:testbed_results}
\end{figure}

\paragraph{Convergence rate}\label{para:convergence_rate} As illustrated  in Fig.~\ref{fig:testbed_results}, {\proposed} demonstrates its efficiency by surpassing $\mathrm{CFL}$ in \textit{improving convergence rate} under non-IID settings for both CIFAR-10 and FMNIST datasets. Particularly, {\proposed} achieves a significant increase in performance during the initial FL training phase. This improvement is attributed to its strategy of attracting the highest-quality clients by offering an optimal set of contract items, which includes timely and substantial rewards, early in the training process. For example, {\proposed} tends to reach the final target accuracy in just 20 rounds, compared with approximately 61 rounds for $\mathrm{CFL}$ on the datasets. 
 
\paragraph{Long-term performance}\label{para:longtermperf} As shown in Fig.~\ref{fig:testbed_results}, for $\alpha = 0.7$ in both datasets, both {\fontfamily{qcr}\selectfont R3T} and CFL maintain stable test accuracies after convergence. Under highly non-IID settings (i.e., $\alpha \in \{0.1, 0.4\}$), {\proposed} demonstrates \textit{more sustainable performance} compared to CFL. In $\mathrm{CFL}$, the practice of treating clients' contributions and efforts equally across all training rounds, especially under a significant quantity and label distribution skew, detrimentally affects the long-term performance of the FL model, regardless of subsequent efforts. {{Specifically, {\proposed} outperforms $\mathrm{CFL}$ in terms of the final accuracy evaluated on the entire test set by up to 9\%.}} Besides, although {\proposed} exhibits slight accuracy fluctuations after {\clps}, the overall test accuracy remains stable around the highest bound and {leads to higher final test accuracy}. These observations demonstrate that attracting high-quality clients to join early yields a 2-3$\times$ training speedup, reduces the number of clients required to reach comparable target learning performance by 5.2\%-47.6\%, and improves final accuracy by up to 9\%.


\begin{table}[t!]
\centering
\footnotesize
{
\caption{Comparison of execution costs between {\proposed} and CFL evaluated on FMNIST and CIFAR-10 datasets}
\label{tab:overhead}
\begin{tabular}{|l|c|c|c|c|c|c|}
\hline
\rowcolor[HTML]{EFEFEF} 
\textbf{Dataset} & \textbf{Meth.} & \textbf{$\alpha$} & \textbf{Latency} & \textbf{Comm.} & \textbf{TPS} & \textbf{Cost} \\
\hline \hline 
\multirow{6}{*}{FMNIST} 
& \multirow{3}{*}{{\proposed}} 
& 0.1 & 1809.37 & 243876768 & 8.79 & 1.1521 \\
& & 0.4 & 1823.61 & 243876768 & 8.89 & 1.1521 \\
& & 0.7 & 1883.50 & 243124068 & 8.86 &1.1487\\
\cline{2-7}
& \multirow{3}{*}{CFL} 
& 0.1 & 1808.77 & 243874800 & N/A & N/A \\
& & 0.4 & 1823.03 & 243874800 & N/A & N/A \\
& & 0.7 & 1882.91 & 243122100 & N/A & N/A \\
\hline
\multirow{6}{*}{CIFAR-10} 
& \multirow{3}{*}{{\proposed}} 
& 0.1 & 2240.88 &  404248992 & 8.48 & 1.1521 \\
& & 0.4 & 2146.87 &  404248992 & 8.74  & 1.1522 \\
& & 0.7 & 2155.94 & 404248992  & 8.90 &  1.1522 \\
\cline{2-7}
& \multirow{3}{*}{CFL} 
& 0.1 &  2240.32 & 404247024 & N/A & N/A \\
& & 0.4 & 2146.32 & 404247024 & N/A & N/A \\
& & 0.7 & 2155.40 & 404247024 & N/A & N/A \\
\hline
\end{tabular}
}
\begin{minipage}{0.95\linewidth}
\footnotesize
\textbf{Meth.} denotes the method. 
\textbf{Latency} is reported in \textit{seconds}, and \textbf{Comm.} denotes the communication overhead in \textit{bytes}. 
\textbf{Cost} denotes the gas cost in \textit{ETH}, computed using $1$ gas $= 2 \times 10^{-8}$ ETH~\cite{nguyen2021blockchain}.
\end{minipage}
\end{table}

{  \paragraph{System overhead}\label{para:overhead} {Table~\ref{tab:overhead} presents the execution overhead of {\proposed} and $\mathrm{CFL}$ on the FMNIST and CIFAR-10 datasets across $\alpha \in \{0.1, 0.4, 0.7\}$. Unlike  $\mathrm{CFL}$, {\proposed} incorporates blockchain smart contracts for decentralized training and incentive management, thereby introducing additional on-chain operations such as IDs, contract-selection recording, model-update referencing, hash digests of local and global model updates, client contribution-metadata logging, and reward settlement.
Accordingly, blockchain-specific metrics, including transactions per second (TPS) of the blockchain under the FL workload and gas cost, are not applicable to  $\mathrm{CFL}$. Despite these additional operations, {\proposed} incurs only marginal additional overhead in latency and communication relative to $\mathrm{CFL}$. This efficiency stems from the architectural design in which the blockchain layer is confined to lightweight coordination tasks, including integrity verification, auditability, and payment settlement, while computationally intensive FL operations, including local model training and server-side aggregation, are executed entirely off-chain. Consequently, the incremental overhead remains negligible, while {\proposed} yields the complementary benefits of tamper-evident records of training contributions, dispute mitigation through immutable evidence, and automated reward distribution via smart contracts.}}













%% file: Tables/Parameter_settings.tex
\begin{table}[t!]
    \caption{Experimental parameters}
    \label{tab_main:parameters}
    \footnotesize
    \centering
    \begin{tabular}{|cc|cc|cc|}
        \hline
        \rowcolor[HTML]{EFEFEF}
        \textbf{Param} & \textbf{Value} 
        & \textbf{Param} & \textbf{Value} 
        & \textbf{Param} & \textbf{Value} \\
        \hline \hline
        $\lambda_{\text{{\fontfamily{qcr}\selectfont CLPs}}}$ & 21 
        & $\lambda_{\text{{\fontfamily{qcr}\selectfont non-CLPs}}}$ & 20 
        & $\theta_k$ & \textit{U}(1, 2) \\
        
        $T_{\text{sim}}$ & 25 
        & $N$ & \{10, 15\} 
        & $K$ & 10 \\
        
        $C$ & 2.4 
        & $\beta$ & 3  
        & $t_{\nonclps}$ & [11, $T$] \\
        
        $\tau_{\clp}^{\text{cifar10}}$
        & \makecell[c]{{\color{blue}0.01}\\[-0.2em]\cite{huang2025pa3fed, yan2023criticalfl}} 
        & $\tau_{\clp}^{\text{fmnist}}$
        & \makecell[c]{0.01\\[-0.2em]\cite{huang2025pa3fed, yan2023criticalfl}} 
        & $T_{\text{poc}}$
        & 81 \\
        \hline
    \end{tabular}
\end{table}

%% file: Sections_revision_wt_trackchanges/5_Conclusion.tex
\section{Conclusion}
\label{sec:conclusion}
In this work, we studied {\fontfamily{qcr}\selectfont R3T}, a time-aware contract-theoretic incentive framework for FL, implemented on blockchain smart contracts. Motivated by the pivotal role of {\clps } and challenges caused by information asymmetries, we modeled how clients' private time and system capabilities, effort, and joining time jointly shape the cloud utility. Building on this characterization, we derived a temporally integrated optimal contract that allocates incentives to provide the right reward at the right time, thereby inducing clients to self-reveal types, join early, and contribute higher effort, particularly in {\clps}, under both complete- and incomplete-information settings. Simulation and proof-of-concept experimental results demonstrate the superiority of {\fontfamily{qcr}\selectfont R3T} over incentive benchmark mechanisms in cloud utility and economic efficiency, while simultaneously accelerating convergence and reducing the required number of participating clients without degrading final global model performance.

In future work, we will consider the impacts of free-riding issues and adversarial behaviors on the efficiency of {\fontfamily{qcr}\selectfont R3T}, especially in {\fontfamily{qcr}\selectfont CLPs}. Besides, it would be interesting to evaluate the efficiency and scalability of {\fontfamily{qcr}\selectfont R3T} across different blockchain platforms integrating  FL systems, where the entire system is operated without relying on the cloud server.